\documentclass[11pt]{article}

\usepackage{authblk}

\title{Beyond IID: data-driven decision making\\ in heterogeneous environments}
\author[1]{Omar Besbes}
\author[2]{Will Ma}
\author[3]{Omar Mouchtaki}
\affil[1]{Columbia University, Graduate School of Business, \texttt{ob2105@gsb.columbia.edu}}
\affil[2]{Columbia University, Graduate School of Business, \texttt{wm2428@gsb.columbia.edu}}
\affil[3]{NYU Stern School of Business, \texttt{om2166@stern.nyu.edu}}

\date{first version: May 26, 2022; last revised: December 31, 2024}
\usepackage[utf8]{inputenc}
\usepackage[T1]{fontenc}
\usepackage[english]{babel}
\usepackage[]{algorithm2e}
\usepackage{mwe}
\usepackage{amsmath,mathtools}

\usepackage{minitoc}

\usepackage{amsthm,bm}
\usepackage{lineno}

\usepackage{pgf,tikz,pgfplots}
\usepackage{subfigure}
\usetikzlibrary{intersections}
\usetikzlibrary{arrows.meta, calc, fit,positioning,shapes.geometric}
\usetikzlibrary{shapes.multipart}
\usetikzlibrary{backgrounds,fit}

\usepackage{booktabs}
\usepackage{multirow}
\usepackage[normalem]{ulem}
\usepackage{scalerel}

\usepackage[shortlabels]{enumitem}

\usepackage{color}              
\usepackage{color-edits}
\addauthor{Will}{blue}

\usepackage{bbm}
\newtheorem{lemma}{Lemma}
\newtheorem{proposition}{Proposition}
\newtheorem{corollary}{Corollary}
\newtheorem{theorem}{Theorem}

\newtheorem{definition}{Definition}

\theoremstyle{definition}
\newtheorem{remark}{Remark}
\newtheorem{example}{Example}

\newcommand {\beq}{\begin{equation}}
\newcommand {\eeq}{\end{equation}}
\newcommand {\beqn}{\begin{equation*}}
\newcommand {\eeqn}{\end{equation*}}
\newcommand {\bear}{\begin{eqnarray}}
\newcommand {\eear}{\end{eqnarray}}
\newcommand {\bearn}{\begin{eqnarray*}}
\newcommand {\eearn}{\end{eqnarray*}}

\DeclareMathOperator*{\argmax}{arg\,max}
\DeclareMathOperator*{\argmin}{arg\,min}

\newcommand{\actVr}{x}
\newcommand{\actSp}{{\mathcal{X}} }
\newcommand{\envVr}{\xi}
\newcommand{\envSp}{\Xi}
\newcommand{\envSamples}{\bm{\envVr^\numS}}
\newcommand{\mesSp}{\mathcal{P}}
\newcommand{\mesOut}{\mu}
\newcommand{\mesIn}[1]{\nu_{#1}}
\newcommand{\oracle}[1]{\mathtt{ORACLE} \left( #1 \right)}
\newcommand{\Lip}[1]{\mathrm{Lip} \left( #1 \right)}
\newcommand{\V}[1]{\mathrm{V}\left( #1 \right)}

\newcommand{\diam}[1]{\mbox{diam} \left( #1 \right)}
\newcommand{\dev}[1]{\pi^{\mathrm{SAA}(#1)}}
\newcommand{\g}[2][]{ \mathcal{G} \left(#1,#2\right) }
\newcommand{\agSp}{\Pi_{\text{agn}}}
\newcommand{\dis}[1]{d_{#1}}
\newcommand{\gen}[1]{\mathfrak{M}_{#1}}
\DeclareMathOperator{\Span}{span}

\newcommand{\UB}{\mathrm{Unif}}

\newcommand{\opt}{\mathrm{opt}}
\newcommand{\numS}{n}

\newcommand{\SAA}{ \pi^{\mathrm{SAA}}}

\usepackage[right=1.0in, top=1in, bottom=1.0in, left=1.0in]{geometry}

\usepackage{lmodern}

\usepackage{graphicx}
\usepackage{lscape}
\usepackage[final]{pdfpages}
\usepackage{amsthm}
\usepackage[authoryear,round]{natbib}
\usepackage{amsfonts}
\usepackage{mathtools}
\usepackage{bbm}
\usepackage{dsfont}

\usepackage{setspace}
\usepackage[colorlinks=true, linkcolor=blue, citecolor=blue, hypertexnames=false]{hyperref}  
\usepackage{nameref}
\usepackage{cleveref}

\begin{document}

\maketitle

\vspace{-.9cm}

\begin{abstract}
How should one leverage historical data when past observations are not perfectly indicative of the future, e.g., due to the presence of unobserved confounders which one cannot ``correct'' for?
Motivated by this question, we study  a data-driven decision-making framework in which historical samples are generated from unknown and different distributions assumed to lie in a heterogeneity ball with known radius and centered around the (also) unknown future (out-of-sample) distribution on which the performance of a decision will be evaluated.  
This work aims at  analyzing the performance of central data-driven policies   but also near-optimal ones in these \textit{heterogeneous environments} and understanding key drivers of performance. 
We establish a first result which allows to upper bound the asymptotic worst-case  regret of a broad class of policies.
Leveraging this result, for any integral probability metric, we provide a general analysis of the performance achieved by Sample Average Approximation (SAA) as a function of the radius of the heterogeneity ball. This analysis is centered around the approximation parameter, a notion of complexity we introduce to capture how the interplay between the heterogeneity and the problem structure impacts the performance of SAA. In turn, we illustrate through several widely-studied problems -- e.g., newsvendor, pricing-- how this methodology can be applied and find that the performance of SAA varies considerably depending on the combinations of problem classes and heterogeneity. 
The failure of SAA for certain instances motivates the design of alternative policies to achieve rate-optimality. We derive problem-dependent policies achieving strong guarantees for the illustrative problems described above and provide initial results towards a principled approach for the design and analysis of general rate-optimal algorithms.

\end{abstract}

\doparttoc 
\faketableofcontents 

\section{Introduction}
\label{sec:intro}
In optimization under uncertainty, the desirability of a decision (e.g., inventory) depends on an unknown future outcome (e.g., demand).
Typically, past data is collected to be indicative of the future, and hence inform our decision.
However, ideal data-driven decision-making requires postulating beliefs about the reliability of the past data, and importantly, whether the future may \textit{deviate} from it.
In practice, past data may depend on contexts, some of which can be controlled for, and some of which cannot (unobserved confounders). This may introduce data heterogeneity that is not ``correctable.''

In this paper, we consider  a framework for modeling how the future may deviate from past data.
At a high level, it accomplishes three goals:
$i.)$ capture different forms of future deviation, including no deviation, under the same umbrella, and develop an approach for performance evaluation for general policies.
$ii.)$  understand the drivers of performance of a central policy, which makes the decision that optimizes the \textit{average} objective value over past data, and
$iii.)$ if the central policy performs poorly, suggest modifications that are robust to different forms of anticipated deviation, and illustrate the power of doing so through classical problems.

This framework is general, capturing different problems, and our analysis leads to insightful  
problem-specific conclusions. The only assumption is that the policy defined in $ii.)$ above can be computed, which requires perfect counterfactual evaluation of any decision on all data points, as opposed to settings where past data are affected by previous actions \citep{besbes2014stochastic,cheung2022hedging,karnin2016multi,luo2018efficient,lykouris2018stochastic,lykouris2021corruption}.
The policy described in $ii.)$ is widely studied across different fields, known as\footnote{It is also known under different names, e.g., Empirical Optimization or Empirical Risk Minimization.} \textit{Sample Average Approximation (SAA)} \citep{kleywegt2002sample,kim2015guide,lam2021impossibility}, although we emphasize that we also derive \textit{new policies beyond SAA} when it falters.

\subsection{Framework description}

\noindent \textbf{Framework (\Cref{sec:hetEnv}).} We study a framework which models the future outcome as being drawn from an unknown distribution, and past data as being drawn independently from (also unknown) ``nearby'' distributions. 
In this way, past data is indicative of the future, and we emphasize that the samples of data can be drawn from different nearby distributions.
In accord, we call this setting a \textit{heterogeneous environment}.
Using varying definitions of ``nearby distribution'' (e.g.\ based on Kolmogorov and Wasserstein distances; see \Cref{sec:IPM} for exact definition), we analyze different forms of deviations possible between the past and future, and let a radius $\epsilon$ bound the allowed deviation.
When $\epsilon=0$, our framework captures the classical independent and identically distributed (i.i.d.) setting, in which past data are drawn from the same distribution as the future outcome \citep{vapnik1974theory,kleywegt2002sample,mohajerin2018data,bertsimas2018robust}.
The framework we study presents a parametrization to interpolate between the i.i.d. setting and the adversarial one and shares conceptually similar goals to previous approaches considering algorithm analysis beyond worst-case \citep{rakhlin2011online,bilodeau2020relaxing,roughgarden2021beyond,haghtalab2022smoothed,block2022smoothed,haghtalab2022oracle}.

\noindent \textbf{Performance measure (\Cref{sec:performance}).}
We consider the performance of different data-driven policies, which map past data into a decision for a given problem; SAA defines one such feasible data-driven policy for any problem.
Regret is measured as the difference in objective between the policy's decision and the optimal decision knowing the future distribution, taking an expectation over the draws of past data (which affect the policy's decision), any intrinsic randomness in the policy, and the outcome realization (which affects the objective of both the policy's decision and the optimal decision).
We then take a worst case regret over all possible distributions that could have been chosen for the future outcome and the data points, to evaluate the performance of a fixed data-driven policy.  This performance depends on the problem, the number of data points $n$, the definition of ``nearby distribution,'' and the radius $\epsilon$.
In this paper we adopt an asymptotic view point, in which the number of data point $\numS$ tends to $\infty$ and we aim at understanding the worst-case regret of data-driven policies in this regime.

\subsection{Contributions}

\subsubsection{Performance evaluation for general policies (\Cref{sec:reduction})}
Our first contribution lies in developing a general result to evaluate the asymptotic worst-case regret of a large class of data-driven policies.
In particular, we establish in \Cref{thm:reduction_stat} the following: asymptotically, as the number of samples goes to $\infty$, the worst-case regret of a policy, whose decision only depends on the empirical distribution and not the number of samples, is upper-bounded by that of an alternative problem with two major simplifications: $i.)$ the worst-case is now taken over only a \emph{single} distribution for historical environments as opposed to a sequence of heterogeneous distributions and, $ii.)$ the decision-maker (DM) has access to the exact historical distribution.

The upper bound we derive can be seen as a uniform distributionally robust optimization problem \textit{under a regret performance metric}, and offers a systematic way to derive upper bounds for policies of interest. We leverage it to analyze both the SAA policy  
 but also to derive guarantees on alternative policies when SAA falters. 
Moreover, we show that the assumption that the policy does not depend on the number of samples can be made without loss of optimality, because for any problem, there exists such a policy achieving the optimal asymptotic worst-case regret (see \Cref{prop:lb_agnostic}).

\subsubsection{Performance drivers for SAA and illustration for central problems (\Cref{sec:SAA-analysis})}

\noindent \textbf{Analysis of SAA for integral probability metrics (IPMs).} 
Our second contribution is the development of a general framework to bound the asymptotic regret of SAA under the assumption that the distance is an integral probability metric (IPM), as defined in \Cref{sec:IPM}. Notably, this broad class of metrics—which includes the Kolmogorov and Wasserstein distances—can be linked to specific generating classes of functions. We leverage this connection to introduce the ``approximation parameter,'' a quantity that captures the complexity required to uniformly approximate the objective function using functions from the generating class of the IPM. This parameter not only provides a bound on the regret of SAA but also has a natural geometric interpretation. By utilizing standard topological results, we show that this approach enables meaningful regret bounds. Specifically, for the Kolmogorov and Wasserstein distances, our framework relates the worst-case regret of SAA to practical quantities, such as the Lipschitz constant or the total variation of the objective function. Our analysis extends the results derived in the context of ``the method of probability metrics''  \citep{rachev2002quantitative}. We discuss the relation of our work to these results in \Cref{rem:comp_litt}.

\noindent \textbf{Illustration of our method for central problems.}
To illustrate the versatility of our methodology we study the performance of SAA for two central problems (newsvendor and pricing) and two distances defining the heterogeneity (Kolmogorov and Wasserstein). These problems have  received significant attention in the literature across Operations Research, Economics, and Computer Science in the known-prior, data-driven, distributionally robust, and advice-augmented settings. 

We show in \Cref{sec:SAA_bound} that the upper bounds based on the approximation parameter directly imply bounds on the worst-case regret of SAA for Newsvendor under both heterogeneity types and for pricing under the Kolmogorov heterogeneity. Furthermore, we complement these results with lower bounds on the best achievable performance and show that for these three settings, SAA achieves the best possible dependence in $\epsilon$ and its asymptotic worst-case regret scales linearly with $\epsilon$. For pricing under Wasserstein distance, the picture is starkly different: the approximation parameter becomes infinite,  and hence does not allow us to derive any meaningful upper bound on the worst-case regret of SAA. As a matter of fact, we show that SAA performs extremely poorly and the asymptotic regret does not even vanish when $\epsilon$ goes to $0$.  Hence, the performance of SAA deteriorates considerably by slightly deviating from the i.i.d. regime for this class of problems. 

These two problems illustrate that the approximation parameter is a powerful tool to derive tight bounds on the worst-case regret of SAA. Furthermore, while we do not derive a formal lower bound on the worst-case regret of SAA which depends on the approximation parameter, the pricing problem with Wasserstein heterogeneity suggests that a large approximation parameter usually corresponds to a poor performance of SAA. 
This finding is consistent with a third problem we investigate in \Cref{sec:ski} , the ski-rental problem and for which SAA falters similarly when the approximation parameter is large.

We also illustrate in \Cref{sec:apx_auction} how our framework can be generalized to tackle a multi-dimensional problem. In particular, we investigate the Bayesian mechanism design problem and show that the result derived by \cite{brustle2020multi} for SAA under the Kolmogorov distance can be retrieved using our approximation parameter framework.

\subsubsection{Design and analysis of policies beyond SAA (\Cref{sec:beyond_SAA})}\label{sec:intro_beyond_SAA}

The failure of SAA for the Wasserstein pricing problem naturally raises the question of designing and quantifying the performance of alternative policies to improve upon SAA and achieve near-optimality.
Our third main contribution consists in developing and analyzing policies which operate efficiently in heterogeneous environments. 

\noindent \textbf{Analyzing policies beyond SAA to achieve rate-optimality.}
In \Cref{sec:rate_opt_pricing} we complete the picture for the pricing problem under Wasserstein heterogeneity. We propose a policy which appropriately deflates the price selected by SAA, and  show that this policy achieves a worst-case regret which has a $\sqrt{\epsilon}$ dependence in the radius of heterogeneity. We also show that this performance is rate-optimal.
To our understanding, analyzing these non-SAA policies (which are required for good performance in pricing) deviates from standard analyses used in learning theory and critically requires the reduction we derive in \Cref{thm:reduction_stat}, as we elaborate on in \Cref{sec:rate_opt_pricing}. To derive our result, we leverage the structure of the objective function in pricing: while it is not continuous in general, it is ensured to be one-sided Lipschitz-continuous (when deflating the price). We combine this observation with a critical relation between the Wasserstein distance (in 1 dimension) of two probability measures and their associated cumulative distribution functions to obtain the desired result. We believe that this problem-specific analysis may be of independent interest.

We also study the ski-rental problem, a third central problem for which one needs to go beyond SAA to improve performance even for the Kolmogorov distance (cf. \Cref{sec:ski}). In this setting, we consider the policy proposed by \cite{diakonikolas2021learning} and derive a lower bound on the best achievable performance which matches their upper bound up to logarithmic factors. 
For the Wasserstein distance, we show that similarly to pricing, SAA incurs a worst-case regret which does not shrink to $0$ as $\epsilon$ goes to $0$; but a policy which inflates the SAA decision appropriately achieves rate-optimality. The proof techniques and results derived for ski-rental leverage the structure of the problem and could be of independent interest.

\noindent \textbf{Towards a general design principle for data-driven policies.}
The examples of pricing and ski-rental illustrate how critical \Cref{thm:reduction_stat} is to derive guarantees for a wide range of data-driven policies but leaves open the choice of the policy that should be analyzed. In \Cref{sec:UDRO} we take initial steps for the design of general policies with strong asymptotic worst-case regret guarantees. We first introduce a policy naturally suggested by \Cref{thm:reduction_stat}, which minimizes the DRO regret objective and show that this policy is essentially minimax optimal. We also derive a general result regarding the performance of the classical DRO policy in \cite{rahimian2019distributionally} (formally, a variant which optimizes the worst-case objective over randomized actions) and show in \Cref{{prop:RDRO_to_DRO}} that when the optimal value of the objective is not too sensitive to the structure of the heterogeneity, the DRO policy achieves a rate-optimal regret guarantee. We further show that this condition holds for the pricing problem under Wasserstein heterogeneity. 
We note that the two general policies which we investigate could be  challenging to compute, and we leave as open the question regarding efficient implementation of these two policies.

\subsubsection{Impact of heterogeneity on performance across some classical problems}\label{sec:intro_insights}

The results above allow to illustrate the interplay of heterogeneity and problem class.  In \Cref{tab:results}, we present a high level summary of the problem-specific results we derive. We assume that the support of the unknown environment is in $[0,M]$ for some positive real number $M$ which can be interpreted as parametrizing the ``extent'' of uncertainty one faces (e.g., maximal values of demand in newsvendor).
We track the dependence on both the heterogeneity radius $\epsilon$, as well as $M$. These results are for the asymptotic setting in which the number of samples grows large.
\begin{table}[h!]
\centering

\begin{tabular}{c ccc } 
~  & ~& Kolmogorov & Wasserstein \\ 
 \hline
 \hline
\multirow{2}{*}{Newsvendor} &SAA & $\Theta ( M \epsilon )$ & $\Theta ( \epsilon )$  \\  
&best policy & $\Theta ( M \epsilon )$ & $\Theta ( \epsilon )$\\
 \hline
\multirow{2}{*}{Pricing} &SAA & $\Theta ( M \epsilon )$ & $\Omega ( M )$  \\  
&best policy & $ \Theta ( M \epsilon )$ &  $\Theta ( \sqrt{M \epsilon} )$\\
\hline
 \multirow{2}{*}{Ski-rental} &SAA & $\Theta ( M \epsilon )$ & $\Theta ( 1 )$  \\  
  &best policy & $\tilde{\Theta} ( \epsilon )^*$ &  $\Theta ( \sqrt{\epsilon} )$
\end{tabular}
 \caption{\textbf{SAA and achievable performance for different problem classes and heterogeneity balls.} *The upper bound was proved in \cite{diakonikolas2021learning} and we complement it by a matching lower bound up to logarithmic factors} 
 \label{tab:results}
\end{table}

Our illustrative examples unveil a critical insight regarding the performance of data-driven policies in heterogeneous environments. It highlights the important fact that the quality of decisions made with SAA in heterogeneous environments is highly dependent on the combination of problem class and heterogeneity. 
Furthermore, the dependence of the best achievable performance in the heterogeneity radius $\epsilon$ varies considerably across problem classes and notions of heterogeneity.
Therefore, our work implies that it is critical for a decision-maker to understand the \textit{nature of the heterogeneity} faced, and the impact it may have on the objective of interest when making decisions.

\subsection{Further Related Work}\label{sec:relatedWork}
 
\noindent \textbf{Connection to offline data corruption models.}  
Different notions of offline data corruptions and robustness have been previously studied in statistics. \cite{huber1992robust} introduces the so-called $\epsilon$-contamination model in which an $\epsilon$ fraction of the observed samples is corrupted. They study the problem of location estimation for $\epsilon$-contaminated gaussian distributions and for approximations of gaussian distributions that are $\epsilon$ away in Kolmogorov distance. They derive in both cases minimax estimators among the class of $M$-estimators. \cite{hampel1971general} considers distributions that are $\epsilon$ away for the Prokhorov distance and establish that the classical empirical mean is a non-robust estimator. 
These models have been extensively developed in learning theory \citep{servedio2003smooth,haussler1992decision,kearns1993learning,kearns1994toward,servedio2003smooth,klivans2009learning,diakonikolas2019robust,blanc2021power,DuchiNamkoong2021,bennouna2022holistic} and, more recently in revenue management  \citep{cai2017learning, brustle2020multi, XiWang2022pricingoutliers, guo2021robust}.
Most approaches consider $\epsilon$-contamination settings \citep{huber1992robust} in which the adversary can shift a small proportion of data arbitrarily far from the true distribution.  
They consider either oblivious adversaries who fix a \textit{single} corrupted distribution from which samples are generated or adaptive ones who observe samples before corrupting them.  \citep{blanc2021power} establishes that adaptive adversaries yield equivalent performance to oblivious ones in many settings.
We note that our framework is general enough to derive guarantees on the regret under the $\epsilon$-contamination model and the strong contamination model. We discuss in more details connections between our work and previous contamination models in \Cref{sec:relation_to_models}.

Closest to us are \cite{brustle2020multi, guo2021robust} who study optimal auctions. They assume the distribution is common across all past observations; the former focuses on regret  for various distances and the the latter  on a ratio metric when  data is generated from MHR and regular distributions close in Kolmogorov distance. The framework we develop is anchored around a  setting that  allows for \textit{heterogeneous} ``nearby'' distributions for each past observation, and is general in that it allows us to unify a variety of problems/metrics and highlight how these affect the levels of achievable performance. 
Related to auction design, \cite{dutting2019posted} study the robustness of well-known posted pricing and prophet-inequality-based mechanisms when computed on distributions that differ from the true distributions for various distances.  

\noindent \textbf{Connection to online and batch learning in non-stationary environments.}
Our work also relate broadly to papers studying learnability in changing environment. \cite{hanneke2021learning,blanchard2022universal} adopt a universal learning approach and analyze algorithms which are able to achieve vanishing long-term average loss for general data-generation processes. \cite{mansour2009domain,ben2010theory,mohri2012new} derive generalization guarantees in a setting where the training set distribution may differ from the test set distribution.

\noindent \textbf{Connection to robust optimization and stability in stochastic programming.}
Our work is also conceptually related to robust optimization in which, one design decisions for optimization problems that are robust to some variations in the parameters of the problem (see the survey by \cite{bertsimas2011theory}). When the variation is on the underlying probability measure, people refer to this approach as distributionally robust optimization (see references in \cite{rahimian2019distributionally}). \cite{glasserman2014robust} quantifies the variation in risk that may result from distribution misspecification and discusses for several applications in portfolio optimization the shape and impact of the worst-case distribution shift within a Kullback-Leibler ball. In our work, we quantify how the problem structure \textit{and the notion of heterogeneity} impact the performance.
Another related line of papers use the \textit{method of probability metrics} to understand when small perturbations of the distribution with respect to a given metric result in a small perturbation of the value of the stochastic problem and of the set of optimal solutions \citep{Romisch1991stability,schultz1996rates,rachev2002quantitative,dupavcova2003scenario,xu2021quantitative,pichler2022quantitative}. These papers fix the problem of interest (e.g. newsvendor, two-stage problems with recourse, etc...) and construct a ``canonical IPM'' under which the stochastic optimization problem is stable. Most of these works focus on the Wasserstein distance or the related Fortet-Mourier metric.
Our paper differs from this literature along two critical dimensions. First, when analyzing SAA with IPMs, we take the IPM as given and do not construct one that ensures stability of the problem (see for instance pricing with Wasserstein distance). In that case, we derive a broader methodology which enables to characterize the interplay between the problem structure and the heterogeneity structure (see detailed discussion in \Cref{rem:comp_litt}). Second, we develop and analyze the performance of alternative data-driven policies in settings where SAA fails.

\noindent \textbf{Prior works on our illustrative examples.} Our paper also relates to previous work specifically addressing each of the prototypical problems we consider. The newsvendor problem is a foundational problem in Operations Research,
with early focus on the setting where the demand distribution is known \citep[see][]{morse1946methods},
and since then studied in various distributionally robust settings \citep{scarf1958min,gallego1993distribution,perakis2008regret}  but also in data-driven ones \citep{levi2015data,cheung2019sampling,lin2022data,besbes2023big}. The pricing problem is a central problem in Economics, with early focus on the  setting with known prior \citep[see][]{myerson1981optimal,riley1983optimal},
and more recently in the distributionally robust \citep{eren2010monopoly,bergemann2011robust} and data-driven settings \citep{fu2015randomization,huang2018making,babaioff2018two,daskalakis2020more,allouah2022pricing}. Ski-rental on its end is a foundational problem in Computer Science, critical to the development of competitive analysis \citep[see][]{borodin2005online} which considers the \textit{adversarial} setting where nothing is known,
and recently also important to the development of \textit{advice-augmented algorithms} which receive a prediction on the length of the ski season \citep{purohit2018improving,diakonikolas2021learning}.

\subsection{Preliminaries}

\noindent \textbf{Notation.}  
We denote by $\mathbb{N}^*$ the set of positive integers and by $\mathbb{R}_+$ the set of non-negative real numbers.
 For any set $A$, $\Delta \left( A \right)$ denotes the set of probability measures on $A$. 
 For any deterministic sequences $(a_k)_{k \in \mathbb{N}}$ and $(b_k)_{k \in \mathbb{N}}$ both indexed by a common index $k$ that goes to $\infty$, we say that $a_k = o(b_k)$ if $a_k/b_k \to 0$, $a_k = \mathcal{O} \left( b_k \right)$ if there exists a finite $M >0$ such that $|a_k| \leq M |b_k|$ for $k$ large enough, $a_k = \omega(b_k)$ if $|a_k/b_k| \to \infty$, $a_k = \Omega(b_k)$ if there exist a finite $M > 0$, such that $|a_k| > M |b_k|$ for $k$ large enough, $a_k = \Theta \left( b_k \right)$ if $a_k = \mathcal{O} \left( b_k \right)$ and $a_k = \Omega \left( b_k \right)$, and $a_k \sim b_k$ if $a_k/b_k \to 1$.  For any metric spaces $(\mathcal{M}_1, d_1)$, $(\mathcal{M}_2,d_2)$, and any mapping $f$ from $\mathcal{M}_1$ to $\mathcal{M}_2$, we denote by $\Lip{f}$, the Lipschitz constant of $f$, i.e. the largest $L \in [0, \infty) \cup \{ \infty \}$  such that, for every $x,y \in \mathcal{M}_1$, we have $d_2(f(x),f(y)) \leq L \cdot d_1(x,y)$. For a function $f$ defined on a closed interval $[a,b] \subset \mathbb{R}$, we define the total variation of $f$  as $\V{f} := \sup \{ \sum_{i=1}^\numS |f(z_i) - f(z_{i+1})| \,  \text{s.t.} \, a = z_1 < z_2 < \ldots < z_{\numS+1} = b \}$.
 For any metric space, $(\mathcal{M}, d)$, we denote by $\diam{\mathcal{M}}:= \sup_{x,y \in \mathcal{M}} d(x,y)$ the diameter of $\mathcal{M}$. For a set $\mathcal{S}$, we denote by $\Span \left( \mathcal{S} \right)$ the set generated by finite linear combinations of elements from $\mathcal{S}$, i.e. $\Span \left( \mathcal{S} \right) = \{ \sum_{i=1}^k \lambda_i \cdot v_i \vert \, k \in \mathbb{N}, \, v_i \in \mathcal{S}, \, \lambda_i \in \mathbb{R}  \}$. For any real-valued function defined on a set $\mathcal{S}$, we denote by $\|f\|_{\infty} = \sup_{\actVr \in S} |f(\actVr)|$ the uniform norm. 
 
 All proofs of statements are deferred to the appendix.

\section{Problem formulation}\label{sec:formulation}

We consider a general stochastic optimization problem in which a DM wants to optimize an objective in a stochastic environment. Formally, we denote by $\actSp$ the  decision space and by $\envSp$ the  space of environment realizations. 
The objective of the DM is represented by a function $g : \actSp \times \envSp \to \mathbb{R}$ which maps each pair of decision $\actVr \in \actSp$ and environment realization $\envVr \in \envSp$ to an objective value. 

The DM faces a subset of possible probability measures $\mesSp \subset \Delta \left( \envSp \right)$, where $\Delta \left( \envSp \right)$ denotes the space of probability measures supported on $\envSp$.\footnote{Rigorously, one should specify a $\sigma$-algebra $\mathcal{A}$ on $\envSp$. We do not specify explicitly these when not necessary to define the objects.} For every probability measure $\mesOut \in \mesSp$,  the goal of the DM is to optimize the objective defined as
\begin{equation}
\g[\cdot]{\mesOut} : \begin{cases}
\actSp \to \mathbb{R} \\
\actVr \mapsto \mathbb{E}_{ \envVr \sim \mesOut } \left[ g(\actVr, \envVr) \right].
\end{cases}
\end{equation}
We extend the definition of  $\g[\cdot]{\mesOut}$ to allow for randomized decisions $\rho \in \Delta \left( \actSp \right)$,  in which case we define $\g[\rho]{\mesOut} = \mathbb{E}_{\actVr \sim \rho} \left[ \g[\actVr]{\mesOut} \right]$ to also take an expectation over the randomized decision.

Remark that depending on the nature of the problem, the objective could represent for example  a function to maximize (profit) or to minimize (cost). To unify exposition we consider that the decision-maker aims at solving the maximization problem,
\begin{equation*}
    \opt(\mesOut):= \sup_{\actVr \in \actSp} \g[\actVr]{\mesOut},
\end{equation*}
where $\mesOut \in \mesSp$ is an unknown distribution. The value  $\opt(\mesOut)$ is the value achieved by an oracle who knows the demand distribution.

\subsection{Heterogeneous environments} \label{sec:hetEnv}
In the absence of distributional knowledge on the underlying environment, a common approach consists in assuming that historical samples are independent and identically drawn from a fixed probability measure $\mesOut$. We relax the assumption that the samples are identically distributed and introduce a framework with heterogeneous distributions based on a parameter $\epsilon$.

Let $d : \mesSp \times \mesSp \to \mathbb{R}$ be a metric, which we refer to interchangeably as a distance, on $\mesSp$. Given a target probability measure $\mesOut \in \mesSp$ and a radius of heterogeneity $\epsilon$, we define the heterogeneity ball anchored around $\mesOut$ with radius $\epsilon$ as
\begin{equation*}
\mathcal{U}_\epsilon(\mesOut) := \{ \nu\in \mesSp, \; \text{s.t} \; d \left( \mesOut,\nu\right) \leq \epsilon \}.
\end{equation*}
The heterogeneity ball $\mathcal{U}_\epsilon(\mesOut)$ includes all distributions that are within a distance of $\epsilon$ from $\mesOut$.
Next, we formally define a data-driven decision-making problem in a heterogeneous environment.
\begin{definition}[Data-driven decision-making problem in heterogeneous environment]
An instance $\mathcal{I}$ of a data-driven decision-making problem in a heterogeneous environment is defined by the tuple: $\mathcal{I}=\left( \actSp, \envSp, \mesSp, d, g \right)$, where $\actSp$ is the set of possible actions, $\envSp$ is the environment space, $\mesSp$ is the space of measures nature can select from, $d$ is a metric on the space of measures, and $g$ is an objective function.
\end{definition}

\noindent \textbf{Focal metrics.}  The definition above is general, and the framework can be applied to a variety of metrics. In this work, we also aim at developing an understanding of the effect of heterogeneity under various metrics and
various problem classes. 
To illustrate the versatility of our framework, we will focus in this work on two common distances: the Kolmogorov distance denoted by $d_{K}$ and the Wasserstein distance denoted by $d_{W}$ (see \Cref{sec:IPM} for formal definitions).
These distances describe different ways in which two distributions may differ and, as we will see,  imply different tradeoffs in the problems we analyze.

\subsection{Data-driven policies and performance} \label{sec:performance}

\noindent \textbf{Data-driven policies.}  In most settings, the out-of-sample (or future) probability measure $\mesOut$ is unknown to the DM and the optimal objective $\opt(\mesOut)$ is not achievable. In turn, the DM only observes a sequence of historical samples $\envVr_1,\envVr_2, \ldots, \envVr_\numS$ representing the previous environment realizations they observed. We assume that these samples are drawn independently from unknown distributions $\mesIn{1}, \ldots, \mesIn{\numS}$, respectively,  and the only knowledge is that these distributions are ``not far'' from the future unknown distribution $\mesOut$, and belong to   $\mathcal{U}_\epsilon(\mesOut)$.

In general, a data-driven policy is a mapping from the past samples to a (randomized) decision. In this work, we use an equivalent definition of data-driven policies through the empirical probability measure.
Define the empirical probability  measure  $\hat{\nu}_{\envSamples}$  as
\begin{equation*}
\hat{\nu}_{\envSamples}(A) :=  \frac{1}{\numS} \sum_{i=1}^\numS \mathbbm{1} \left \{ \envVr_i \in A \right \}, \qquad \mbox{for every measurable set $A$.} 
\end{equation*}
In a setting in which the order of the samples does not matter (as the one we consider), to represent a policy taking as input  $\envSamples := (\envVr_1, \ldots , \envVr_\numS)$, one can instead consider mappings that take as an input the number of samples and the empirical measure. Indeed, $(\envVr_1, \ldots , \envVr_\numS)$ is only characterized by the values taken by the previous samples and the count of each value. Furthermore, the empirical probability measure gives the frequency of each value which can be multiplied by the sample size to obtain the actual count of appearances of values.

We therefore define a data-driven policy  $\pi$ as a mapping from $\mathbb{N}^* \times \mesSp$ to $\Delta \left( \actSp \right)$, the space of measures on the action space $\actSp$, that associates $(\numS,\hat{\nu})$ to a distribution of actions $\pi(\numS,\hat{\nu})$. 
We let $\Pi$ denote the set of all such mappings. We assume that such a mapping can be made with knowledge of the metric $d$ and radius $\epsilon$.
In \Cref{sec:apx_formal_policy} we provide a detailed discussion about the relation between $\Pi$ and the usual definition of data-driven policies.

\noindent \textbf{Performance through worst-case regret.} 
Given an  out-of-sample distribution $\mesOut \in \mesSp$,   the best achievable objective is $\opt(\mesOut)$. 
For a  sequence of historical distributions $\mesIn{1}, \ldots, \mesIn{\numS}$ in  $\mathcal{U}_\epsilon(\mesOut)$, the expected objective of a policy $\pi$ is
\begin{equation*}
\mathbb{E}_{\envVr_i \sim \mesIn{i}} \left[ \g[\pi(\numS,\hat{\nu}_{\envSamples})]{\mesOut} \right].
\end{equation*}

In turn, we define the expected (absolute) regret of a policy $\pi \in \Pi$ against an out-of-sample distribution $\mesOut \in \mesSp$ and a sequence of $\numS$ historical distributions $\mesIn{1},\ldots, \mesIn{\numS} \in  \mathcal{U}_\epsilon(\mesOut)$ as
\begin{equation*}
\mathcal{R}_\numS \left( \pi, \mesOut, \mesIn{1}, \ldots, \mesIn{\numS} \right) :=  \opt(\mesOut) - \mathbb{E}_{\envVr_i \sim \mesIn{i}} \left[ \g[\pi(\numS,\hat{\nu}_{\envSamples})]{\mesOut} \right] .
\end{equation*}
In this work, we will assess the performance of a  policy based on its \textit{worst-case expected regret}. Given an instance $\mathcal{I} = \left( \actSp, \envSp, \mesSp, d, g \right)$, we define the worst-case expected regret of a data-driven policy $\pi \in \Pi$ for a radius of heterogeneity $\epsilon \geq 0$ as
\begin{equation}
\label{eq:worst_case_reg}
\mathfrak{R}_{\mathcal{I},\numS}^{\pi}(\epsilon) := \sup_{\mesOut \in \mesSp } \sup_{ \mesIn{1},\ldots, \mesIn{\numS} \in \mathcal{U}_\epsilon(\mesOut)} \mathcal{R}_\numS \left( \pi, \mesOut, \mesIn{1}, \ldots, \mesIn{\numS} \right).
\end{equation}

We finally define the \textit{worst-case asymptotic regret} of a data-driven policy $\pi \in \Pi$ in an environment with radius of heterogeneity $\epsilon$ as
\begin{equation}
\label{eq:asymp_reg}
\mathfrak{R}_{\mathcal{I},\infty}^{\pi}(\epsilon) := \limsup_{\numS \to \infty}  \mathfrak{R}_{\mathcal{I},\numS}^{\pi}(\epsilon).
\end{equation}

We sometimes write $\mathfrak{R}_{K,\infty}^{\pi}(\epsilon)$ or $\mathfrak{R}_{W,\infty}^{\pi}(\epsilon)$ instead of $\mathfrak{R}_{\mathcal{I},\infty}^{\pi}(\epsilon)$ when the problem class is specified to highlight the dependence on the Kolmogorov or Wasserstein metrics.

\begin{remark}[Relations between different notions of heterogeneity]\label{rem:dom}
Note that many notions of distances or divergences are related to one another \citep{gibbs2002choosing}.
We show in \Cref{sec:apx_rel} that, in general, by leveraging relations between different distances one may relate the worst-case regret of data-driven decision-making instances in heterogeneous environments which have the same $\epsilon$ but differ along the type of distance used. In particular, when $\envSp$ is a bounded subset of $\mathbb{R}$, we obtain by specifying our result to the Kolmogorov and the Wasserstein distances that for any data-driven policy $\pi \in \Pi$ and any sample size $\numS$ and any radius of heterogeneity $\epsilon$, 
\begin{equation}
\label{eq:dom_reg}
\mathfrak{R}_{K,\numS}^{\pi}(\epsilon) \leq \mathfrak{R}_{W,\numS}^{\pi}(\diam{\envSp} \cdot \epsilon).
\end{equation}
This relation follows from the fact that these distances are related as for every $\mu_1$ and $\mu_2 \in \mesSp,$
\begin{equation*}
\label{eq:dom}
d_W \left(\mu_1,\mu_2 \right) \leq \diam{\envSp} \cdot   d_K \left(\mu_1,\mu_2 \right).
\end{equation*}
\end{remark}

\section{Reduction for policy evaluation}
\label{sec:reduction}

When trying to understand the performance of a particular data-driven policy,
a significant challenge in the analysis of Problem \eqref{eq:asymp_reg} stems from nature's problem. This problem requires to optimize over a general set of product measures which represent the possible sequences of historical probability measures $\mesIn{1},\ldots, \mesIn{\numS}$ and of an out-of-sample measure $\mesOut$. 

When considering Problem \eqref{eq:asymp_reg}, a natural approach would be to try to characterize the worst-case product measure as a function of $\numS$, and in turn the limiting behavior of this product measure as the product becomes infinite. Our next result shows that this approach is a priori challenging as nature will attempt to use different distributions for each of the past environments.  

\begin{proposition}
\label{prop:heterogeneity_helps}
There exists a data-driven decision problem in a heterogeneous environment $\mathcal{I} = \left( \actSp, \envSp, \mesSp, d, g \right)$ and a radius of heterogeneity $\epsilon > 0$ such that the SAA policy defined in \Cref{sec:SAA-analysis} satisfies,
\begin{equation*}
 \mathfrak{R}_{\mathcal{I},2}^{\SAA}(\epsilon) > \sup_{\mesOut \in \mesSp} \sup_{\nu\in \mathcal{U}_\epsilon(\mesOut)} \mathcal{R}_2 \left( \SAA, \mesOut, \nu,  \nu \right).
\end{equation*}
\end{proposition}
\Cref{prop:heterogeneity_helps} formalizes that the performance of SAA, a central data-driven policy formally introduced in \Cref{sec:SAA-analysis}, may be strictly worse when allowing for various historical distributions than when restricting attention to a fixed one. We prove this result by considering a variant of the ski-rental problem in which we restrict $\envSp$ to be three points. In this setting, we argue that the performance of SAA is the worse when observing two samples with different values and show that selecting two different point mass distributions is the unique worst-case sequence of historical distributions. 

While the previous result shows that nature will take full advantage of heterogeneity by selecting different distributions in-sample that are also different from the out-of-sample distribution, we will next establish that, under fairly general conditions, one can upper bound the asymptotic worst-case regret of a broad class of data-driven algorithms by considering a surrogate problem in which nature selects a common distribution for all of the past environments.
Recall, that a data-driven policy is a general mapping from $\mathbb{N}^{*} \times \mesSp$ to a potentially randomized decision in $\actSp$. We say that $\pi$ is sample-size-agnostic if its decision does not depend on $\numS$. Formally, a policy is sample-size-agnostic if for any $\hat{\nu} \in \mesSp$, we have that $\pi(\numS_1, \hat{\nu}) = \pi(\numS_2, \hat{\nu})$ for any $\numS_1, \numS_2 \in \mathbb{N}^* $, in which case we write $\pi(\hat{\nu})$ for brevity. This condition, implies that the policy only needs information on the empirical probability measure (i.e., the frequency of appearance of past samples) as opposed to the actual count of appearance.
The class of sample-size-agnostic policies encompasses central policies such as the central SAA policy analyzed in \Cref{sec:SAA-analysis} but also the distributionally robust optimization (DRO) policies commonly analyzed in the literature and which will be defined in \Cref{sec:UDRO}. We denote by $\agSp$ the subclass of sample-size-agnostic policies.

We define for a sample-size-agnostic policy $\pi \in \agSp$ the following optimization problem:
\begin{equation}
\label{eq:DRO_problem}
\mathfrak{R}_{\scaleto{\mathcal{I},\mathrm{DRO}}{6pt}}^{\pi} (\epsilon) := \sup_{\mesOut \in \mesSp} \sup_{\nu\in \mathcal{U}_\epsilon(\mesOut)} \left \{ \opt(\mesOut) - \g[\pi(\nu)]{\mesOut} \right \}. 
\end{equation}
While Problem \eqref{eq:DRO_problem} resembles Problem \eqref{eq:worst_case_reg}, it differs on two crucial dimensions:  $i.)$ in Problem \eqref{eq:DRO_problem}, nature selects  the \textit{same}  distribution $\nu$ in the heterogeneity ball for all past environments and, $ii.)$ in Problem \eqref{eq:DRO_problem}, the policy $\pi$ is assumed to know the actual distribution of the past environments $\nu$,  whereas in Problem \eqref{eq:worst_case_reg}, the policy only observed the empirical distribution of past realizations.

Next, we will establish that Problems \eqref{eq:worst_case_reg} and \eqref{eq:DRO_problem} are tightly connected under some mild assumptions. We define the following conditions on the distance $d$ defining the heterogeneity.

\begin{definition}[Empirical triangular convergence]
\label{def:etc}
We say that a distance $d$ on the space of probability measures satisfies the empirical triangular convergence (ETC) property if  for every triangular array sequence of probability measures $(\mu_{i,\numS})_{1 \leq i \leq \numS, \numS \in \mathbb{N}^*}$ all belonging to $\mesSp$, we have
\begin{equation*}
\lim_{\numS \to \infty}d(\hat{\mu}_\numS, \bar{\mu}_\numS) = 0 \qquad \text{a.s.},
\end{equation*}
where for every measurable set $A$, $\hat{\mu}_\numS(A) :=  \frac{1}{\numS} \sum_{i=1}^\numS \mathbbm{1} \left \{ \envVr_{i,\numS} \in A \right \}$ 
for samples $\envVr_{i,\numS} \sim \mu_{i,\numS}$ and $\bar{\mu}_\numS = \frac{1}{\numS} \sum_{i=1}^\numS \mu_{i,\numS}$. 
\end{definition}

\begin{definition}[Convexity] \label{def:convex}
A distance $d$ on a set $S$ is convex if for any $x$ in $S$, $d(\cdot,x)$ is convex.
\end{definition}

The ETC property requires that the sequence of empirical distributions converges to the average of ground truth distributions for arbitrary triangular arrays.  This assumption is  satisfied by many common distances and, we show in \Cref{sec:apx_ETC} that this property is satisfied by the Kolmogorov and Wasserstein distances. 
Convexity is also trivially satisfied by these two distances.

We now state our first main result.
\begin{theorem}[Performance Evaluation through Upper Bound]
\label{thm:reduction_stat}
Let $\mathcal{I} = \left( \actSp, \envSp, \mesSp, d, g \right)$ be a data-driven decision problem in a heterogeneous environment and let $\pi \in \Pi$ be a sample-size-agnostic policy. Assume that $g$ is bounded on $\actSp \times \envSp$ and that $d$ is a convex distance which satisfies the empirical triangular convergence property. Then
 \begin{equation} \label{eqn:reduction} 
 \mathfrak{R}_{\mathcal{I},\infty}^{\pi}(\epsilon) \leq \lim_{\eta \to \epsilon^{+}} \mathfrak{R}_{\scaleto{\mathcal{I},\mathrm{DRO}}{6pt}}^{\pi} (\eta). 
 \end{equation}
\end{theorem}

\Cref{thm:reduction_stat} establishes, that under some mild  conditions, in the asymptotic regime, one may restrict attention to problem \eqref{eq:DRO_problem} to evaluate a policy through an upper bound on its regret. This reduction significantly simplifies the analysis of sample-size-agnostic policies as problem \eqref{eq:DRO_problem} does not involve any form of randomness in the input received by the policy, and only allows the adversary to select a single distribution. As discussed in \Cref{prop:heterogeneity_helps}, the asymptotic regime is a key driver of this result since the reduction to a single distribution does not hold for finite samples. We also remark that considering  sample-size-agnostic policies is without loss of optimality (see \Cref{prop:lb_agnostic}).

We note that \citep[eq. (7)]{mohri2012new} derives a finite-sample bound on the regret of SAA in a similar setting to ours which can be used to derive an asymptotic bound for SAA. We discuss in more details the relation between our result, when applied to SAA, and theirs in \Cref{sec:SAA-analysis}.
Quite importantly, their result is only applicable to SAA while our reduction applies to \textit{any} policy in the wide range of sample-size-agnostic policies. In particular, this will be critical to derive policies which achieve minimax optimal asymptotic regret rates when SAA fails and is not (rate) optimal (see \Cref{sec:beyond_SAA}).

We prove \Cref{thm:reduction_stat} through a sample path analysis. We show that almost surely (over all possible historical samples observed), the asymptotic worst-case regret of a sample-size-agnostic policy can be bounded by the right-hand-side of \eqref{eqn:reduction}. In particular, we use the ETC property of the distance to show that asymptotically, the empirical cdf observed by the policy belongs to the heterogeneity ball of radius $\epsilon$. The proof also relies on the formalism we develop to define data-driven policies as discusses in \Cref{sec:apx_formal_policy}.

 We note that in settings where the ETC property fails, one may derive results on the performance of data-driven policies by using relations between different distances, as discussed in \Cref{rem:dom}. We discuss in \Cref{sec:apx_rel} how one can use the results established in the main body for the Kolmogorov and Wasserstein distances to other metrics and $\phi$-divergences such as the total variation or the KL-divergence.

\section{Analysis of SAA for integral probability metrics} \label{sec:SAA-analysis}
In the i.i.d. setting, i.e., when $\epsilon = 0$, a widely studied policy that achieves excellent performance across a vast range of data-driven decision-making problems is Sample Average Approximation (SAA). SAA is usually defined for empirical probability measures, but for the purpose of this work we extend its definition for arbitrary probability measures.\footnote{We refer the reader to \Cref{sec:apx_formal_policy} for a discussion regarding the definition of policies on probability measures that may not be empirical ones.} For every $\nu \in \mesSp$, SAA is defined as 
\begin{equation*}
 \SAA(\nu) = \argmax_{\actVr \in \actSp} \g[\actVr]{\nu},
\end{equation*}
whenever the extremum of the optimization problem is achieved. This is the case, whenever the function $\g[\cdot]{\nu}$ is upper semi-continuous and the action space is compact.
In the case of multiple optimal solutions, our results hold for an arbitrary choice of optimal solution.

Given the centrality of SAA both in theory and practice, it is natural to ask whether this policy performs well for data-driven problems in heterogeneous environments. In this section, we provide a general methodology to analyze  SAA and provide an understanding of some of the key drivers of its worst-case regret.

The analysis in this section is general and is motivated by the following upper bound on the worst-case regret of SAA, 
\begin{equation}
\label{eq:bound_mohri}
    \mathfrak{R}_{\mathcal{I},\infty}^{\SAA} (\epsilon) \leq 2 \sup_{\mesOut \in \mesSp} \sup_{\nu \in \mathcal{U}_{\epsilon}(\mesOut)}  \UB_{\mathcal{I}}(\mesOut,\nu),
\end{equation}
where for every $\mesOut, \nu \in \mesSp$,
\begin{equation}
\label{eq:uniform_bound}
\UB_{\mathcal{I}}(\mesOut,\nu) = \sup_{\actVr \in \actSp} \left | \g[\actVr]{\mesOut} - \g[\actVr]{\nu} \right|. 
\end{equation}
To explain, the inequality \eqref{eq:bound_mohri} is a trivial corollary of \Cref{thm:reduction_stat}   whenever the ETC property holds and the distance is convex (see proof of  \Cref{cor:dro_to_reg}). The \emph{uniform deviation} quantity defined in \Cref{eq:uniform_bound} is a quantity of interest in the statistical learning literature, which is commonly analyzed (in the special case where $\nu$ is an empirical distribution generated from i.i.d. samples of $\mesOut$) to prove that the finite-sample generalization error of a learning algorithm vanishes at a certain rate. Motivated by \eqref{eq:bound_mohri}, we will analyze the uniform deviation for arbitrary $\mesOut$ and $\nu$ to be able to derive bounds on the asymptotic worst-case regret of SAA in heterogeneous environments.

As mentioned in \Cref{sec:reduction}, \cite{mohri2012new} develops a finite-sample analysis of SAA in the heterogeneous setting and provides a finite-sample bound which involves a sequential notion of Rademacher complexity. In the settings where the Rademacher complexity vanishes their bound implies that  \eqref{eq:bound_mohri} holds asymptotically.

In the rest of this section, we derive a methodology to bound the performance of SAA for data-driven problems in heterogeneous environments when the distance defining the heterogeneity belongs to the class of integral probability metrics (IPM), formally introduced in \Cref{sec:IPM}.
In \Cref{sec:UB_IPM}, we develop a general analysis to bound $\UB_{\mathcal{I}}$ for IPMs. We introduce the approximation parameter, a key quantity which enables us to derive guarantees across problem classes and distances. We finally translate in \Cref{sec:SAA_bound} our bounds on $\UB_{\mathcal{I}}$ into bounds on the asymptotic worst-case regret of SAA and illustrate our methodology in the context of the newsvendor and the pricing problems.

As mentioned in \Cref{sec:relatedWork}, the results developed in this section are conceptually close to the method of probability metrics (see e.g. \cite{rachev2002quantitative}). We discuss in detail the relation to this work in \Cref{rem:comp_litt}.

\subsection{Integral probability metrics : definition and examples} \label{sec:IPM}

We focus on a broad class of metrics often referred to as \textit{probability metrics with a $\zeta$-structure} \citep{zolotarev1984probability} or \textit{integral probability metrics} \citep{muller1997integral,sriperumbudur2010hilbert}. In this work we will use the latter denomination. These metrics are defined as follows.
\begin{definition}[Integral Probability Metrics]
Given a class of functions $\mathcal{F}$ from $\envSp$ to $\mathbb{R}$, which we will refer to as a generating class of functions,
the associated integral probability metric $\dis{\mathcal{F}}$ is defined for every $\mesOut, \nu \in \mesSp$ as
\begin{equation} \label{eqn:IPM}
\dis{\mathcal{F}}(\mesOut,\nu) = \sup_{ f \in \mathcal{F}} \left | \int_{\envVr \in \envSp} f(\envVr) d\mu(\envVr) - \int_{\envVr \in \envSp} f(\envVr) d\nu(\envVr) \right|.
\end{equation}
\end{definition}
We note that different classes of functions may induce the same metric. Furthermore, many common metrics (that we define below) such as the Wasserstein distance or the Kolmogorov distance are in fact integral probability metrics.

\noindent \textbf{Wasserstein distance.}
For any metric space $\left( \envSp, d_{\envSp} \right)$,  the Wasserstein distance, denoted as $d_{W}$, is defined for all probability measures $\mu$, $\nu$  as,
\begin{equation*}
d_{W} (\mu,\nu) = \inf_{\gamma \in \Gamma(\mu,\nu)} \int_{\envVr, \envVr' \in \envSp} d_{\envSp} \left( \envVr, \envVr' \right) d \gamma(\envVr, \envVr'),
\end{equation*}
where $\Gamma(\mu,\nu)$ is the set of measures on $\envSp \times \envSp$ with marginal distributions $\mu$ and $\nu$.

The Wasserstein distance also has a dual representation as an integral probability metric \citep{kantorovich1958space,kantorovich2016functional}. Recall that for any function we denote by $\Lip{f}$ its Lipschitz constant and consider the generating class, 
$
\mathcal{W} = \{ f : \envSp \to \mathbb{R} \; \text{s.t.} \; \Lip{f} \leq 1 \}$. 
One can equivalently describe the Wasserstein distance as an integral probability metric by setting $\mathcal{F}=\mathcal{W}$ in~\eqref{eqn:IPM}.

\noindent \textbf{Kolmogorov distance.}
When $\envSp$ is a subset of $\mathbb{R}$, one may consider the Kolmogorov distance, denoted as $d_{K}$ and defined for all probability measures $\mu$ and $\nu$ as,
\begin{equation*}
d_{K} (\mu,\nu ) = \sup_{t \in \envSp} \left | \int_{-\infty}^t d\mu(\envVr) - \int_{-\infty}^t d\nu(\envVr) \right|.
\end{equation*}
The Kolmogorov distance is therefore the uniform distance between the cumulative distribution functions associated to $\mu$ and $\nu$. Furthermore, it is clear that, by considering $\mathcal{K} := \{ \envVr \in \mathbb{R} \mapsto \mathbbm{1} \left \{ \envVr \leq t  \right \} \, \vert \, t \in \mathbb{R} \}$, one may equivalently define the Kolmogorov as an integral probability metric by setting $\mathcal{F}=\mathcal{K}$ in~\eqref{eqn:IPM}. In what follows, whenever we refer to an instance in which the distance is the Kolmogorov one, we implicitly assume that the space $\envSp$ is a subset of $\mathbb{R}$.

\subsection{Bound on $\UB_{\mathcal{I}}$ under IPM heterogeneity}\label{sec:UB_IPM}
In this section, we develop a general analysis of $\UB_{\mathcal{I}}$ defined in \eqref{eq:uniform_bound} for different instances of data-driven decision making problems under heterogeneous environments, by relating it to the generating class of the integral probability metric of interest. 

\subsubsection{Bound on $\UB_{\mathcal{I}}$ through the approximation parameter} 

Let $\mathcal{F}$ be a class of functions and let $\dis{\mathcal{F}}$ be the induced integral probability metric.
We develop an upper bound for every $\mesOut, \nu \in \mesSp$ on $\UB_{\mathcal{I}}(\mesOut,\nu)$, defined in \eqref{eq:uniform_bound},  by relating the objective function $g$ to the generating class of functions $\mathcal{F}$.

First remark that if the objective function $g(x,\cdot):\Xi\to\mathbb{R}$ lies in $\mathcal{F}$ for all actions $x\in\mathcal{X}$, then by definition, $\sup_{\actVr \in \actSp}|\g[\actVr]{\mu}-\g[\actVr]{\nu}|\le d_{\mathcal{F}}(\mu,\nu)$.  More generally, for an action $\actVr \in \actSp$, if there exists $f \in \mathcal{F}$ and $\delta \geq 0$ such that $\|f - g(x,\cdot)\|_{\infty} \leq \delta$, we have that for any $\mesOut$ and $\nu$,
\begin{align*}
\left| \g[\actVr]{\mesOut} - \g[\actVr]{\nu} \right| &= \left| \int_{\envVr \in \envSp} g(\actVr, \envVr) d\mu(\envVr) - \int_{\envVr \in \envSp} g(\actVr, \envVr) d\nu(\envVr)\right|\\  
&\leq \left| \int_{\envVr \in \envSp} \left( g(\actVr, \envVr) - f(\envVr) \right)  d\mu(\envVr) - \int_{\envVr \in \envSp} \left( g(\actVr, \envVr) - f(\envVr) \right)  d\nu(\envVr) \right|\\
&\qquad + \left| \int_{\envVr \in \envSp} f(\envVr) d\mu(\envVr) - \int_{\envVr \in \envSp} f(\envVr) d\nu(\envVr) \right|\\
&\leq 2 \| f- g(x,\cdot) \|_{\infty} + \sup_{\tilde{f} \in \mathcal{F}}  \left| \int_{\envVr \in \envSp} \tilde{f}(\envVr) d\mu(\envVr) - \int_{\envVr \in \envSp} \tilde{f}(\envVr) d\nu(\envVr) \right|\\
&\leq 2\delta + \dis{\mathcal{F}}(\mesOut,\nu),
\end{align*}
where the last inequality holds because $\| f- g(x,\cdot) \|_{\infty}  \leq \delta$ and by definition of $\dis{\mathcal{F}}$.

Therefore, $\UB_{\mathcal{I}}$ can be bounded by studying how well the function class $\mathcal{F}$ approximates the family of objective functions $(g(\actVr,\cdot))_{\actVr \in \actSp}$. We formalize this notion by introducing the approximation parameter.
\begin{definition}[Approximation parameter]
\label{def:approx}
Given a class of functions $\mathcal{F}$, an objective function $g$ an action $\actVr \in \actSp$ and a real number $\delta \geq 0$, we define the approximation parameter $\lambda_{\actVr}(\mathcal{F},g,\delta)$ as,
\begin{subequations}
\label{eq:approx_param}
\begin{alignat}{2}
\lambda_{\actVr}(\mathcal{F},g,\delta) :=  &\! \inf_{(\lambda_i)_{i \in \mathbb{N}}, (f_i)_{i \in \mathbb{N}}}        &\qquad& \sum_{i=1}^\infty |\lambda_i|  \\
&\text{subject to} &      & \limsup_{m \to \infty} \left \| \sum_{i=1}^m \lambda_i \cdot f_i - g(x,\cdot) \right \|_{\infty} \leq \delta, \label{eq:approx_constraint} \\
&		    &     & f_i \in \mathcal{F}, \; \lambda_i \in \mathbb{R}.
\end{alignat}
\end{subequations}
Furthermore, we let $\lambda(\mathcal{F},g,\delta) =  \sup_{\actVr \in \actSp}  \lambda_{\actVr}(\mathcal{F},g,\delta)$.
\end{definition}
The approximation parameter quantifies the complexity required to approximate, within a uniform distance of at most $\delta$, all functions $(g(\actVr,\cdot))_{\actVr \in \actSp}$ with functions from the generating class $\mathcal{F}$\footnote{We note that the sequence of functions \((f_{i})_{i \in \mathbb{N}} \in \mathcal{F}^{\mathbb{N}}\) in \Cref{def:approx} is restricted to be countable. Formalizing an optimization problem over an uncountable set of functions as variables remains an open question.}.
The objective function $\sum_{i=1}^{\infty} |\lambda_i|$ represents the notion of complexity and the constraint \eqref{eq:approx_constraint} enforces that $g(\actVr,\cdot)$ is well-approximated. In the special case where $\delta=0$ and $g(x,\cdot)\in\mathcal{F}$, we note that $\lambda_{\actVr}(\mathcal{F},g,\delta)\le1$ (which can be obtained by setting $\lambda_1=1$, $\lambda_i = 0$ for $i \geq 2$ and letting $f_1 = g(\actVr,\cdot)$).
Furthermore, the constraint \eqref{eq:approx_constraint} is looser as $\delta$ grows therefore, $\lambda(\mathcal{F},g,\delta)$ is decreasing in $\delta$. Finally, we note that when problem \eqref{eq:approx_param} is infeasible for some $\actVr \in \actSp$, then $\lambda_{\actVr}(\mathcal{F},g,\delta) = +\infty$ (which is the standard value of an infeasible minimization problem), hence in that case, $\lambda(\mathcal{F},g,\delta) = + \infty.$

Our next main result relates the uniform deviation $\UB_{\mathcal{I}}$ to the approximation parameter of the problem.
\begin{theorem}[Bounding Uniform Deviation under IPM heterogeneity]
\label{thm:generating_class}
Let $\mathcal{F}$ be a class of functions and let $\dis{\mathcal{F}}$ be the associated integral probability metric.
Let $\mathcal{I} = \left( \actSp, \envSp, \mesSp, \dis{\mathcal{F}}, g\right)$ be a data-driven decision making problem under heterogeneity then for every $\mesOut, \nu \in \mesSp$, we have that,
\begin{equation*}
\UB_{\mathcal{I}}(\mesOut,\nu)  \leq   \inf_{\delta \geq 0} \left \{ \lambda \left( \mathcal{F}, g, \delta \right) \cdot \dis{\mathcal{F}}(\mesOut,\nu) + 2\delta \right \}.
\end{equation*}
\end{theorem}
\Cref{thm:generating_class} establishes an upper bound on $\UB_{\mathcal{I}}$ by relating the objective function with the structure of the heterogeneity distance through the approximation parameter. Therefore, the analysis of $\UB_{\mathcal{I}}$ (which will imply bounds on the performance of SAA) can be reduced to bounding the approximation parameter. We next show how to derive simple upper bounds on the approximation parameter for common IPMs.

\subsubsection{Bounds on the approximation parameter through the maximal generator}\label{sec:max_gen} 
For two generating classes of functions $\mathcal{F} \subset \mathcal{F}'$, an objective $g$ and $\delta \geq 0$, we note that the approximation parameter satisfies, $\lambda(\mathcal{F}',g,\delta) \leq \lambda(\mathcal{F},g,\delta)$. Furthermore, we note that the generating class is not unique for a given distance. Therefore, this motivates the study of the largest class of functions generating a particular distance in order to derive the tightest possible bounds. The notion of ``largest class''  has been extensively analyzed in \cite{muller1997integral} and is formally defined below.

\begin{definition}[Maximal Generator]
Let $\mathcal{F}$ be a generating class and let $\dis{\mathcal{F}}$ be the associated integral probability metric. The set $\gen{\mathcal{F}}$ of all functions $f$ satisfying,
\begin{equation*}
\left | \int_{\envVr \in \envSp} f(\envVr) d\mu(\envVr) - \int_{\envVr \in \envSp} f(\envVr) d\nu(\envVr) \right| \leq \dis{\mathcal{F}}\left(\mesOut, \nu \right) \quad \text{for all $\mesOut$ and $\nu \in \mesSp$,}
\end{equation*}
is called a maximal generator.
\end{definition}

\cite{muller1997integral} characterizes the maximal generator for many common integral probability metrics. In particular, they show the following.
\begin{lemma}[Common maximal generators (\cite{muller1997integral})]
\label{prop:common_max}
~
\begin{enumerate}
\item $\gen{\mathcal{K}}$ is the set of one-dimensional measurable functions $f$ with total variation $\V{f}  \leq 1$. 
\item $\gen{\mathcal{W}}$ is the set of measurable functions $f$ with Lipschitz constant $\Lip{f} \leq 1$.
\end{enumerate}
\end{lemma}
We note that the maximal generator for the Kolmogorov distance is different from the initial generating class $\mathcal{K}$, whereas the maximal generator for the Wasserstein distance corresponds to the initial class $\mathcal{W}$.

Furthermore, remark that for any function $f$ with bounded variation and bounded lipschitz constant, we have that, $\frac{f}{\V{f}} \in \gen{\mathcal{K}}$ and $\frac{f}{\Lip{f}} \in \gen{\mathcal{W}}$. Therefore, \Cref{prop:common_max} immediately implies the following bounds on the approximation parameter for common distances.

\begin{corollary}
\label{lem:bound_apx_param}
For any objective function $g$, we have that
\begin{equation*}
 \quad \lambda(\gen{\mathcal{W}}, g, 0) \leq \sup_{\actVr \in \actSp} \Lip{g(\actVr, \cdot)}.
 \end{equation*}
Furthermore, if $\envSp = [0,M]$ for some $M >0$, we have that,
\begin{equation*}
\lambda(\gen{\mathcal{K}}, g, 0) \leq \sup_{\actVr \in \actSp} \V{g(\actVr, \cdot)}.
\end{equation*}
\end{corollary}
\Cref{lem:bound_apx_param} enables to derive bounds on $\UB_{\mathcal{I}}$ by solely understanding analytical properties of each function in the family $(g(\actVr,\cdot))_{\actVr \in \actSp}$.
We note that while we stated this result for the two IPMs we are analyzing in this paper, this methodology readily generalizes to other distances such as the total variation distance, the Dudley metric, or for IPMs for which the function class $\mathcal{F}$ is the unit ball in a reproducing kernel Hilbert space \citep{sriperumbudur2010hilbert}. These other IPMs can be related to alternative analytical properties of the functions $(g(\actVr,\cdot))_{\actVr \in \actSp}$ (e.g. differentiability, continuity...).

On the flip side, \Cref{lem:bound_apx_param} does not enable to leverage the full generality of \Cref{thm:generating_class} as it only applies to settings where for all $\actVr \in \actSp$, $\frac{1}{L} g(\actVr,\cdot) \in \gen{\mathcal{F}}$ for some constant $L$ and $\lambda(\gen{\mathcal{F}}, g, 0) < \infty$. While this result is already sufficient to handle a broad class of problems, we highlight that there exist problem instances such that $\lambda(\gen{\mathcal{F}}, g, 0) = \infty$ and for which one can still derive meaningful bounds on $\UB_{\mathcal{I}}$ by bounding $\lambda(\gen{\mathcal{F}}, g, \delta)$ for $\delta > 0$. For sake of completeness, we next state a more general result which enables to derive a stronger geometric connection between the approximation parameter and the maximal generator.
Recall that, for a set $\mathcal{S}$, we denote by $\Span \left( \mathcal{S} \right)$ the set generated by finite linear combinations of elements from $\mathcal{S}$. Given, a maximal generator $\gen{\mathcal{F}}$, we define $\overline{\Span\left( \gen{\mathcal{F}} \right)}$ as the closure of the set $\Span \left( \gen{\mathcal{F}} \right)$ (where the closure is taken with respect to the topology induced by the supremum norm $\| \cdot \|_{\infty}$).
The following result provides a more detailed connection between the analytical properties of $g(\actVr,\cdot)$, the geometry of the generating class of function $\mathcal{F}$ and the approximation parameter.
\begin{proposition}[Geometric interpretation of the approximation parameter]
\label{prop:geometric_approximation}
Let $\mathcal{F}$ be a class of functions and let $\dis{\mathcal{F}}$ be the associated integral probability metric and let $g$ be an objective function. Then,
$g(\actVr,\cdot) \in \overline{\Span\left( \gen{\mathcal{F}} \right)}$ if and only if, for every $\delta > 0$, $\lambda_{\actVr} \left( \mathcal{F},g,\delta \right) < + \infty.$ 
\end{proposition}

\Cref{prop:geometric_approximation} shows that finiteness of the approximation parameter at $\actVr$ for arbitrarily small $\delta$ is equivalent to the fact that $g(\actVr,\cdot)$ is in the closure of $\Span\left( \gen{\mathcal{F}} \right)$ which topologically means that $g(\actVr,\cdot)$ is well approximated by a linear combination of functions in $\gen{\mathcal{F}}$. We note that \Cref{prop:geometric_approximation} considerably generalizes the simpler argument requiring that $g(\actVr,\cdot) \in \Span\left( \gen{\mathcal{F}} \right)$ to obtain that $\lambda_{\actVr}(\gen{\mathcal{F}}, g, 0) < \infty$. The next example illustrates how \Cref{prop:geometric_approximation} can be used to derive more general bounds by using analytical properties of $g(\actVr,\cdot)$.

\begin{example}
\label{example:more_complex}
Consider the maximal generator $\gen{\mathcal{K}}$ associated to the Kolmogorov distance. \Cref{lem:bound_apx_param} implies that if $g(\actVr,\cdot)$ has bounded variation, i.e., $V(g(\actVr,\cdot)) < \infty$, then $\lambda(\gen{\mathcal{K}}, g, 0) < \infty$. This corresponds to the case where $g(\actVr,\cdot) \in \Span \left( \gen{\mathcal{K}} \right)$. However, if $g(\actVr,\cdot)$ is only assumed to be continuous without any assumption on its variation, one can show that $g(\actVr,\cdot) \in \overline{\Span \left( \gen{\mathcal{K}} \right)}$ and conclude from \Cref{prop:geometric_approximation} that $\lambda_{\actVr}(\gen{\mathcal{K}}, g, \delta) < \infty$ for all $\delta > 0$. This argument, in combination to \Cref{thm:generating_class} leads to bounds on $\UB_{\mathcal{I}}$ in this setting. We expand more on this in \Cref{sec:apx_boudary}.
\end{example}
 
We next complement our upper bound results by a lower bound which allows to further relate the uniform deviation and the approximation parameter.
\begin{proposition}
\label{prop:lower_bound_approx}
Let $\mathcal{F}$ be a class of functions and let $\dis{\mathcal{F}}$ be the associated integral probability metric.
Let $\mathcal{I} = \left( \actSp, \envSp, \mesSp, \dis{\mathcal{F}}, g\right)$ be a data-driven decision making problem under heterogeneity. Assume that $\sup_{\actVr \in \actSp} \|g(\actVr,\cdot)\|_{\infty} < \infty$. If $\lambda(\gen{\mathcal{F}},g,0) = + \infty$, then 
\begin{equation*}
\limsup_{\epsilon \to 0} \frac{1}{\epsilon} \cdot \sup_{\mesOut \in \mesSp} \sup_{\mesIn \in \mathcal{U}_{\epsilon}(\mesOut)} \UB_{\mathcal{I}}(\mesOut,\nu) = \infty .
\end{equation*}
\end{proposition}
\Cref{prop:lower_bound_approx} shows that, to some extent, the approximation parameter provides a characterization of the behavior of the worst-case uniform deviation for small $\epsilon$. Indeed, when the approximation parameter is infinite for $\delta = 0$,  the worst-case uniform deviation is at least $\omega(\epsilon)$ when $\epsilon$ goes to $0$. We note that \Cref{prop:lower_bound_approx} does not preclude the uniform deviation to converge to $0$ at a slower rate than $\epsilon$. As a matter of fact, this can happen by bounding the approximation parameter for $\delta > 0$ as discussed in \Cref{example:more_complex}. We also acknowledge that \Cref{prop:lower_bound_approx} only provides a lower bound on the uniform deviation, which does not provide a lower bound on the original regret of interest---it only establishes a limitation of using the uniform deviation to analyze regret.

\begin{remark}[Comparison to the method of probability metrics]\label{rem:comp_litt}
\cite{rachev2002quantitative} derive a result conceptually close to the result we have established in this section. In particular, given a generating class $\mathcal{F}$, they establish that whenever $(g(\actVr,\cdot))_{\actVr \in \actSp} \subset \kappa \cdot \mathcal{F}$ for some $\kappa > 0$, one can bound the uniform deviation $\UB_{\mathcal{I}}(\mesOut,\nu)$ by $\kappa \cdot d_{\mathcal{F}}(\mesOut,\nu)$. Our analysis considerably expands this reasoning by $i)$ considering the maximal generator of a generating class and, $ii)$ allowing for $g(\actVr,\cdot)$ to be in the topological closure (with respect to the uniform convergence norm). Point $i)$ allows to assess whether the uniform deviation is sensitive to the heterogeneity structure  through analytical properties of the objective function. Furthermore, the lower bound in \Cref{prop:lower_bound_approx} shows that the maximal generator is in fact a driver of the behavior of $\UB_{\mathcal{I}}$. Finally, $ii)$ is required to establish guarantees in settings where the objective function needs to be approximated by the generating class as discussed in \Cref{sec:apx_boudary}.
\end{remark}

\subsection{Bounding scheme for SAA under IPMs and applications}\label{sec:SAA_bound}
We next leverage the relation between the performance of SAA and the quantity $\UB_{\mathcal{I}}$ (see \eqref{eq:bound_mohri}) along with the  analysis developed in \Cref{sec:UB_IPM} to derive a bound on the asymptotic worst-case regret of SAA.
In particular, we obtain the following results.
\begin{proposition}
\label{cor:dro_to_reg}
Let $\envSp = [0,M]$ for some $M > 0$. Let $\mathcal{I}_K = \left( \actSp, \envSp, \mesSp, d_K, g \right)$ and  $\mathcal{I}_W$  be instances of the data-driven decision problem under the Kolmogorov and Wasserstein distances, respectively.
Then for every $\epsilon \geq 0$,
\begin{align*}
 \mathfrak{R}_{W, \infty}^{\SAA}(\epsilon) &\leq 2 \sup_{\actVr \in \actSp} \Lip{g(\actVr, \cdot)} \cdot \epsilon\\
 \mathfrak{R}_{K, \infty}^{\SAA}(\epsilon) &\leq 2 \sup_{\actVr \in \actSp}  \V{g(\actVr, \cdot)} \cdot \epsilon.
\end{align*}
\end{proposition}
\Cref{cor:dro_to_reg} translates the bounds derived on $\UB_{\mathcal{I}}$ through the approximation parameter in \Cref{sec:UB_IPM} into bounds on the worst-case asymptotic regret of SAA. It implies that for a broad class of problems, the asymptotic worst-case regret of SAA vanishes as the radius of the heterogeneity ball goes to $0$. More specifically,  the asymptotic worst-case regret scales linear for the Kolmogorov distance whenever the objective function $g$ has uniformly bounded variation, and for the Wasserstein distance 
whenever the objective function $g$ is uniformly Lipschitz continuous. 

We prove this result by first showing that \eqref{eq:bound_mohri} holds using \Cref{thm:reduction_stat}. We then leverage the results developed in \Cref{sec:UB_IPM}.

We next illustrate the application of \Cref{cor:dro_to_reg} for two central problems. 

\subsubsection{Newsvendor problem: performance of SAA} \label{sec:newsvendor_SAA}

We first consider the newsvendor problem, a widely studied model in operations. 
Using $M>0$ to denote an upper bound on demand, we let $\envSp=\actSp=[0,M]$, and our regret bounds may (or may not) depend on $M$. The objective function is the newsvendor cost which is parameterized by two quantities: $c_u$ the underage cost and $c_o$ the overage cost. The cost of carrying an inventory level $\actVr \in \actSp$ and then observing demand $\envVr \in \envSp$ is,
\begin{equation} \label{eqn:defNV}
g ( \actVr, \envVr) = c_u \left( \envVr - \actVr \right)^{+} + c_o \left( \actVr - \envVr \right)^{+}, 
\end{equation}
where $(\cdot)^+:=\max\{\cdot,0\}$ is the positive part operator. The measure space $\mesSp$  is the set of probability measures on $\envSp$ without restrictions and for every $\mesOut \in \mesSp$, we define the cost of the oracle as
$\opt(\mesOut) = \inf_{\actVr \in \mathcal{X}} \g[\actVr]{\mesOut}.$\footnote{Our framework was introduced for maximization problems (see \Cref{sec:formulation}) but it readily extends to minimization problems.} 

We note that for the newsvendor objective $g$ defined in \eqref{eqn:defNV}, we have that $g$ lies in the linear span of the maximal generator class for Kolmogorov and Wasserstein distances. In particular,
\begin{equation*}
\sup_{\actVr \in \actSp} \V{g(\actVr, \cdot)} =  \max \left( c_u, c_o \right) \cdot M \quad \text{and} \quad \sup_{\actVr \in \actSp}  \Lip{g(\actVr, \cdot) } = \max \left( c_u, c_o \right)
\end{equation*}
where the suprema are attained at some action $x\in\{0,M\}$.
Therefore, one can derive a vanishing regret guarantee for SAA under these two forms of heterogeneity by exploiting \Cref{cor:dro_to_reg}. Our next result formalizes this and presents upper bounds on the asymptotic worst-case regret of SAA. We also derive
matching lower bounds on the minimax achievable regret by any data-driven policy.
\begin{proposition}[Heterogeneous newsvendor]
\label{cor:SAA_Newsvendor}
For the data-driven newsvendor  problem in a heterogenous environment under the Kolmogorov and the Wasserstein distances, we have for $\epsilon$ small enough that,
\begin{align*}
\frac{c_u + c_o}{2} \cdot M \cdot \epsilon&\leq  \inf_{\pi \in \Pi}  \mathfrak{R}_{K,\infty}^{\pi}(\epsilon) \leq  \mathfrak{R}_{K,\infty}^{\SAA}(\epsilon) \leq 2 \max\left(c_u,c_o \right) \cdot M \cdot \epsilon,\\
\frac{c_u + c_o}{2} \cdot \epsilon &\leq \inf_{\pi \in \Pi}  \mathfrak{R}_{W,\infty}^{\pi}(\epsilon) \leq  \mathfrak{R}_{W,\infty}^{\SAA}(\epsilon) \leq  2 \max\left(c_u,c_o \right) \cdot \epsilon.
\end{align*}
\end{proposition}
\Cref{cor:SAA_Newsvendor} first shows that SAA is robust to deviation under multiple forms of heterogeneity and that the regret scales linearly with the radius of the heterogeneity ball for both distances. We show furthermore that the linear dependence is, for the newsvendor problem, the best rate achievable for any data-driven policy across these forms of heterogeneity.
  
The upper bounds derived in \Cref{cor:SAA_Newsvendor} naturally follow from \Cref{cor:dro_to_reg}.
We also note that the lower bound obtained for the newsvendor depends linearly on the total variation and the Lipschitz constant through the factor $(c_u + c_o)$. This therefore shows that these characterizing quantities for the function $g$ are indeed respectively driving the regret of SAA in this case.

It is also worth noting that for the newsvendor problem, there is a significant difference between considering an initial ``naïve'' generating class for an integral probability metric and the maximal generator. For instance, when considering the Kolmogorov distance, it would have been more challenging to remark that the newsvendor cost can be well approximated by the generating class $\mathcal{K}= \{ \envVr \in \mathbb{R} \mapsto \mathbbm{1} \left \{ \envVr \leq t  \right \} \, \vert \, t \in \mathbb{R} \}$, whereas it is straightforward when considering $\gen{\mathcal{K}}$ the set of functions with bounded variation. 

While SAA achieves vanishing regret for both forms of heterogeneity in the newsvendor problem, we next analyze the performance of SAA for pricing and show that SAA may  fail in some cases.

 \subsubsection{Pricing problem: performance of SAA and potential failure}\label{sec:pricing_SAA}
We next consider the classical data-driven pricing.
The environment space $\envSp =  [0,M]$ represents the set of possible values. The set of possible prices $\actSp$ is $[0,M]$. The set of considered probability measures $\mesSp$ is the whole set of probability measures on $[0,M]$ without restriction. The revenue generated by the DM with a price decision $\actVr \in \actSp$ when facing a customer with value $\envVr \in \envSp$ is,
\begin{equation}
\label{eq:defP}
g(\actVr,\envVr) = \actVr \cdot \mathbbm{1} \left \{ \envVr \geq \actVr \right \}. 
\end{equation}
Therefore for any probability measure $\mesOut \in \mesSp$, we have that 
$\g[\actVr]{\mesOut} = \actVr \cdot \mathbb{P}_{\mesOut} \left( \envVr \geq  \actVr \right).$ The goal of the DM is to set a price $\actVr$ in order to maximize its revenue. Equivalently, the goal is to minimize the regret against the oracle which posts the optimal price given a wtp distribution.

We note that for the pricing objective $g$ defined in \eqref{eq:defP}, we have that $g$ lies in the linear span of the maximal generator class for Kolmogorov. In particular,
\begin{equation*}
\sup_{\actVr \in \actSp} \V{g(\actVr, \cdot)} =   M,
\end{equation*}
where the supremum for $\V{g(\actVr, \cdot)}$ is attained for $\actVr = M$. Therefore, \Cref{cor:dro_to_reg} implies directly an upper bound on the asymptotic worst-case performance of SAA, which we complement by a lower bound on the performance of any data-driven policy. 
\begin{proposition}[Kolmogorov Pricing]
\label{prop:Kolmogorov_pricing}
For the data-driven pricing problem in a heterogenous environment under the Kolmogorov distance, for any $\epsilon$ small enough we have,
\begin{align*}
\frac{M \cdot \epsilon}{2} &\leq \inf_{\pi \in \Pi}  \mathfrak{R}_{K,\infty}^{\pi}(\epsilon) \leq  \mathfrak{R}_{K,\infty}^{\SAA}(\epsilon) \leq 2 M \cdot \epsilon.
\end{align*}
\end{proposition}

\noindent \textbf{Failure of SAA.}
Note that for the pricing problem $g(\actVr,\cdot)$ is not Lispchitz continuous for any $\actVr > 0$ (in fact it is not even continuous) and hence does not lie in $\Span \left( \gen{\mathcal{W}} \right)$. Therefore, for this problem $\lambda( \gen{\mathcal{W}},g,0) = \infty$ and \Cref{prop:lower_bound_approx} implies that $\UB_{\mathcal{I}}$ is at least supralinear for small $\epsilon$. While, \Cref{prop:lower_bound_approx} does not provide a lower bound on SAA itself (because $\UB_{\mathcal{I}}$ is an upper bound on the performance of SAA), we formally derive in our next result a lower bound on the asymptotic worst-case regret of SAA.
\begin{proposition}[Failure of SAA for Wasserstein pricing]
\label{thm:SAA_Wasserstein_pricing}
For the data-driven pricing problem in a heterogenous environment under the Wasserstein distance, we have for $\epsilon > 0$ that 
\begin{equation*}
 \mathfrak{R}_{W,\infty}^{\SAA}(\epsilon) = M.
\end{equation*}
\end{proposition}
\Cref{thm:SAA_Wasserstein_pricing} shows that, for the pricing problem,  the asymptotic regret of SAA  under Wasserstein heterogeneity does not shrink to zero as $\epsilon$ goes to zero. As a matter of fact, the  regret of $M$ is the worst possible, as the revenue of the oracle is bounded above by $M$ and the revenue of any data-driven policy is bounded below by $0$. 

The key challenge in pricing is that by shifting the willingness-to-pay distribution from a point mass $\nu = \delta_1$ to a point mass $\mesOut = \delta_{1-\epsilon}$, the revenue obtained by posting the price $\$1$ drops from $\$1$ to $\$0$. Furthermore, the Wasserstein distance between two point masses which are close enough is small (as opposed to the Kolmogorov distance which is $1$ in that case).
Quite notably, this result holds for any radius of heterogeneity $\epsilon$, and an arbitrarily small deviation from the i.i.d. case leads to extremely poor performance of SAA, even with infinite data.

\subsubsection{Additional applications}
While the purpose of the above concrete examples was to establish the variety of cases that many emerge, the framework readily applies to other problem classes such as, e.g., the classical ski-rental problem, or multi-dimensional problems such as Bayesian Mechanism Design. We elaborate on these in \Cref{sec:ski} and \Cref{sec:apx_auction} respectively.

\section{Alternative policies: design and analysis}
\label{sec:beyond_SAA}
As seen in \Cref{sec:pricing_SAA}, the performance of SAA is poor for certain instances even with infinitely many samples. Our goal in this section is twofold: first, we complement the discussion sparked by \Cref{thm:SAA_Wasserstein_pricing} by investigating whether there exists a policy which can ensure a vanishing regret for the pricing problem under Wasserstein heterogeneity. Then, we discuss open directions towards a principled approach to design and analyze policies beyond SAA.

\subsection{Rate-optimal policy for the Wasserstein pricing problem}\label{sec:rate_opt_pricing}

Recall the pricing problem with Wasserstein distance introduced in \Cref{sec:pricing_SAA}. For this problem, we have seen that SAA incurs an asymptotic worst-case regret which is not vanishing as $\epsilon$ goes to $0$. 
We next present an alternative sample-size-agnostic policy for which the asymptotic worst-case vanishes as $\epsilon$ goes to $0$. Furthermore, we characterize the worst-case performance of that policy, showing that it has the best possible dependence with respect to $\epsilon$.

To be able to design and analyze these alternative policies, \Cref{thm:reduction_stat} is critical because $i)$ any analysis involving the uniform deviation $\UB_{\mathcal{I}}$ will lead to a non-vanishing upper bound (since the worst-case regret of SAA is a lower bound on $\UB_{\mathcal{I}}$) and, $ii)$ we need to derive upper bounds on $\mathfrak{R}_{W,\infty}^{\pi}(\epsilon)$ for policies $\pi$ for which standard upper bounding schemes are not present in the literature (the analysis of \cite{mohri2012new} only applies to SAA).

Fix $\delta \in \mathbb{R}$. For every $\nu \in \mesSp$ and $\numS \in \mathbb{N}^{*}$, we define the $\delta$-SAA policy as,
\begin{equation*}
\dev{\delta}(\numS,\nu) = \SAA(\nu) - \delta,
\end{equation*}
where we project the decision on the closest point in $\actSp$ if it is outside the action space. Given, that SAA is a sample-size-agnostic policy, it is clear that $\dev{\delta}$ is also sample-size-agnostic. Therefore, in what follows we drop the dependence in the sample size.  Our next result shows, that for an appropriate choice of $\delta$, the policy $\dev{\delta}$ can achieve a rate-optimal performance.

\begin{proposition}[Rate-optimal regret for Wasserstein pricing]
\label{thm:ub_pricing_W}
Consider the data-driven pricing problem in a heterogenous environment under the Wasserstein distance. For any $\epsilon$ small enough, let $\delta = \sqrt{M \cdot \epsilon}$.  Then,
\begin{equation*}
\mathfrak{R}_{W,\infty}^{\dev{\delta}}(\epsilon) \leq 4 \cdot \sqrt{M \cdot \epsilon}.
\end{equation*}
Furthermore,
\begin{equation*}
\frac{\sqrt{M \cdot \epsilon}}{4}  \leq \inf_{\pi \in \Pi} \mathfrak{R}_{W,\infty}^{\pi}(\epsilon).
\end{equation*}
\end{proposition}
\Cref{thm:ub_pricing_W} shows that the asymptotic worst-case regret of $\dev{\delta}$ scales as $\mathcal{O} \left( \sqrt{M \cdot \epsilon} \right)$ when $\epsilon$ goes to $0$. It highlights the need to adapt to the heterogeneity structure in pricing. Quite notably, the significant loss of performance of SAA can be mitigated by slightly robustifying the SAA decision and by deflating the SAA price.

\Cref{thm:ub_pricing_W} also implies that the dependence of the asymptotic worst-case regret of $\dev{\delta}$ in terms of $\epsilon$ is optimal across all data-driven policies. Therefore, this very mild modification of the decision of SAA does not only improve performance but actually leads to a near-optimal one.
We finally remark that the $\dev{\delta}$ policy improves over SAA when one can make use of the knowledge of $\epsilon$ and $d$. It is worth noting that SAA does not need such knowledge.
A natural question, which is beyond the scope of the current paper, would be to understand the best achievable performance without knowledge of $\epsilon$ and/or $d$.

\noindent \textbf{Proof sketch.}
To prove \Cref{thm:ub_pricing_W} we first leverage \Cref{thm:reduction_stat} to reduce the analysis of $\dev{\delta}$ to the DRO regret problem. We then decompose the DRO regret objective for any policy $\pi$ as follows,
\begin{align} \label{eqn:1809}
\opt(\mesOut)-\g[\pi(\nu)]{\mu}
= [\opt(\mesOut) - \opt(\nu) ]   + [\opt(\nu)
-\g[\pi(\nu)]{\mu} ].
\end{align}
The decomposition in \eqref{eqn:1809} starkly differs from the bound derived in \eqref{eq:bound_mohri} to analyze SAA. Indeed, the latter involves the uniform difference $\UB_{\mathcal{I}}(\mesOut,\nu)$, quantifying the difference in performance when changing the distribution from $\mesOut$ to $\nu$ for a \textit{fixed action} $\actVr$, whereas the former involves two differences where both the action and the distribution change.

To bound the difference in objectives when changing both actions and distributions we develop a structural result which leverages properties of the pricing objective.
In particular, the key observation is to remark that for the pricing problem the revenue can not abruptly \textit{decrease} when decreasing prices (a property that is false when increasing prices). Using this notable fact, one can control the difference in revenues at two different prices.  Formally, the following bound holds.
\begin{lemma}
\label{prop:diff_rev}
Consider the pricing objective.
Let $\mesOut$ and $\nu\in \mesSp$ be two different distributions and let $\actVr_1$ and $\actVr_2$ be two different prices such that $\actVr_1 < \actVr_2$. We have that,
\begin{equation*}
\g[\actVr_2] {\nu} - \g[\actVr_1] {\mesOut}   \leq (\actVr_2 - \actVr_1) + M \cdot \frac{d_W \left( \mesOut, \nu\right)}{\actVr_2 - \actVr_1}.
\end{equation*} 
\end{lemma}

\Cref{prop:diff_rev} unveils an interesting tradeoff faced by a DM who prices in a heterogeneous environment. By selecting appropriately the difference between two prices $\actVr_1$ and $\actVr_2$, one can ensure that the gap in revenues for two different distributions can be controlled by the Wasserstein distance. Intuitively, one should think of $\nu$ as the historical empirical distribution, $\mesOut$ as the out-of-sample distribution, $\actVr_2$ as the optimal decision when facing $\nu$, and $\actVr_1$ as the actual price posted by the DM. In that case, \Cref{prop:diff_rev} implies that by deflating appropriately, the revenue generated by the posted price $\actVr_1$ against the out-of-sample distribution $\mesOut$ cannot be much lower than the optimal revenue generated for distribution $\nu$.

We leverage \Cref{prop:diff_rev} to show that
\begin{equation}
\label{eq:diff_opt}
\opt \left(\mesOut\right) - \opt \left( \nu \right) \leq 2 \sqrt{M \cdot d_{W}\left(\mesOut,\nu \right)}.
\end{equation}
We then conclude the proof by using  \Cref{prop:diff_rev} again to show that for an appropriately chosen $\delta$, $\opt(\nu) -\g[\dev{\delta}(\nu)]{\mu} \large$ is also order $\mathcal{O}\left(\sqrt{M \cdot d_{W}\left(\mesOut,\nu \right)} \right)$ which concludes the proof.

 We note that related arguments allow us to design alternative policies for the ski-rental with Wasserstein distance for which SAA performs poorly. Furthermore, quite notably the performance of SAA can be significantly improved upon  for the ski-rental problem, even under the Kolmogorov distance. We discuss in details the derivation of alternative policies for ski-rental in \Cref{sec:ski}.

\subsection{Towards a general method to design policies}\label{sec:UDRO}
In \Cref{sec:rate_opt_pricing} we have seen that there are problems, such as the pricing problem with Wasserstein distance, for which one can obtain significantly better performances than SAA by using a specific alternate policy to SAA. In this section we propose initial steps towards a general paradigm to design and analyze policies other than SAA in heterogeneous environments. 

\noindent \textbf{Regret DRO policy.}  \Cref{thm:reduction_stat}
suggests a systematic way to design policies which perform well in heterogeneous environments when the number of samples becomes large---minimize the upper bound constructed in \Cref{thm:reduction_stat}.
This implies solving for every $\epsilon \geq 0$ the problem,
\begin{align*}
\inf_{\pi \in \agSp} \mathfrak{R}_{\scaleto{\mathcal{I},\mathrm{DRO}}{6pt}}^{\pi} (\epsilon) &= \inf_{\pi \in \agSp} \sup_{\mesOut \in \mesSp} \sup_{\nu\in \mathcal{U}_\epsilon(\mesOut)}  | \opt(\mesOut) - \g[\pi(\nu)]{\mesOut} | \notag \\
&= \sup_{\nu\in \mesSp} \inf_{\pi \in \agSp} \sup_{\mesOut \in \mathcal{U}_\epsilon(\nu)} | \opt(\mesOut) - \g[\pi(\nu)]{\mesOut} |, \label{eq:unif_DRO}
\end{align*}
where the last equality holds because the policy observes the historical distribution. 
In settings where the infimum is achieved, the previous equality suggests the following sample-size agnostic policy. For every $\nu \in \mesSp$,
\begin{equation}
\label{eq:RDRO_policy}
\pi^{\mathrm{RDRO}}(\nu) = \argmin_{\rho \in \Delta \left(\actSp \right)} \sup_{\mesOut \in \mathcal{U}_\epsilon(\nu)} \left\{ \opt(\mesOut) - \g[\actVr]{\mesOut} \right\} .
\end{equation}
We note that $\pi^{\mathrm{RDRO}}$ resembles the common DRO policies studied in the literature (and formally defined in \eqref{eq:DRO_policy}) to the difference that the problem in \eqref{eq:RDRO_policy} involves a \textit{regret objective}. 
We show in \Cref{prop:optimal_RDRO} (see \Cref{sec:apx_optimal_RDRO}) that the policy $\pi^{\mathrm{RDRO}}$ which emerges from optimizing the surrogate DRO regret is in fact optimal for the asymptotic worst-case regret $\mathfrak{R}_{\mathcal{I},\infty}$.
We note that $\pi^{\mathrm{RDRO}}$ is a priori challenging to compute as it requires to solve a minimax non-convex, infinite dimensional optimization problem. We leave as open the exciting computational question regarding efficient implementations of \eqref{eq:RDRO_policy}.

\noindent \textbf{Absolute DRO policy.}
A natural policy in heterogeneous environment which has been widely studied is the literature and ressembles the policy defined in \eqref{eq:RDRO_policy} is the DRO policy defined as,
\begin{equation}
    \label{eq:DRO_policy}
\pi^{\mathrm{DRO}}(\nu) = \argmax_{\rho \in \Delta(\actSp)} \inf_{\mesOut \in \mathcal{U}_\epsilon(\nu)} \g[\rho]{\mesOut}.
\end{equation}
Methods to compute the solution of the max-min problem when the maximization is over deterministic actions\footnote{In general taking randomized actions in the max-min problem can strictly improve the performance of the policy \citep{delage2019dice}.} have been developed under different assumptions on $\mathcal{G}$ and for different structures of the sets $\mathcal{U}_{\epsilon}$ \citep{rahimian2019distributionally}. 
Our next result shows that the standard $\pi^{\mathrm{DRO}}$ policy has an asymptotic worst-case regret ``close'' to the one of the minimax optimal policy.
\begin{proposition}
\label{prop:RDRO_to_DRO}
Let $d$ be a convex distance satisfying the ETC property and let $\mathcal{I} = \left( \actSp, \envSp, \mesSp, d, g \right)$ be a data-driven decision problem in a heterogeneous environment. Then for every $\epsilon \geq 0$, we have,
\begin{equation*}
 \lim_{\eta \to \epsilon^{-}}\mathfrak{R}_{\mathcal{I},\infty}^{\pi^{\mathrm{DRO}}}(\eta) \leq  \inf_{\pi \in \Pi} \mathfrak{R}_{\mathcal{I},\infty}^{\pi}(\epsilon)+ \sup_{\mesOut \in \mesSp}  \sup_{\nu \in \mathcal{U}_{2\epsilon}(\mesOut)} \left \{ \opt(\mesOut) - \opt(\nu) \right \} 
\end{equation*}
\end{proposition}

\Cref{prop:RDRO_to_DRO} shows that the difference between the performance of the commonly studied policy $\pi^{\mathrm{DRO}}$ and the best asymptotic performance among all data-driven policies is at most $\sup_{\mesOut \in \mesSp}  \sup_{\nu \in \mathcal{U}_{2\epsilon}(\mesOut)} \left \{ \opt(\mesOut) - \opt(\nu) \right \}$ (when $\mathfrak{R}_{\mathcal{I},\infty}^{\pi^{\mathrm{DRO}}}$ is continuous at $\epsilon$). Furthermore, \Cref{prop:RDRO_to_DRO} suggests a general methodology to prove that the $\pi^{\mathrm{DRO}}$ policy achieves the best dependence in $\epsilon$ when $\epsilon$ is small. Indeed, assume that one can establish that, there exists some constant $C >0$ such that for every $\epsilon$ small enough,
\begin{equation}
\label{eq:condtion_rate_opt}
\sup_{\mesOut \in \mesSp}  \sup_{\nu \in \mathcal{U}_{\epsilon}(\mesOut)} \left \{ \opt(\mesOut) - \opt(\nu) \right \}  \leq C \cdot \inf_{\pi \in \Pi} \mathfrak{R}_{\mathcal{I},\infty}^{\pi}(\epsilon),
\end{equation}
then one can conclude that for $\epsilon$ small enough the asymptotic worst-case regret of $\pi^{\mathrm{DRO}}$ is at most a constant factor away from the optimal achievable performance. Thus, in that case $\pi^{\mathrm{DRO}}$  has the best possible dependence in terms of $\epsilon$ achievable by any data-driven policy .

We note that \Cref{prop:RDRO_to_DRO} can be used to prove that the DRO policy achieves rate-optimality for the pricing problem with Wassserstein distance. Indeed, the lower bound derived in \Cref{thm:ub_pricing_W} along with the bound established in \eqref{eq:diff_opt} implies that \eqref{eq:condtion_rate_opt} holds.

This section highlights that $\pi^{\mathrm{DRO}}$ is a strong candidate to be a general rate-optimal policy. Similarly to $\pi^{\mathrm{RDRO}}$, we leave as open the question of being able to \textit{compute} the $\pi^{\mathrm{DRO}}$ policy.

\section{Open questions and concluding remarks} \label{sec:conclusion}

All in all, our framework enables us to develop a common language to compare achievable performance across central classes of problems and to unveil novel insights about the performance of the SAA policy, when slightly deviating from the widely studied i.i.d. regime. 
In settings where SAA fails, we also design algorithms achieving rate-optimal asymptotic guarantees. 
A key takeaway of this analysis across a broad class of problems and for different heterogeneity structures is that it is necessary to understand the structure of the problem we are facing but also the nature of the heterogeneity when designing data-driven policies that are robust to these environments.
While our work mainly focuses on the Kolmogorov and Wasserstein distances to illustrate our framework, we note that one can naturally extend it for $\phi$-divergences or study alternative metrics. We discuss in \Cref{sec:apx_rel} how our results for these two distances imply results for the total variation and the KL-divergence.
This work also opens many additional questions. We summarize below what we believe are exciting future research directions. 

First, \Cref{thm:reduction_stat} applied to SAA establishes an upper bound which is at least as tight as the previous one present in the literature \citep{mohri2012new} which involves the uniform deviation (see \eqref{eq:bound_mohri}). In this work we mainly  used \Cref{thm:reduction_stat} to analyze non-SAA policies for which  no bound were previously established, but our general analysis of SAA in \Cref{sec:SAA-analysis} still goes through the uniform deviation. While this was sufficient to derive tight guarantees on the rate of the worst-case regret, a natural question is whether using the tighter upper bound in \Cref{thm:reduction_stat} can lead to improved guarantees. We show in \Cref{sec:apx_strict_inequality} that, for the newsvendor problem under Kolmogorov heterogeneity, the bound derived by \Cref{thm:reduction_stat} is \textit{strictly} tighter than \eqref{eq:bound_mohri} which leaves open then question of using our theorem to derive problem-specific guarantees on SAA with better constant factors.

While the theory developed in \Cref{sec:SAA-analysis} was sufficient to establish tight regret guarantees for the problem classes and heterogeneity structures explored in this work, an important question remains unanswered: does having a low approximation parameter truly constitute a necessary condition for SAA to perform well? The general lower bound presented in \Cref{prop:lower_bound_approx} does not fully address this, leaving open the opportunity for future research to clarify the sensitivity of empirical optimization methods to heterogeneity.

Furthermore, our work focuses mainly on the evaluation of data-driven policies and the understanding of the achievable performance across heterogeneity structures and problem classes. It leaves open the question of designing \textit{general} alternative data-driven policies which can be efficiently computed and which perform well. \Cref{sec:UDRO} presents results which suggest that both the RDRO and the DRO policies are very good candidate policies in terms of achievable performance but it leaves open the question of computing these policies, or deriving approximations which would still perform well.

Finally, this paper opens up broader questions. Our analysis characterizes the performance in the asymptotic regime where the sample size grows large, and could be complemented by quantifying the performance of policies with finite samples in heterogeneous environments. 
Another key question would be to understand whether there exists a policy that does not use the knowledge of the type of heterogeneity and performs well across heterogeneity types. We believe that additional exciting and practical complements to this work include incorporating contexts to historical samples,  deriving statistical tests to characterize the type of heterogeneity along with its radius, and potentially providing an empirical validation of such procedures.

{\setstretch{.8}
\bibliographystyle{agsm}
\bibliography{ref}}

\newpage
\appendix

\setstretch{1.2}
\renewcommand{\theequation}{\thesection-\arabic{equation}}
\renewcommand{\theproposition}{\thesection-\arabic{proposition}}
\renewcommand{\thelemma}{\thesection-\arabic{lemma}}
\renewcommand{\thetheorem}{\thesection-\arabic{theorem}}
\pagenumbering{arabic}
\renewcommand{\thepage}{App-\arabic{page}}

\setcounter{equation}{0}
\setcounter{proposition}{0}
\setcounter{lemma}{0}
\setcounter{theorem}{0}

\part{Appendix} 
\parttoc 

\section{Preliminary results and remarks}

\subsection{Formal description of data-driven policies} \label{sec:apx_formal_policy}
In this section we formally present the definition of data-driven policies as mappings from distributions (as opposed to samples) to actions.
The classical definition of a data-driven policy is a sequence of mappings $(\pi_\numS)_{\numS \in \mathbb{N}}$ such that for every $\numS \in \mathbb{N}$, $\pi_{\numS}: \envSp^{\numS} \to \Delta(\actSp)$. We denote by $\Pi_{\mathrm{standard}}$ the set of such sequences.

As discussed in \Cref{sec:formulation}, we adopt in this work a different formalism where we define a data-driven policy as a mapping from $\mathbb{N}^* \times \mesSp \to  \Delta(\actSp)$. We next formally define the correspondence between these two definitions. We denote by $\Pi$ the set of such mappings.

We first define the sequence of mappings $(\phi_k)_{k \in \mathbb{N}^*}$, as follows. For every $k \geq 1$,
\begin{equation*}
\phi_k: \begin{cases}
 \qquad \quad \envSp^k &\to \mathcal{P}\\
(\envVr_1,\ldots,\envVr_k) &\mapsto \nu \quad \text{s.t. $\nu(A) :=  \frac{1}{k} \sum_{i=1}^k \mathbbm{1} \left \{ \envVr_i \in A \right \}, \quad \mbox{for every measurable set $A$.}$} 
\end{cases}
\end{equation*}
Furthermore, we denote by $\Pi_{\mathrm{discrete}}$ the set of mappings from $\bigcup_{k=1}^\infty \left( \{k\} \times \phi_k(\envSp^k) \right)  \to  \Delta(\actSp)$. Intuitively, $\Pi_{\mathrm{discrete}}$ is the set of data-driven policies taking as inputs empirical distributions.

We note that $\Pi_{\mathrm{discrete}}$ can be used to reparametrize the set $\Pi_{\mathrm{standard}}$. 
Indeed, consider the mapping
\begin{equation*}
\Phi : \begin{cases}
\Pi_{\mathrm{discrete}} &\to \Pi_{\mathrm{standard}}  \\
\quad \pi' &\mapsto (\pi_\numS)_{\numS \in \mathbb{N}} \quad \text{s.t. $\pi_\numS(\bm{\envVr}^\numS) = \pi'(\numS, \phi_\numS(\bm{\envVr}^\numS))$ \quad for every $\numS \geq 1, \, \bm{\envVr}^\numS \in \envSp^\numS$}.
\end{cases}
\end{equation*}
Then, we have that $\Phi$ is bijective.
This shows that one can redefine data-driven policies by only looking at the number of samples and the empirical distribution.

In this work, we define data-driven policies through the set $\Pi$, to do so we need to handle the extension of policies when taking as input a distribution which is not necessarily an empirical distribution (for instance a continuous distribution). To formalize this idea we consider the following equivalence relation.
\begin{definition}[Equivalence of policies]
Let $\sim$ be the equivalence relation on $\Pi$ such that, for every $\pi, \pi' \in \Pi$, $\pi \sim \pi'$ if and only if, for every $k \geq 1$ and for every $\nu \in \phi_k \left( \envSp^k \right)$, we have $\pi(k,\nu) = \pi'(k,\nu)$.
\end{definition}
The equivalence relation $\sim$ intuitively requires policies in $\Pi$ to be equal when restricting attention to empirical distributions. The next statement justifies the definition of data-driven policies as objects in $\Pi$.
\begin{proposition}
\label{prop:justify_pi}
Let $\Pi/\sim$ be the quotient set (i.e., the set of equivalent classes of elements $\Pi$ w.r.t to the relation $\sim$). 
Furthermore, let $q: \Pi \to \Pi/\sim$ denote the canonical surjection.

Then, there exists a bijective mapping $\psi$ from $\Pi/\sim$ to $\Pi_{\mathrm{discrete}}$ such that for every $\pi \in \Pi$ and $\tilde{\pi} \in \Pi_{\mathrm{discrete}}$, if $\psi \circ q(\pi) = \tilde{\pi}$, we have that 
$\mathfrak{R}_{\mathcal{I},\numS}^{\pi}(\epsilon)  = \mathfrak{R}_{\mathcal{I},\numS}^{\tilde{\pi}}(\epsilon)$ for every $\epsilon \geq 0$ and every $\numS \geq 1$.
\end{proposition}

\Cref{prop:justify_pi} implies that our more general definition of data-driven policies can be used to analyze classical ones without worrying about the way we extend the data-driven policies to distributions which are not empirical ones. In particular, \Cref{prop:justify_pi} shows that a policy $\pi \in \Pi$ can be used to analyze a \textit{classical} policy of interest $\tilde{\pi} \in \Pi_{\mathrm{discrete}}$ as long as they agree on the set of empirical distributions. Given this strong equivalence between policies defined in $\Pi$ and polices defined in $\Pi_{\mathrm{standard}}$ we do not discuss in the main body these subtleties.
 For instance, if one wants to analyze the SAA policy, defining a policy on the whole space of probability measures as in \Cref{sec:SAA-analysis} is valid as long as it selects the same action as SAA when given an empirical distribution as an input.

Given the above discussion, one may wonder what is the purpose of defining policies as in $\Pi$ given that the worst-case regret of these policies end-up being defined by their restriction to empirical distribution. This is motivated by the DRO regret defined in \eqref{eq:DRO_problem} as,
\begin{equation*}
\mathfrak{R}_{\scaleto{\mathcal{I},\mathrm{DRO}}{6pt}}^{\pi} (\epsilon) := \sup_{\mesOut \in \mesSp} \sup_{\nu\in \mathcal{U}_\epsilon(\mesOut)} \left \{ \opt(\mesOut) - \g[\pi(\nu)]{\mesOut} \right \}. 
\end{equation*}
This notion can only be well-defined by extending the definition of  data-driven policies on all distributions.
Quite importantly, we note that even if two policies $\pi,\pi' \in \Pi$ satisfy $\pi \sim \pi'$ their DRO regret may in general be \textit{different}.

The DRO regret plays a central role in \Cref{thm:reduction_stat} and upper bounds the asymptotic worst-case regret of a broad class of policies in $\Pi$. For completeness our next result formalizes the implication of \Cref{thm:reduction_stat} for data-driven policies which are only defined for empirical distributions.
\begin{theorem}
\label{thm:reduction_general_policies}
Let $\mathcal{I}$ be an instance satisfying the assumptions of \Cref{thm:reduction_stat}. 
Furthermore let $q$ and $\psi$ be the mappings defined in \Cref{prop:justify_pi}.
Then, for every $\tilde{\pi} \in \Pi_{\mathrm{discrete}}$,
\begin{equation*}
 \mathfrak{R}_{\mathcal{I},\infty}^{\tilde{\pi}}(\epsilon) \leq \inf_{\substack{\pi \in \agSp\\ \text{s.t. $\psi \circ q(\pi) = \tilde{\pi}$}}}  \lim_{\eta \to \epsilon^{+}} \mathfrak{R}_{\scaleto{\mathcal{I},\mathrm{DRO}}{6pt}}^{\pi} (\eta). 
\end{equation*}
\end{theorem}
\Cref{thm:reduction_general_policies} implies that analyzing the DRO regret of any extension of sample-size-agnostic policies to all probability measures is a valid procedure to upper bound the worst-case regret of  data-driven policies defined in a standard way. 
We note that \Cref{thm:reduction_general_policies} involves an infimum which suggests that one should choose the best extension to derive the tightest bounds. In this work we only considered a single extension (usually the one naturally suggested by the policy defined on the empirical distributions) and this was sufficient to derive rate optimal bounds matching the lower bounds we derive.

\subsubsection{Proofs of auxiliary results}

\begin{proof}[\textbf{Proof of \Cref{prop:justify_pi}}]
Consider the canonical surjection $q$ and define $\psi$ as follows.
\begin{equation*}
\psi : \begin{cases}
\Pi_{\mathrm{discrete}} &\to  \Pi/\sim  \\
\quad \tilde{\pi} &\mapsto [ \pi ] \quad \text{s.t. for all $\pi' \in \Pi \text{ s.t. } q(\pi')=[ \pi ], \; \tilde{\pi}(k,\nu) = \pi'(k,\nu)$ for all $k \geq 1$ and $\nu \in \phi_k \left( \envSp^k \right)$.}
\end{cases}
\end{equation*}
Let us show that $\psi$ is one-to-one.

We first show that $\psi$ is an injective mapping. Assume that there exits $\tilde{\pi}_1,\tilde{\pi}_2 \in \Pi_{\mathrm{discrete}}$ such that $\psi(\tilde{\pi}_1) = \psi(\tilde{\pi}_2)$. This implies that  $\tilde{\pi}_1(k,\nu) = \tilde{\pi}_2(k,\nu)$ for all $k \geq 1$ and $\nu \in \phi_k \left( \envSp^k \right)$ and therefore $\tilde{\pi}_1 = \tilde{\pi}_2$.

We next show that $\psi$ is a surjective mapping. Let $[\pi] \in \Pi/\sim$ and consider some $\pi'_0 \in \Pi$ such that $q(\pi'_0) =  [\pi] $. Then define $\tilde{\pi} \in \Pi_{\mathrm{discrete}}$ such that for every $k \geq 1$ and $\nu \in \phi_k \left( \envSp^k \right)$ we have that $\tilde{\pi}(k,\nu) = \pi'_0(k,\nu)$ . We next argue  that $\psi(\tilde{\pi}) = [\pi]$. Remark that for every $\pi' \in \Pi$, if $q(\pi') = [\pi]$ then $\pi' \sim \pi'_0$ and therefore  $\pi'(k,\nu) = \pi'_0(k,\nu) = \tilde{\pi}(k,\nu)$ for all $k \geq 1$ and $\nu \in \phi_k \left( \envSp^k \right)$. By definition of $\psi$, this implies that $\psi(\tilde{\pi}) = [\pi]$, which concludes the proof.

\end{proof}

\begin{proof}[\textbf{Proof of \Cref{thm:reduction_general_policies}}]
Let $\tilde{\pi} \in \Pi_{\mathrm{discrete}}$ and $\pi \in \agSp$ such that $\psi \circ q(\pi) = \tilde{\pi}$. Then,
\begin{equation*}
 \mathfrak{R}_{\mathcal{I},\infty}^{\tilde{\pi}}(\epsilon)  \stackrel{(a)}{=} \mathfrak{R}_{\mathcal{I},\infty}^{\pi}(\epsilon) \stackrel{(b)}{\leq} \lim_{\eta \to \epsilon^{+}} \mathfrak{R}_{\scaleto{\mathcal{I},\mathrm{DRO}}{6pt}}^{\pi} (\eta),
\end{equation*}
where $(a)$ follows from \Cref{prop:justify_pi} and $(b)$ follows from \Cref{thm:reduction_stat}. Then by taking an infimum over all policies in $\agSp$ such that $\psi \circ q(\pi) = \tilde{\pi}$ we obtain the desired result.

\end{proof}

\subsection{Properties on distances and remarks}

\subsubsection{Relations between distances}\label{sec:apx_rel}
In this section, we show that one can translate relations between different distances or divergences into relations between worst-case regret associated to them. We first show a general result which enables us to translate relate guarantees and then illustrate how this result can be used to derive guarantees for the total variation distance of the KL-divergence.
We first show the following result.
\begin{lemma}
\label{lem:dist_to_reg}
Consider two distances or divergences $d_1$ and $d_2$ on the probability measure space $\mesSp$ and assume that there exists a function $h$ such that for every $\mu$ and $\nu \in \mesSp$, $d_1(\mu,\nu) \leq h \left( d_2 ( \mu,\nu) \right)$. Let $\mathcal{I}_1 = \left( \actSp, \envSp, \mesSp, d_1, g\right)$ and $\mathcal{I}_2 = \left( \actSp, \envSp, \mesSp, d_2, g\right)$ be two data-driven decision problems in heterogeneous environments that differ only along the distance  dimension.  Then, for every data-driven policy $\pi \in \Pi$, any sample size $\numS \geq 1$ we have that for every $\epsilon \geq 0$,
\begin{equation*}
\mathfrak{R}_{\mathcal{I}_2,\numS} \left( \epsilon \right) \leq \mathfrak{R}_{\mathcal{I}_1,\numS} \left( h(\epsilon) \right) 
\end{equation*}
\end{lemma}
\begin{proof}[\textbf{Proof of \Cref{lem:dist_to_reg}}]
For every $\mesOut \in \mesSp$, let $\mathcal{U}^1_{\epsilon} \left( \mesOut \right)$ (resp. $\mathcal{U}^2_{\epsilon} \left( \mesOut \right)$) denote the heterogeneity ball centered around $\mesOut$ and with radius $\epsilon$ with respect to $d_1$ (resp. $d_2$).

We first argue that for any probability measure $\mesOut \in \mesSp$, and any $\epsilon \geq 0$, we have that, $ \mathcal{U}^2_{  \epsilon} \left( \mesOut \right) \subset \mathcal{U}^1_{h(\epsilon)} \left( \mesOut \right)$. Indeed, for every $\nu\in  \mathcal{U}^2_{ \epsilon} \left( \mesOut \right) $ we have that 
\begin{equation*}
d_1\left(\mesOut,\nu\right) \stackrel{(a)}{\leq}  h(d_2\left(\mesOut,\nu\right)) \leq h(\epsilon),
\end{equation*}
where $(a)$ follows from the relation between $d_1$ and $d_2$. Therefore, for every $\numS \geq 1$, $\epsilon \geq 0 $ and $\mesOut \in \mesSp$ we have
\begin{equation*}
\sup_{ \mesIn{1},\ldots, \mesIn{\numS} \in \mathcal{U}^{2}_{\epsilon}(\mesOut)} \mathcal{R}_\numS \left( \pi, \mesOut, \mesIn{1}, \ldots, \mesIn{\numS} \right) \leq \sup_{ \mesIn{1},\ldots, \mesIn{\numS} \in \mathcal{U}^{1}_{h(\epsilon)}(\mesOut)} \mathcal{R}_\numS \left( \pi, \mesOut, \mesIn{1}, \ldots, \mesIn{\numS} \right). 
\end{equation*}
By taking the supremum over $\mesOut \in \mesSp$ we obtain the desired result.
\end{proof}

\Cref{lem:dist_to_reg} is a very powerful tool to translate some of our regret guarantees derived for the Kolmogorov and the Wasserstein distance to other metrics or $\phi$-divergences. This is especially useful, as it does not require the distance to satisfy the ETC property under which \Cref{thm:reduction_stat} holds. For instance, for the total variation distance (which does not satisfy the ETC property), one can show that the worst-case regret of SAA under all three problems discussed in \Cref{tab:results} is also order $\epsilon$. This follows from the fact that  for every $\mu_1$ and $\mu_2 \in \Delta([0,M]),$
\begin{equation*}
d_K(\mu_1,\mu_2) \leq d_{TV}(\mu_1,\mu_2),
\end{equation*}
where $d_{TV}$ denotes the total variation distance. Therefore, \Cref{lem:dist_to_reg} implies that the worst-case regret of SAA under the total variation distance is lower than the one under the Kolmogorov distance which is order $\mathcal{O}(\epsilon)$ (where we omit the dependence in the size of the support) for these problems. Furthermore, for KL-divergence (which is not an IPM), one can show that the regret of SAA for all three problems is of the order of $\mathcal{O}(\sqrt{\epsilon})$. The upper bound can be obtained from the regret under total variation heterogeneity by using Pinsker's inequality and applying \Cref{lem:dist_to_reg}.

\subsubsection{Empirical triangular convergence property}\label{sec:apx_ETC}
We next show that the Kolmogorov distance and the Wasserstein distance satisfy the ETC property on compact intervals of $\mathbb{R}$.

First, when $\envSp$ is a compact interval of $\mathbb{R}$, the following theorem in \cite{shorack1979weighted} implies that the Kolmogorov distance satisfies the ETC property. 
\begin{proposition}[\cite{shorack1979weighted}, Theorem 2.1]
\label{prop:etc_Kol}
If $\envSp = [0,M]$, the Kolmogorov distance satisfies the empirical triangular convergence property.
\end{proposition}
A direct corollary of \Cref{prop:etc_Kol} along with the fact that $d_{W} \leq \diam{\envSp} \cdot d_K$ is that the empirical triangular convergence property can be transferred to the Wasserstein distance.
\begin{corollary}
\label{cor:ec_War}
If $\envSp = [0,M]$, the Wasserstein distance satisfies the empirical triangular convergence property.
\end{corollary}

\subsubsection{Relation to previous models of contamination}\label{sec:relation_to_models} 
 As discussed in \Cref{sec:relatedWork}, our framework is closely related to previous models of contamination. We next present how some of the most popular oblivious adversary models can be casted in our framework and analyzed. 
 
\noindent \textbf{Strong contamination or nasty model.} In the oblivious version of this model, 
the adversary samples all data from a \textit{fixed} distribution that is $\epsilon$ away from the true distribution in the total variation distance.
Remark that this model is a special case of ours, where we set the distance to be $d_{TV}$, and we impose that all historical distributions $\mesIn{1}, \ldots, \mesIn{\numS}$ are identical (but may be different from the out-of-sample distribution $\mesOut$). 

 \noindent \textbf{Huber contamination model.} In this model, the samples $\envSamples$ are assumed to be generated from a contaminated distribution $\nu = (1-\epsilon) \cdot \mesOut + \epsilon \cdot \tilde{\mu}$, where $\tilde{\mu}$ is an arbitrary distribution. 
We note that any upper bound on the regret derived for the total variation distance implies an upper bound for the Huber contamination model.

\subsection{A general lower bound scheme}
We next present a result that will be helpful to derive lower bounds on the asymptotic worst-case regret.
\begin{proposition}[Lower bound scheme]
\label{prop:lower_bound}
Fix $\epsilon \geq 0$. For any instance $\mathcal{I}$ and given two distributions $\mu_{-}$ and $\mu_{+}$ such that $ \mathcal{U}_{\epsilon}(\mu_{-}) \cap \mathcal{U}_{\epsilon}(\mu_{+}) \neq \emptyset$ we have that, for every $\numS \geq 1$,
\begin{equation*}
\inf_{\pi \in \Pi} \mathfrak{R}_{\mathcal{I},\numS}^{\pi} \left( \epsilon \right) \geq \sup_{\beta \in \Delta \left( \{ \mu_{-}, \mu_{+} \} \right)} \inf_{\actVr \in \actSp} \mathbb{E}_{\mesOut \sim \beta} \left[ |\g[x]{\mesOut} - \opt(\mesOut)|\right].
\end{equation*}
\end{proposition}

\begin{proof}[\textbf{Proof of \Cref{prop:lower_bound}}]

For a fixed $\numS \geq 1$, and fix $\tilde{\nu} \in \mathcal{U}_{\epsilon}(\mu_{-}) \cap \mathcal{U}_{\epsilon}(\mu_{+})$ we have that,
\begin{align*}
    \inf_{\pi \in \Pi} \mathfrak{R}_{\mathcal{I},\numS}^{\pi}(\epsilon) &\stackrel{(a)}{\geq} \inf_{\pi \in \Pi} \sup_{\mesOut \in \{ \mu_{-}, \mu_{+}\} }  \sup_{\nu \in \mathcal{U}_\epsilon(\mu)} \mathcal{R}_\numS \left( \pi, \mesOut, \nu, \ldots, \nu \right)\\
    &\stackrel{(b)}{\geq} \inf_{\pi \in \Pi} \sup_{\mesOut \in \{ \mu_{-}, \mu_{+}\} }  \mathcal{R}_\numS \left( \pi, \mesOut, \tilde{\nu}, \ldots, \tilde{\nu} \right)\\
    &= \inf_{\pi \in \Pi} \sup_{\mesOut \in \{ \mu_{-}, \mu_{+}\} } |\mathbb{E}_{\envVr_i \sim \tilde{\nu}} \left[ \g[\pi(\numS,\hat{\nu}_{\envSamples})]{\mesOut} \right] - \opt(\mesOut)|\\
    &\stackrel{(c)}{=} \inf_{\alpha \in \Delta \left( \actSp \right)} \sup_{\mesOut \in \{ \mu_{-}, \mu_{+}\}}    |\g[\alpha]{\mesOut}  - \opt(\mesOut)| \stackrel{(d)}{\geq} \sup_{\beta \in \Delta \left( \{ \mu_{-}, \mu_{+} \} \right) } \inf_{\actVr \in \actSp} \mathbb{E}_{\mesOut \sim \beta} \left[|\g[\actVr]{\mesOut}  - \opt(\mesOut) |\right] 
\end{align*}
where $(a)$ follows from the fact that we restrict the space of out-of-sample distributions to be $\{ \mu_{-}, \mu_{+} \}$ and we restrict historical distributions $\mesIn{1},\ldots, \mesIn{\numS}$ to be all equal. $(b)$ holds as $\tilde{\nu} \in \mathcal{U}_{\epsilon}(\mu_{-}) \cap \mathcal{U}_{\epsilon}(\mu_{+})$. We obtain $(c)$ by remarking that when $\tilde{\nu}$ is fixed, the policy does not need to look at the data as it does not give any extra information on the value of $\mesOut$. Indeed, for every $\pi \in \Pi$, we construct a probability measure $\alpha \in \Delta \left( \actSp \right)$ which satisfies for every $\actVr \in \actSp$,
\begin{equation*}
    \alpha(\actVr) = \mathbb{P}_{\envVr_i \sim \tilde{\nu}} \left(\pi(\numS, \hat{\nu}_{\envSamples}) = \actVr \right).
\end{equation*}
We then have for every $\mesOut \in \{\mu_{-}, \mu_{+}\}$,
\begin{equation*}
    \mathbb{E}_{\envVr_i \sim \tilde{\nu}} \left[ \g[\pi(\numS,\hat{\nu}_{\envSamples})]{\mesOut} \right] - \opt(\mesOut) = \mathbb{E}_{\actVr \sim \alpha} \left[ \g[\actVr]{\mesOut} \right] - \opt(\mesOut) = \g[\alpha]{\mesOut} - \opt(\mesOut).
\end{equation*}
By taking the supremum over $\mesOut$, we obtain,
\begin{equation*}
    \sup_{\mesOut \in \{ \mu_{-}, \mu_{+}\} } |\mathbb{E}_{\envVr_i \sim \tilde{\nu}} \left[ \g[\pi(\numS,\hat{\nu}_{\envSamples})]{\mesOut} \right] - \opt(\mesOut)| = \sup_{\mesOut \in \{ \mu_{-}, \mu_{+}\} } |\g[\alpha]{\mesOut} - \opt(\mesOut)|
\end{equation*}
which implies the inequality,
\begin{equation*}
    \inf_{\pi \in \Pi}    \sup_{\mesOut \in \{ \mu_{-}, \mu_{+}\} } |\mathbb{E}_{\envVr_i \sim \tilde{\nu}} \left[ \g[\pi(\numS,\hat{\nu}_{\envSamples})]{\mesOut} \right] - \opt(\mesOut)| \geq \inf_{\alpha \in \Delta(\actSp)} \sup_{\mesOut \in \{ \mu_{-}, \mu_{+}\} } |\g[\alpha]{\mesOut} - \opt(\mesOut)|.
\end{equation*}
In addition, the inequality 
\begin{equation*}
    \inf_{\pi \in \Pi}    \sup_{\mesOut \in \{ \mu_{-}, \mu_{+}\} } |\mathbb{E}_{\envVr_i \sim \tilde{\nu}} \left[ \g[\pi(\numS,\hat{\nu}_{\envSamples})]{\mesOut} \right] - \opt(\mesOut)| \leq \inf_{\alpha \in \Delta(\actSp)} \sup_{\mesOut \in \{ \mu_{-}, \mu_{+}\} } |\g[\alpha]{\mesOut} - \opt(\mesOut)|
\end{equation*}
holds by inclusion. Finally, $(d)$ follows from Yao's minimax principle.
\end{proof}

\section{Proofs of  \Cref{sec:reduction} and additional results}
\label{sec:apx_reduction}

\begin{proof}[\textbf{Proof of \Cref{prop:heterogeneity_helps}}]
Let $k$ be a positive integer.
We consider a variant of the ski-rental problem presented in \Cref{sec:ski}, in which we set $b = 2 k +1$, $\envSp = [0, 3k+2]$, $\actSp = \{0, \ldots, 3k+2 \}$ and  $\mesSp = \Delta \left(\{ k, k+1, 3k+2 \}\right)$. Note that we restrict here the space of distributions to the ones that are supported on  $\{ k, k+1, 3k+2 \}$.

The proof intuition is then as follows.  We let $\epsilon=\infty$.  The adversary will make $\xi=k+1$ w.p.~1 in the true distribution $\mu$, but wants SAA to buy on day $k$, right before the skiing stops, to cause maximum regret.  In order to accomplish this, the adversary gives SAA a distribution $\nu$ that is $k$ and $3k+2$ w.p.~1/2 each, on which the optimal policy is to rent until time $k$ and buy at that moment if the skiing were to continue.  We then show that it is \textit{impossible} to make SAA buy at time exactly $k+1$ through any \textit{deterministic} distribution, completing the proof that the adversary has increased power when it can give SAA a distribution $\nu$ that is contructed from heterogeneous distributions.

For the sake of  simple notations, we denote by $\SAA(\envVr_1,\envVr_2)$ the decision of SAA after observing samples $\envVr_1$ and $\envVr_2$.
We next enumerate all possible decisions of SAA as a function of the samples observed. The order of the samples does not affect the decision and we have,
\begin{align}
\SAA(k,k) = \SAA(k+1,k+1) &= \SAA(k,k+1) = 3k+2,  \quad \SAA(3k+2,3k+2) = 0, \nonumber\\
\SAA(k,3k+2) &\stackrel{(a)}{=} k \quad \mbox{and} \quad \SAA(k+1,3k+2) \stackrel{(b)}{=} 0,\label{eq:SAA_actions}
\end{align}
where we next give a detailed derivation of $(a)$ and $(b)$.

First, for every action $\actVr \in \actSp$, denote by $\hat{c}(\actVr)= \frac{1}{2} g(\actVr,k) + \frac{1}{2} g(\actVr,3k+2)$ the empirical cost of the action $\actVr$ against the samples $k$ and $3k+2$. We note that in that case, all actions are dominated by one of the following ones $\{0,k,3k+2\}$ as it is never optimal to purchase between skiing days. Furthermore, we have that,
\begin{equation*}
\hat{c}(0) = b = 2k+1 \qquad  \hat{c}(k) = \frac{1}{2} k + \frac{1}{2}\left(k+b\right) = 2k + \frac{1}{2} \qquad \hat{c}(3k+2)= \frac{1}{2}k + \frac{1}{2} \left( 3k+2 \right) = 2k+1.
\end{equation*}
This implies $(a)$. Similarly, we now let $\hat{c}(\actVr)= \frac{1}{2} g(\actVr,k+1) + \frac{1}{2} g(\actVr,3k+2)$. In that case, all actions are dominated by one of the following ones $\{0,k+1,3k+2\}$ and,
\begin{align*}
\hat{c}(0) = b = 2k+1 &\qquad  \hat{c}(k+1) = \frac{1}{2}(k+1) + \frac{1}{2}\left(k+1+b\right) = 2k + \frac{3}{2}\\
\hat{c}(3k+2)&= \frac{1}{2}(k+1) + \frac{1}{2} \left( 3k+2 \right) = 2k+\frac{3}{2}.
\end{align*}
This implies $(b)$. 

Let $\epsilon = \infty$.
Remark that for every $\mesOut \in \mesSp$ and for any action $\actVr \in \actSp$, we have 
\begin{align*}
\g[x]{\mesOut} - \opt(\mesOut) &= \sum_{\envVr \in \{k,k+1,3k+2\}} \mu(\envVr) \cdot \left( \g[x]{\delta_\envVr} - \opt(\mu) \right)\\  
&\stackrel{(a)}{\leq} \sum_{\envVr \in \{k,k+1,3k+2\}} \mu(\envVr) \cdot \left( \g[x]{\delta_\envVr} - \opt(\delta_\envVr) \right)
 \leq \max_{\envVr \in \{k,k+1,3k+2\}} \g[x]{\delta_\envVr} - \opt(\delta_\envVr), 
\end{align*}
where $(a)$ holds because $\opt(\mesOut) \geq  \sum_{\envVr \in \{k,k+1,3k+2\}} \mu(\envVr) \cdot \opt(\delta_\envVr)$.
 
Therefore the maximum regret incurred by SAA is at most
\begin{equation*}
 \mathfrak{R}_{\mathcal{I},2}^{\SAA}(\epsilon) \leq \max_{x \in \{0,k,3k+1\}} \max_{\envVr \in \{k,k+1,3k+2\}} \g[x]{\delta_\envVr} - \opt(\delta_\envVr) \stackrel{(a)}{\leq} 2k, 
\end{equation*}
where $(a)$ is obtained by enumerating all possible cases as presented in \Cref{tab:enumerate}.
\begin{table}[h]
\centering
\begin{tabular}{ c| c c c }
~ & $\envVr = k$ & $\envVr = k+1$ & $\envVr = 3k+2$ \\
\hline 
 $\actVr = 0$ & $k+1$ & $k$ & $0$\\ 
 $\actVr = k$ & $0$ & $2k$ &$k$\\  
 $\actVr = 3k+2$ & $0$ & $0$ & $k+1$    
\end{tabular}
\caption{Different values of $\g[x]{\delta_\envVr} - \opt(\delta_\envVr)$ .}
\label{tab:enumerate}
\end{table}

We also note from \Cref{tab:enumerate} that the worst-case regret is achieved when SAA selects action $k$ and when $\mesOut = \delta_{k+1}$. We finally, show that this worst-case regret can only be obtained when nature uses different historical distribution $\mesIn{1} \neq \mesIn{2}.$
First note that the out-of-sample distribution must be $\delta_{k+1}$ in the worst-case to achieve a regret of $2k$.
 Furthemore, we have seen in \eqref{eq:SAA_actions} that SAA selects $k$ if and only if it observes samples $k$ and $3k+2$. 
 Therefore, a regret of $2k$ can only be achieved if the probability of selecting action $k$ is $1$.
By setting $\mesIn{1} = \delta_k$ and $\mesIn{2} = \delta_{3k+2}$ we ensure that SAA selects action $k$ with probability $1$. In contrast, for any distribution $\nu$ such that $\nu(k) > 0$ and $\nu(3k+2) >0$, the probability that SAA selects action $k$ when observing two samples from $\nu$ is strictly smaller than $1$. 

This concludes the proof and shows that,
\begin{equation*}
 \mathfrak{R}_{\mathcal{I},2}^{\SAA}(\epsilon) > \sup_{\mesOut \in \mesSp} \sup_{\nu\in \mathcal{U}_\epsilon(\mesOut)} \mathcal{R}_2 \left( \SAA, \mesOut, \nu,  \nu \right).
\end{equation*}

\end{proof}

\begin{proof}[\textbf{Proof of \Cref{thm:reduction_stat}}]
Fix $\epsilon >0$. To show this theorem, it is sufficient to prove that for every sequence of distributions $(\mesOut_\numS)_{\numS \in \mathbb{N}^*}$, $\bm{\nu} := (\mesIn{i,\numS})_{1 \leq i \leq \numS, \numS \in \mathbb{N}^*}$, such that $\mesIn{i,\numS} \in \mathcal{U}_\epsilon(\mesOut_\numS)$ for every $i$ and $\numS$, we have that
\begin{equation}
\label{eq:fixed_sequence}
\limsup_{\numS \to \infty}   \mathcal{R}_\numS \left( \pi, \mesOut_\numS, \mesIn{1,\numS}, \ldots, \mesIn{\numS,\numS} \right) \leq \lim_{\eta \to \epsilon^{+}} \sup_{\mesOut \in \mesSp} \sup_{\nu\in \mathcal{U}_\eta(\mesOut)} | \opt(\mesOut) - \g[\pi(\nu)]{\mesOut} |. 
\end{equation}
Indeed, for every $\alpha >0$, we can construct a sequence of distribution, $(\mesOut_\numS)_{\numS \in \mathbb{N}^*}$, $\bm{\nu} := (\mesIn{i,\numS})_{1 \leq i \leq \numS, \numS \in \mathbb{N}^*}$, such that for every $i$ and $\numS$, we have $\mesIn{i,\numS} \in \mathcal{U}_\epsilon(\mesOut)$  and
\begin{equation*}
 \mathfrak{R}_{\mathcal{I},\numS}^{\pi}(\epsilon) \leq  \mathcal{R}_\numS \left( \pi, \mesOut_\numS, \mesIn{1,\numS}, \ldots, \mesIn{\numS,\numS} \right) + \alpha.
\end{equation*}
Assuming that \eqref{eq:fixed_sequence} holds and by taking the limit we obtain that for every $\alpha >0$,
\begin{equation*}
 \mathfrak{R}_{\mathcal{I},\infty}^{\pi}(\epsilon) \leq  \lim_{\eta \to \epsilon^{+}} \sup_{\mesOut \in \mesSp} \sup_{\nu\in \mathcal{U}_\eta(\mesOut)} | \opt(\mesOut) - \g[\pi(\nu)]{\mesOut} | + \alpha.
\end{equation*}
We obtain the desired result by taking $\alpha$ to $0$.
We now prove \eqref{eq:fixed_sequence}.

For every sample path $(\envVr_{i,\numS})_{1\leq i \leq \numS, \numS \in \mathbb{N}^*}$ s.t. $\envVr_{i,\numS} \sim \mesIn{i,\numS}$  we define the sequence of empirical distributions $\left( \hat{\nu}_{\numS} \right)_{\numS \in \mathbb{N}^*}$ such that for every $\numS \geq 1$ and every measurable set $A$, we have $\hat{\nu}_\numS(A) :=  \frac{1}{\numS} \sum_{i=1}^\numS \mathbbm{1} \left \{ \envVr_{i,\numS} \in A \right \}$. We also define the sequence $\left( \bar{\nu}_{\numS} \right)_{\numS \in \mathbb{N}^*}$ such that for every $\numS \geq 1$, we have $\bar{\nu}_\numS =  \frac{1}{\numS}\sum_{i=1}^\numS \nu_{i,\numS}$.

In turn, let $(Z_\numS)_{\numS \in \mathbb{N}^*}$ be the sequence of random variables defined for every $\numS \geq 1$ as $Z_\numS := | \opt(\mesOut_{\numS})  - \g[\pi(n,\hat{\nu}_\numS)]{\mesOut_\numS} |$. Remark that there exists a constant $K$ such that for every $\numS \geq 1$, $|Z_\numS| \leq K$ almost surely because $g$ is bounded on $\actSp \times \envSp$.

Therefore we have that,
 \begin{align*}
 \limsup_{\numS \to \infty}   \mathcal{R}_\numS \left( \pi, \mesOut_\numS, \mesIn{1,\numS}, \ldots, \mesIn{\numS,\numS} \right) &=    \limsup_{\numS \to \infty}  |\mathbb{E}_{\bm{\nu}} \left[ \opt(\mesOut_{\numS})  - \g[\pi(n,\hat{\nu}_\numS)]{\mesOut_\numS} \right]| \\
 &  \stackrel{(a)}{\leq}  \limsup_{\numS \to \infty} \mathbb{E}_{\bm{\nu}} \left[ | \opt(\mesOut_{\numS})  - \g[\pi(n,\hat{\nu}_\numS)]{\mesOut_{\numS}}|\right]\\
 &= \limsup_{\numS \to \infty} \mathbb{E}_{\bm{\nu}} \left[ Z_\numS \right]
 \stackrel{(b)}{\leq}  \mathbb{E}_{\bm{\nu}}\left[ \limsup_{\numS \to \infty} Z_\numS \right]
 \end{align*}
 where $(a)$ is a consequence of Jensen's inequality 
and $(b)$ holds by the reversed Fatou lemma \cite[Section 5.4]{williams1991probability} which can be applied because $(Z_\numS)_{\numS \in \mathbb{N}}$ are almost surely bounded. 
 
 We next show that $\mathbb{E}_{\bm{\nu}}\left[ \limsup_{\numS \to \infty} Z_\numS \right] \leq \lim_{\eta \to \epsilon^{+}} \sup_{\mesOut \in \mesSp} \sup_{\nu\in \mathcal{U}_\eta(\mesOut)} | \opt(\mesOut) - \g[\pi(\nu)]{\mesOut} |$. To do so, we show a stronger statement and show that $ \limsup_{\numS \to \infty} Z_\numS  \leq \lim_{\eta \to \epsilon^{+}} \sup_{\mesOut \in \mesSp} \sup_{\nu\in \mathcal{U}_\eta(\mesOut)} | \opt(\mesOut) - \g[\pi(\nu)]{\mesOut} |$ almost surely.
 
 Fix a realization $\omega$ from the set of realizations such that ETC is verified. Remark that, for every $\numS \geq 1$, we have
\begin{align}
d(\hat{\nu}_\numS \left( \omega \right) , \mesOut_\numS) &\leq d(\hat{\nu}_\numS\left( \omega \right), \bar{\nu}_\numS) + d(\bar{\nu}_\numS, \mesOut_\numS) \nonumber \\
&= d(\hat{\nu}_\numS\left( \omega \right), \bar{\nu}_\numS) +  d \left(\frac{1}{\numS}\sum_{i=1}^\numS \nu_{i,\numS}, \mesOut_\numS \right) \nonumber \\
&\stackrel{(a)}{\leq} d(\hat{\nu}_\numS\left( \omega \right), \bar{\nu}_\numS) + \frac{1}{\numS} \sum_{i=1}^\numS d(\nu_{i,\numS}, \mesOut_\numS) \stackrel{(b)}{\leq} d(\hat{\nu}_\numS\left( \omega \right), \bar{\nu}_\numS) + \epsilon,\label{eq:conv_dist_a} 
\end{align}
where $(a)$ follows from the convexity of $d$ and $(b)$ holds because $d \left( \mesIn{i,\numS}, \mesOut \right) \leq \epsilon$ for all $i \leq \numS$.
 
We now define the sequence $(\delta_\numS \left( \omega \right) )_{\numS \in \mathbb{N}^*}$ such that for every $\numS$, $\delta_\numS \left( \omega \right):= \max \left( d(\hat{\nu}_\numS \left( \omega \right), \mesOut_\numS),\epsilon \right)$. Remark that since $d$ satisfies the empirical triangular convergence property, we have $d(\hat{\nu}_\numS\left( \omega \right), \bar{\nu}_\numS) \rightarrow 0$ as $n \uparrow \infty$. This, together with \eqref{eq:conv_dist_a}, implies that $\delta_\numS\left( \omega \right) \to \epsilon$ as $n \uparrow \infty$.

Furthermore, for every $\numS \geq 1$ we have that, 
\begin{align*}
Z_\numS\left( \omega \right) &\stackrel{(a)}{=} | \opt(\mesOut_{\numS})  - \g[\pi(\hat{\nu}_\numS\left( \omega \right))]{\mesOut_\numS} |\\
&\leq   \sup_{\mesOut \in \mesSp} \sup_{\nu \in \mathcal{U}_{d \left( \mesOut_\numS,\hat{\nu}_\numS \left( \omega \right)\right)}(\mesOut)} |\opt(\mesOut) - 
\g[\pi(\nu)]{\mesOut}|\\ 
&\leq   \sup_{\mesOut \in \mesSp} \sup_{\nu \in \mathcal{U}_{\delta_\numS\left( \omega \right)}(\mesOut)} |\opt(\mesOut) - 
\g[\pi(\nu)]{\mesOut}|.
\end{align*}
where $(a)$ holds because $\pi$ is sample-size-agnostic.

Hence, by taking the limit we obtain that,
\begin{align*}
    \limsup_{\numS \to \infty} Z_\numS \left( \omega \right) &\leq \limsup_{\numS \to \infty} \sup_{\mesOut \in \mesSp} \sup_{\nu \in \mathcal{U}_{\delta_\numS\left( \omega \right)}(\mesOut)} |\opt(\mesOut) - 
\g[\pi(\nu)]{\mesOut}|\\
     &= \lim_{\eta \to \epsilon^{+}} \sup_{\mesOut \in \mesSp} \sup_{\nu\in \mathcal{U}_\eta(\mesOut)} | \opt(\mesOut) - \g[\pi(\nu)]{\mesOut} |.  
\end{align*}
This concludes the proof.
\end{proof}

\subsection{Additional results}\label{sec:apx_agnostic_WLOG}
Our next result complements \Cref{thm:reduction_stat} and shows that the restriction to sample-size-agnostic policies can be assumed without loss of optimality, when analyzing the asymptotic worst-case regret.
\begin{proposition}[Optimality of sample-size-agnostic policies]
\label{prop:lb_agnostic}
Let $\mathcal{I} = \left( \actSp, \envSp, \mesSp, d, g \right)$ be a data-driven decision problem in a heterogeneous environment. Then for every $\epsilon \geq 0$,
 \begin{equation*}
\lim_{\eta \to \epsilon^{-}} \inf_{\pi \in \agSp}  \mathfrak{R}_{\mathcal{I},\infty}^{\pi}(\eta)  \leq \inf_{\pi \in \Pi} \mathfrak{R}_{\mathcal{I},\infty}^{\pi}(\epsilon). 
 \end{equation*}
\end{proposition}
\Cref{prop:lb_agnostic}, shows that for any $\epsilon$ for which $\inf_{\pi \in \agSp} \mathfrak{R}_{\mathcal{I},\infty}^{\pi}(\epsilon)$ is continuous, restricting attention to the class of sample-size-agnostic is without loss of optimality. The proof of this result relies on the following lemma.
\begin{lemma}
\label{lem:lb_agnostic_proof}
Let $\mathcal{I} = \left( \actSp, \envSp, \mesSp, d, g \right)$ be a data-driven decision problem in a heterogeneous environment. Assume $d$ is symmetric. Then, for every $\pi \in \Pi$ and every $\delta >0$, there exits $\pi' \in \Pi \in  \agSp$ such that,
 \begin{equation*} 
\mathfrak{R}_{\scaleto{\mathcal{I},\mathrm{DRO}}{6pt}}^{\pi'} (\epsilon) \leq \mathfrak{R}_{\mathcal{I},\infty}^{\pi}(\epsilon) + \delta
 \end{equation*}
\end{lemma}

\begin{proof}[\textbf{Proof of \Cref{lem:lb_agnostic_proof}}]
Let $\pi \in \Pi$ and $\delta >0$. 
We next construct $\pi' \in \agSp$ such that 
\begin{equation}
\label{eq:lb_agnostic}
\mathfrak{R}_{\scaleto{\mathcal{I},\mathrm{DRO}}{6pt}}^{\pi'}(\epsilon) \leq \mathfrak{R}_{\mathcal{I},\infty}^{\pi}(\epsilon) + \eta,
\end{equation}
holds.
We first note that for every $\numS \geq 1$,
\begin{align*}
\mathfrak{R}_{\mathcal{I},\numS}^{\pi}(\epsilon)  &= \sup_{\mesOut \in \mesSp} \sup_{\mesIn{1},\ldots,\mesIn{\numS} \in \mathcal{U}_\epsilon(\mesOut)}  \mathcal{R}_\numS \left( \pi, \mesOut, \mesIn{1}, \ldots, \mesIn{\numS} \right)\geq \sup_{\mesOut \in \mesSp} \sup_{\nu \in \mathcal{U}_\epsilon(\mesOut)}  \mathcal{R}_\numS \left( \pi, \mesOut, \nu, \ldots, \nu \right).
\end{align*}

Furthermore, we emulate the policy $\pi$ by considering the \textit{randomized} policy $\tilde{\pi}$, defined for every $\nu \in \mesSp$ and every $\numS \geq 1$ as follows. First $\tilde{\pi}$ samples $\envVr_1, \ldots, \envVr_{\numS}$ i.i.d. from $\nu$ and then it selects the action, $\pi(\numS,\hat{\mu}_{\envSamples})$. We note that for every $\nu \in \mesSp$ and $\numS \geq 1$, we have that
\begin{equation}
\label{eq:simulation}
 \mathbb{E}_{\envVr_i \sim \nu} \left[ \g[\pi(\numS,\hat{\mu}_{\envSamples})]{\mesOut} \right]  = \g[\tilde{\pi}(\numS,\nu)]{\mesOut}. 
\end{equation}
\eqref{eq:simulation} implies that,
\begin{align}
\sup_{\mesOut \in \mesSp} \sup_{\nu \in \mathcal{U}_\epsilon(\mesOut)}  \mathcal{R}_\numS \left( \pi, \mesOut, \nu, \ldots, \nu \right) = & \sup_{\mesOut \in \mesSp} \sup_{\nu \in \mathcal{U}_\epsilon(\mesOut)}  | \opt(\mesOut)-\mathbb{E}_{\envVr_i \sim \nu} \left[ \g[\pi(\numS,\hat{\mu}_{\envSamples})]{\mesOut} \right] | \notag \\
&= \sup_{\mesOut \in \mesSp} \sup_{\nu \in \mathcal{U}_\epsilon(\mesOut)}  | \opt(\mesOut)-  \g[\tilde{\pi}(\numS,\nu)]{\mesOut} | \notag \\
&\stackrel{(a)}{=} \sup_{\nu \in \mesSp} \sup_{ \mesOut \in \mathcal{U}_\epsilon(\nu)}  | \opt(\mesOut)-  \g[\tilde{\pi}(\numS,\nu)]{\mesOut}|, \label{eq:relation_to_sim}
\end{align}
where $(a)$ holds because $d$ is symmetric. 

Fix $\nu \in \mesSp$ and note that for every $\eta > 0$, there exists $\numS_0 \geq 1$ such that, 
\begin{equation}
\label{eq:n0}
\sup_{ \mesOut \in \mathcal{U}_\epsilon(\nu)}  | \opt(\mesOut)-  \g[\tilde{\pi}(\numS_0,\nu)]{\mesOut} |  \leq \inf_{\numS \geq 0} \sup_{ \mesOut \in \mathcal{U}_\epsilon(\nu)}  |   \opt(\mesOut)-\g[\tilde{\pi}(\numS,\nu)]{\mesOut} | + \eta
\end{equation}
For every $\nu \in \mesSp$ and $\numS \geq 1$, we define $\pi'(\numS,\nu) = \tilde{\pi}(\numS_0(\nu),\nu)$, where $\numS_0(\nu)$ satisfies \eqref{eq:n0} (we keep the dependence of $\numS_0$ in terms of $\nu$ explicit). $\pi'$ is sample-size agnostic by construction. Furthermore for every $\numS \geq 1$ and $\nu \in \mesSp$, it satisfies,
\begin{equation*}
\sup_{ \mesOut \in \mathcal{U}_\epsilon(\nu)}  |  \opt(\mesOut)- \g[\pi'(\numS,\nu)]{\mesOut} |  = \sup_{ \mesOut \in \mathcal{U}_\epsilon(\nu)}  |   \opt(\mesOut)-\g[\tilde{\pi}(\numS_0(\nu),\nu)]{\mesOut} |  \leq  \sup_{ \mesOut \in \mathcal{U}_\epsilon(\nu)}  |  \opt(\mesOut)- \g[\tilde{\pi}(\numS,\nu)]{\mesOut} | + \eta
\end{equation*}
Therefore, by taking a supremum over $\nu \in \mesSp$ and using \eqref{eq:relation_to_sim}, we obtain that,
\begin{equation*}
\sup_{\nu \in \mesSp}  \sup_{ \mesOut \in \mathcal{U}_\epsilon(\nu)}  |   \opt(\mesOut)-\g[\pi'(\numS,\nu)]{\mesOut} | \leq \sup_{\mesOut \in \mesSp} \sup_{\nu \in \mathcal{U}_\epsilon(\mesOut)}  \mathcal{R}_\numS \left( \pi, \mesOut, \nu, \ldots, \nu \right)  + \eta 
\leq \mathfrak{R}_{\mathcal{I},\numS}^{\pi}(\epsilon)  + \eta
\end{equation*}
We finally obtain \eqref{eq:lb_agnostic} by remarking that the LHS term is equal to $\mathfrak{R}_{\scaleto{\mathcal{I},\mathrm{DRO}}{6pt}}^{\pi'}(\epsilon)$ and by taking a limit as $\numS$ goes to $\infty$.
\end{proof}

\begin{proof}[\textbf{Proof of \Cref{prop:lb_agnostic}}]
Fix $\epsilon > 0$. Note that  \Cref{lem:lb_agnostic_proof} implies that, for every $\pi \in \Pi$ and every $\delta >0$, there exits $\pi' \in \Pi \in  \agSp$ such that,
 \begin{equation*} 
\mathfrak{R}_{\scaleto{\mathcal{I},\mathrm{DRO}}{6pt}}^{\pi'} (\epsilon) \leq \mathfrak{R}_{\mathcal{I},\infty}^{\pi}(\epsilon) + \delta.
 \end{equation*}
 Furthermore given that $\pi' \in \agSp$, we obtain by applying \Cref{thm:reduction_stat} that,
 \begin{equation*}
 \lim_{\eta \to \epsilon^{-}} \mathfrak{R}_{\mathcal{I},\infty}^{\pi'}(\eta) \leq \mathfrak{R}_{\mathcal{I},\infty}^{\pi}(\epsilon) + \delta.
 \end{equation*}
 Given that for every $\delta$, there exists a policy which $\pi' \in  \agSp$ satisfying the previous equality, this implies that,
\begin{equation*}
\inf_{\pi' \in \agSp} \lim_{\eta \to \epsilon^{-}} \mathfrak{R}_{\mathcal{I},\infty}^{\pi'}(\eta) \leq \mathfrak{R}_{\mathcal{I},\infty}^{\pi}(\epsilon).
\end{equation*}
Taking an infimum over $\pi \in \Pi$ we get that,
\begin{equation*}
\inf_{\pi' \in \agSp} \lim_{\eta \to \epsilon^{-}} \mathfrak{R}_{\mathcal{I},\infty}^{\pi'}(\eta) \leq \inf_{\pi \in \Pi} \mathfrak{R}_{\mathcal{I},\infty}^{\pi}(\epsilon).
\end{equation*}

To conclude the proof, we need to show that,
\begin{equation}
\label{eq:lim_inf_inversion}
\lim_{\eta \to \epsilon^{-}} \inf_{\pi' \in \agSp}  \mathfrak{R}_{\mathcal{I},\infty}^{\pi'}(\eta) \leq \inf_{\pi' \in \agSp} \lim_{\eta \to \epsilon^{-}} \mathfrak{R}_{\mathcal{I},\infty}^{\pi'}(\eta).
\end{equation}	
This last inequality follows from the following argument. For any $\pi \in \agSp$ and $\eta < \epsilon$, we have by definition of the infimum that for every , $ \inf_{\pi' \in \agSp} \mathfrak{R}_{\mathcal{I},\infty}^{\pi'}(\eta) \leq \mathfrak{R}_{\mathcal{I},\infty}^{\pi}(\eta).$ Then by taking the limit as $\eta$ goes to $\epsilon^{-}$ and taking an infimum over $\pi \in \agSp$, we obtain the desired inequality  \Cref{eq:lim_inf_inversion}.

\end{proof}

\section{Proofs for \Cref{sec:SAA-analysis} and additional results}
\subsection{Proofs of \Cref{sec:UB_IPM}}

\begin{proof}[\textbf{Proof for \Cref{thm:generating_class}}]
Fix $\actVr \in \actSp$. We next show that $|\g[\actVr]{\mesOut} - \g[\actVr]{\nu}| \leq \lambda \left( \mathcal{F},g,\delta \right) \cdot \dis{\mathcal{F}}(\mesOut,\nu) + 2\delta$. 

Fix $\eta >0$, remark that by definition of the approximation parameter, there exist $(\lambda_i)_{i \in \mathbb{N}} \in \mathbb{R}$ and $(f_i)_{i \in \mathbb{N}} \in \mathcal{F}$ such that,
$\sum_{i \in \mathbb{N}} | \lambda_i| \leq \lambda(\mathcal{F},g,\delta) + \eta$ and $\limsup_{\numS \to \infty} \| \sum_{i =1}^n \lambda_i \cdot f_i - g(\actVr, \cdot) \|_{\infty} \leq \delta$. 
Let $N \in \mathbb{N}$ and define $\bar{f}_{N} = \sum_{i =1}^N \lambda_i \cdot f_i$.
We have,
\begin{align*}
|\g[\actVr]{\mesOut} - \g[\actVr]{\nu}| &= \left | \int_{\envSp} g(\actVr,\envVr) d\mesOut(\envVr) - \int_{\envSp} g(\actVr,\envVr) d\nu(\envVr)  \right |\\
& \stackrel{(a)}{\leq} \left | \int_{\envSp} \left( g(\actVr,\envVr) - \bar{f}_{N} \left( \envVr \right) \right) d\mesOut(\envVr) -  \int_{\envSp} \left( g(\actVr,\envVr) - \bar{f}_{N} \left( \envVr \right)  \right)d\nu(\envVr)  \right |\\
&\quad + \left| \int_{\envSp} \bar{f}_{N}(\envVr) d\mesOut(\envVr) - \int_{\envSp} \bar{f}_{N}(\envVr) d\nu(\envVr) \right|\\
&\leq 2 \cdot \| g(\actVr, \cdot) -\bar{f}_{N}  \|_{\infty} + \left| \int_{\envSp} \bar{f}_{N}(\envVr) d\mesOut(\envVr) - \int_{\envSp} \bar{f}_{N}(\envVr) d\nu(\envVr) \right|\\
&\stackrel{(b)}{\leq} 2 \cdot \| g(\actVr, \cdot) -\bar{f}_{N}  \|_{\infty} + \sum_{i =1}^N |\lambda_i| \cdot \left| \int_{\envSp} f_i(\envVr) d\mesOut(\envVr) - \int_{\envSp} f_i(\envVr) d\nu(\envVr) \right|\\
&\stackrel{(c)}{\leq} 2 \cdot \| g(\actVr, \cdot) -\bar{f}_{N}  \|_{\infty} + \sum_{i \in \mathbb{N}} |\lambda_i| \cdot \dis{\mathcal{F}} (\mesOut,\nu) \stackrel{(d)}{\leq} 2 \cdot \| g(\actVr, \cdot) -\bar{f}_{N}  \|_{\infty} + \left[ \lambda(\mathcal{F}, g, \delta) + \eta \right] \cdot \dis{\mathcal{F}} (\mesOut,\nu),
\end{align*}
where $(a)$ follows from the triangular inequality, $(b)$ follows from the triangular inequality and the fact that $\bar{f}_{N} = \sum_{i =1}^N \lambda_i \cdot f_i$, $(c)$ holds by  definition of integral probability metrics and the fact that for every $i \in \mathbb{N}$ we have that $f_i \in \mathcal{F}$ and $(d)$ is a consequence of the definition of the approximation parameter and its relation to $(\lambda_i)_{i \in \mathbb{N}}$.

By taking a $\limsup$ when $N$ goes to $\infty$, we finally obtain that,
\begin{equation*}
|\g[\actVr]{\mesOut} - \g[\actVr]{\nu}| \leq  2 \cdot \delta + \left[ \lambda(\mathcal{F}, g, \delta) + \eta \right] \cdot \dis{\mathcal{F}} (\mesOut,\nu).
\end{equation*}

Since this holds for every $\eta$, we obtain by taking $\eta$ to $0$ and taking the infimum over $\delta$ that,
\begin{equation*}
|\g[\actVr]{\mesOut} - \g[\actVr]{\nu}| \leq \inf_{\delta \geq 0} \left \{ 2 \delta + \ \lambda(\mathcal{F}, g, \delta)  \cdot \dis{\mathcal{F}} (\mesOut,\nu) \right \}.
\end{equation*}
Finally, this holds for every $\actVr \in \actSp$. Hence, by taking supremum over actions, we obtain that,
\begin{equation*}
 \UB_{\mathcal{I}}(\mesOut,\nu) = \sup_{\actVr \in \actSp}|\g[\actVr]{\mu}-\g[\actVr]{\nu}| \leq  \inf_{\delta \geq 0} \left \{ 2 \delta + \lambda(\mathcal{F}, g, \delta)  \cdot \dis{\mathcal{F}} (\mesOut,\nu)  \right \}.
\end{equation*}
\end{proof}

\begin{proof}[\textbf{Proof of \Cref{prop:geometric_approximation}}]
We first show that if $g(\actVr,\cdot) \in \overline{\Span\left( \gen{\mathcal{F}} \right)}$ then for every $\delta > 0$, we have that $\lambda_{\actVr} \left( \mathcal{F},g,\delta \right) < + \infty$. Fix $\delta > 0$.
Note that since, $g(\actVr,\cdot) \in \overline{\Span\left( \gen{\mathcal{F}} \right)}$, there exists $f \in \Span\left( \gen{\mathcal{F}} \right)$ such that $\|f - g(\actVr,\cdot)\|_{\infty} \leq \delta$. Furthermore, by definition of the span, there exists $m \in \mathbb{N}$, $(f_i)_{i\in \{1,\ldots,m\}} \in \gen{\mathcal{F}}^m$ and $(\lambda_i)_{i\in \{1,\ldots,m\}} \in \mathbb{R}^m$ such that, $f= \sum_{i=1}^m \lambda_i \cdot f_i$. As a consequence, $\|\sum_{i=1}^m  \lambda_i \cdot f_i - g(\actVr,\cdot)\|_{\infty} \leq \delta$ which implies that $(\lambda_i)_{i\in \{1,\ldots,m\}}$, $(f_i)_{i\in \{1,\ldots,m\}}$ is a feasible solution for \eqref{eq:approx_param} which implies that,  $\lambda_{\actVr} \left( \mathcal{F},g,\delta \right) \leq \sum_{i=1}^m |\lambda_i|$.

We next prove reciprocally that if $\lambda_{\actVr} \left( \mathcal{F},g,\delta \right) < + \infty$ for every $\delta <0$ then $g(\actVr,\cdot) \in \overline{\Span\left( \gen{\mathcal{F}} \right)}$. It is sufficient to prove that for every $\delta >0$, there exists $f \in \Span\left( \gen{\mathcal{F}} \right)$ such that $\|f - g(\actVr,\cdot)\|_{\infty} \leq \delta.$ Fix $\delta >0$. We have that $\lambda_{\actVr} \left( \mathcal{F},g,\frac{\delta}{2} \right) < + \infty$, therefore there exists, $(\lambda_i)_{i\in \mathbb{N}}$ in $\mathbb{R}$ and $(f_i)_{i\in \mathbb{N}}$ in $\gen{\mathcal{F}}$ such that, $\sum_{i=1}^{\infty} |\lambda_i| < \infty$ and $\limsup_{m \to \infty} \left \| \sum_{i=1}^m \lambda_i \cdot f_i - g(x,\cdot) \right \|_{\infty} \leq \frac{\delta}{2}$. In particular, for $m$ large enough, we have that $\left \| \sum_{i=1}^m \lambda_i \cdot f_i - g(x,\cdot) \right \|_{\infty} \leq \delta.$ Consider such an $m$. By definition of the span we have that $\sum_{i=1}^m \lambda_i \cdot f_i \in \Span\left( \gen{\mathcal{F}} \right)$. This concludes the proof.
\end{proof}

\begin{proof}[\textbf{Proof of \Cref{prop:lower_bound_approx}}]
Assume for sake of contradiction that 
\begin{equation*}
\limsup_{\epsilon \to 0} \frac{1}{\epsilon} \cdot \sup_{\mesOut \in \mesSp} \sup_{\nu\in \mathcal{U}_{\epsilon}(\mesOut)} \sup_{\actVr \in \actSp}|\g[\actVr]{\mu}-\g[\actVr]{\nu}| < + \infty.
\end{equation*}
Therefore, there exists $C > 0$ and $\epsilon_0 > 0$ such that, for every $\epsilon \leq \epsilon_0$,
\begin{equation}
\label{eq:to_contradict}
 \frac{1}{\epsilon} \cdot \sup_{\mesOut \in \mesSp} \sup_{\nu\in \mathcal{U}_{\epsilon}(\mesOut)} \sup_{\actVr \in \actSp}|\g[\actVr]{\mu}-\g[\actVr]{\nu}|
\leq C.
\end{equation}
Fix $\actVr \in \actSp$. We note that, \eqref{eq:to_contradict} implies that, for every $\epsilon \leq \epsilon_0$,
\begin{equation}
\label{eq:small_bound}
\sup_{\substack{\mesOut , \nu \in  \mesSp\\ \text{s.t $d_{\mathcal{F}}(\mesOut,\nu) = \epsilon$} }} |\g[\actVr]{\mu}-\g[\actVr]{\nu}| \leq \sup_{\mesOut \in \mesSp} \sup_{\nu\in \mathcal{U}_{\epsilon}(\mesOut)} |\g[\actVr]{\mu}-\g[\actVr]{\nu}| \leq C \cdot \epsilon,
\end{equation}
where the first inequality follows from the definition of $\mathcal{U}_\epsilon(\mesOut)$. This establishes that for every $\mu, \nu$ such that $d_{\mathcal{F}}(\mesOut,\nu) \leq \epsilon_0$, we have that,
\begin{equation*}
|\g[\actVr]{\mu}-\g[\actVr]{\nu}| \leq C \cdot d_{\mathcal{F}}(\mesOut,\nu).
\end{equation*}
Let $U = \sup_{\actVr \in \actSp} \|g(\actVr,\cdot)\|_{\infty}$.  For every $\mesOut, \nu \in \mathcal{P}$ such that, $d_{\mathcal{F}}(\mesOut,\nu) > \epsilon_0$, we have that,
\begin{equation}
\label{eq:large_bound}
|\g[\actVr]{\mu}-\g[\actVr]{\nu}| \leq 2 \cdot U = 2 U \cdot \frac{\epsilon_0}{\epsilon_0} \leq \frac{2 U}{\epsilon_0} \cdot d_{\mathcal{F}}(\mesOut,\nu),
\end{equation}
where the last inequality holds because, $d_{\mathcal{F}}(\mesOut,\nu) > \epsilon_0$.

Therefore, \eqref{eq:small_bound} and \eqref{eq:large_bound} imply that for every $\mesOut, \nu \in \mesSp$,
\begin{equation*}
|\g[\actVr]{\mu}-\g[\actVr]{\nu}| \leq \max \left( C, \frac{2 U}{\epsilon_0} \right) \cdot d_{\mathcal{F}}(\mesOut,\nu).
\end{equation*}
Let $C_1 = \max \left( C, \frac{2 \cdot U}{\epsilon_0} \right)$. The last inequality implies that $\frac{1}{C_1}\cdot g(\actVr,\cdot)$ belongs to $\gen{\mathcal{F}}$ for every $\actVr \in \actSp$. By definition of the approximation parameter this would imply that $\lambda(\gen{\mathcal{F}}, g, 0) \leq C_1 < \infty$ which contradicts our assumption.
\end{proof}

\subsection{Proofs for \Cref{sec:SAA_bound} }

\begin{proof}[\textbf{Proof of \Cref{cor:dro_to_reg}}]
Fix $\epsilon >0$. We note that by \Cref{prop:etc_Kol} and \Cref{cor:ec_War} both the Kolmogorov distance and the Wasserstein distance satisfy the ETC property on $[0,M]$.
Therefore for any $\mathcal{I} \in \{\mathcal{I}_K,\mathcal{I}_W \}$, we have by \Cref{thm:reduction_stat},
\begin{equation}
\label{eq:rel_to_infinity}
\mathfrak{R}_{\mathcal{I},\infty}^{\SAA}(\epsilon) \leq \lim_{\eta \to \epsilon^{+}} \mathfrak{R}_{\scaleto{\mathcal{I},\mathrm{DRO}}{6pt}}^{\SAA} (\eta). 
\end{equation}

Furthermore, for every $\mesOut, \nu \in \mesSp$, fix $\delta >0$ and $\actVr$ such that $\opt(\mu)-\g[\actVr]{\mu} \leq \delta$ (this $\actVr$ exists by the definition of the supremum). Then we have that,
\begin{align*}
\opt(\mu)-\g[\SAA(\nu)]{\mu}
&= \opt(\mu) - \g[\actVr]{\nu}\\
&\quad + \g[\actVr]{\nu} - \g[\SAA(\nu)]{\nu}\\
&\quad + \g[\SAA(\nu)]{\nu}
- \g[\SAA(\nu)]{\mu}\\
&\stackrel{(a)}{\leq} \g[\actVr]{\mu} + \delta - \g[\actVr]{\nu} + \g[\SAA(\nu)]{\nu}
- \g[\SAA(\nu)]{\mu}\\
& \stackrel{(b)}{\le} 2 \cdot  \UB_{\mathcal{I}} (\mesOut,\nu) + \delta,
\end{align*}
where $(a)$ holds because $\opt(\mu) \leq \g[\actVr]{\mu} + \delta$ and the middle difference $\g[\actVr]{\nu} - \g[\SAA(\nu)]{\nu}$ is non-positive as $\SAA(\nu)$ is optimal for the distribution $\nu$
and $(b)$ follows from the definition of $\UB_{\mathcal{I}} (\mesOut,\nu)$.

Given that this inequality holds for every $\delta >0$ we have by taking the limit as $\delta$ goes to $0$ that,
\begin{equation*}
    \opt(\mu)-\g[\SAA(\nu)]{\mu} \leq  2 \cdot  \UB_{\mathcal{I}} (\mesOut,\nu).
\end{equation*}
By taking a supremum over $\mesOut$ and $\nu$, we obtain that for every $\eta >0$,
\begin{equation*}
\mathfrak{R}_{\scaleto{\mathcal{I},\mathrm{DRO}}{6pt}}^{\SAA} (\eta) \leq 2 \sup_{\mesOut \in \mesSp} \sup_{\nu \in \mathcal{U}_{\eta}(\mesOut)} \UB_{\mathcal{I}} (\mesOut,\nu).
\end{equation*}
In what follows, we conclude the proof for the Wasserstein distance (the proof for the Kolmogorov distance follows the same argument).

We next note that for every $\mesOut, \nu \in \mesSp$,
\begin{align*}
\UB_{\mathcal{I}_W} (\mesOut,\nu) 
&\stackrel{(a)}{\leq} \inf_{\delta \geq 0} \left \{ \lambda \left( \mathcal{M}_W, g, \delta \right) \cdot \dis{\mathcal{F}}(\mesOut,\nu) + 2\delta \right \}\\
&\leq \lambda \left( \mathcal{M}_W, g, 0 \right) \cdot \dis{\mathcal{F}}(\mesOut,\nu) \stackrel{(b)}{\leq} \sup_{\actVr \in \actSp} \Lip{g(\actVr, \cdot)} \cdot \dis{\mathcal{F}}(\mesOut,\nu),
\end{align*}
where $(a)$ follows from \Cref{thm:generating_class} and $(b)$ from \Cref{lem:bound_apx_param}. Hence, 
\begin{equation*}
\mathfrak{R}_{\scaleto{\mathcal{I},\mathrm{DRO}}{6pt}}^{\SAA} (\eta) \leq 2 \sup_{\mesOut \in \mesSp} \sup_{\nu \in \mathcal{U}_{\eta}(\mesOut)} \UB_{\mathcal{I}} (\mesOut,\nu) \leq 2 \sup_{\actVr \in \actSp} \Lip{g(\actVr, \cdot)} \cdot \eta.
\end{equation*}
We obtain the desired result by taking $\eta \to \epsilon^{+}$ and using \eqref{eq:rel_to_infinity}.

\end{proof}

\begin{proof}[\textbf{Proof of \Cref{cor:SAA_Newsvendor}}]
\textit{Step 1: Regret upper bound for all distances.}
We first derive an upper bound on the worst-case asymptotic regret of SAA for the Kolmogorov and the Wasserstein distances. To do so, we will leverage \Cref{cor:dro_to_reg}.

We note that, for the newsvendor problem, we have that,
$$\sup_{\actVr \in \actSp} V(g(\actVr,\cdot)) = \sup_{\actVr \in [0,M]} \left \{ c_u \cdot \actVr + c_o \cdot ( M - \actVr) \right \} = \max(c_u,c_o) \cdot M,$$
and
$$\sup_{\actVr \in \actSp} \Lip{g(\actVr,\cdot)} =\max(c_u,c_o).$$
Therefore, \Cref{cor:dro_to_reg} implies that,
\begin{align*}
\mathfrak{R}_{K,\infty}^{\SAA}(\epsilon) \leq 2 \max \left( c_u, c_o\right) \cdot M \cdot \epsilon \quad \mbox{and} \quad
\mathfrak{R}_{W,\infty}^{\SAA}(\epsilon) \leq 2 \max \left( c_u, c_o\right)  \cdot \epsilon.
\end{align*}

\textit{Step 2: Regret lower bound for Kolmogorov.} We now show the lower bound for the Kolmogorov distance.

 For every $p \in [0,1]$, let $\mathcal{B}_M(p)$ denote the two-point-mass distribution which puts mass $p$ at $M$ and mass $1-p$ at $0$. Furthermore, let $q:= \frac{c_u}{c_u + c_o}$ denote the critical fractile. 
Fix $\epsilon \leq \min \left(q,1-q\right)$ and let   $ \tilde{\nu} := \mathcal{B}_M\left( 1-q \right)$ , $\mu_{-} := \mathcal{B}_M \left( 1- q - \epsilon \right)$ and $\mu_{+} := \mathcal{B}_M \left( 1- q + \epsilon \right)$. Note that for the Kolmogorov distance we have that $\tilde{\nu} \in \mathcal{U}_{\mu_{-}}(\epsilon) \cap \mathcal{U}_{\mu_{+}}(\epsilon).$ Therefore the intersection is non-empty and by \Cref{prop:lower_bound}, we have for every $\numS \geq 1$ and any $\pi \in \Pi$ that,
\begin{equation*}
\mathfrak{R}_{K,\numS}^{\pi}(\epsilon) \geq \sup_{\beta \in \Delta \left( \{ \mu_{-}, \mu_{+} \} \right) } \inf_{\actVr \in \actSp} \mathbb{E}_{\mesOut \sim \beta} \left[ \g[\actVr]{\mesOut}  - \opt(\mesOut) \right]. 
\end{equation*}
Furthermore, remark that for every $\actVr \in [0,M]$ we have,
\begin{equation*}
    \g[\actVr]{\mesOut_{+}}  - \opt(\mesOut_{+}) =  \left(q-\epsilon \right) \cdot c_o \cdot  x +  \left(1-q+\epsilon \right) \cdot c_u \cdot (M-x) - c_o \cdot \left( q - \epsilon \right) \cdot M = \left( c_u + c_o \right) \cdot \left(M- x \right) \cdot \epsilon,
\end{equation*}
\begin{equation*}
    \g[\actVr]{\mesOut_{-}}  - \opt(\mesOut_{-}) =  \left(q+\epsilon \right) \cdot c_o \cdot  x +  \left(1-q-\epsilon \right) \cdot c_u \cdot (M-x) - c_u \cdot \left(1- q - \epsilon \right) \cdot M = \left( c_u + c_o \right) \cdot x \cdot \epsilon.
\end{equation*}
Finally, let $\tilde{\beta}$ be the prior which puts mass $\frac{1}{2}$ on $\mu_{+}$ and mass $\frac{1}{2}$ on $\mu_{-}$. We have that,
\begin{equation*}
    \inf_{\actVr \in \actSp} \mathbb{E}_{\mesOut \sim \tilde{\beta}} \left[ \g[\actVr]{\mesOut}  - \opt(\mesOut) \right] = \inf_{\actVr \in \actSp} \frac{c_u + c_o}{2} \cdot \left( M- x \right) \cdot \epsilon + \frac{c_u + c_o}{2} \cdot x \cdot \epsilon = \frac{c_u + c_o}{2} \cdot M \cdot \epsilon.
\end{equation*}
Therefore,
\begin{equation*}
 \mathfrak{R}_{K,\numS}^{\pi}(\epsilon) \geq \sup_{\beta \in \Delta \left( \{ \mu_{-}, \mu_{+} \} \right) } \inf_{\actVr \in \actSp} \mathbb{E}_{\mesOut \sim \beta} \left[ \g[\actVr]{\mesOut}  - \opt(\mesOut) \right]  \geq \frac{c_u + c_o}{2} \cdot M \cdot \epsilon.
\end{equation*}
And by taking the limit we obtain the desired result,
\begin{equation*}
    \inf_{\pi \in \Pi} \limsup_{\numS \to \infty} \mathfrak{R}_{K,\numS}^{\pi}(\epsilon) \geq \frac{c_u + c_o}{2} \cdot M \cdot \epsilon.
\end{equation*}

\textit{Step 3: Regret lower bound for Wasserstein.} We now derive the lower bound for the Wasserstein distance. Remark that for every $\epsilon \leq M \cdot \max  \left(q,1-q\right)$,
\begin{align*}
    \inf_{\pi \in \Pi}  \mathfrak{R}_{W,\infty}^{\pi}(\epsilon) &\stackrel{(a)}{\geq}  \inf_{\pi \in \Pi}  \mathfrak{R}_{K,\infty}^{\pi} \left( \frac{\epsilon}{M} \right) \stackrel{(b)}{\geq}   \frac{c_u + c_o}{2} \cdot  \epsilon,\\
\end{align*}
where $(a)$ follows from \eqref{eq:dom_reg} and $(b)$ is a consequence of step 2.

\end{proof}

\begin{proof}[\textbf{Proof of \Cref{prop:Kolmogorov_pricing}}]

\textit{Step 1: Regret upper bound.} 
We note that for the pricing problem,
\begin{equation*}
\sup_{\actVr \in \actSp} V \left( g(\actVr,\cdot) \right) = \sup_{\actVr \in [0,M]} \actVr = M. 
\end{equation*}
Therefore by \Cref{cor:dro_to_reg} we have that,
\begin{equation*}
\mathfrak{R}_{K,\infty}^{\SAA}(\epsilon) \leq 2M \cdot \epsilon.
\end{equation*}

\textit{Step 2: Regret lower bound.}
Let $\tilde{\nu}$, $\mu_{+}$ and $\mu_{-}$  be probability measures supported on $\{\frac{M}{2}, M\}$ such that, $\tilde{\nu}(M) = \frac{1}{2}$, $\mu_{+}(M) = \frac{1}{2} + \epsilon$ and $\mu_{-}(M) = \frac{1}{2} - \epsilon$. Note that for the Kolmogorov distance we have that $\tilde{\nu} \in \mathcal{U}_{\mu_{-}}(\epsilon) \cap \mathcal{U}_{\mu_{+}}(\epsilon).$ Therefore the intersection is non-empty and by \Cref{prop:lower_bound}, we have for every $\numS \geq 1$ that,
\begin{equation*}
\inf_{\pi \in \Pi} \mathfrak{R}_{K,\numS}^{\pi}(\epsilon) \geq \sup_{\beta \in \Delta \left( \{ \mu_{-}, \mu_{+} \} \right) } \inf_{\actVr \in \actSp} \mathbb{E}_{\mesOut \sim \beta} \left[ \opt(\mesOut) - \g[\actVr]{\mesOut}  \right]. 
\end{equation*}

Let $\tilde{\beta}$ be the prior which puts mass $\frac{1}{2}$ on $\mu_{+}$ and mass $\frac{1}{2}$ on $\mu_{-}$. 

Since the optimal price must be in the support of the willingness to pay distribution, we only need to consider the revenue generated by posting prices $\frac{M}{2}$ or $M$. In that case, we have
\begin{equation*}
    \opt(\mesOut_{+}) - \g[\frac{M}{2}]{\mesOut_{+}}  = M \left( \frac{1}{2} + \epsilon \right) - \frac{M}{2} = M \cdot \epsilon, \quad \opt(\mesOut_{-}) - \g[\frac{M}{2}]{\mesOut_{-}} = 0,
\end{equation*}
\begin{equation*}
    \opt(\mesOut_{+}) - \g[M]{\mesOut_{+}}  = 0 \quad \mbox{and} \quad \opt(\mesOut_{-}) - \g[M]{\mesOut_{-}} = \frac{M}{2} - M \left( \frac{1}{2} - \epsilon \right) = M \cdot \epsilon.
\end{equation*}
Therefore,
\begin{align*}
    \inf_{\pi \in \Pi} \mathfrak{R}_{K,\numS}^{\pi}(\epsilon) &\geq  \inf_{\actVr \in \actSp} \mathbb{E}_{\mesOut \sim \tilde{\beta}} \left[ \opt(\mesOut) - \g[\actVr]{\mesOut}   \right]\\
 &= \min\left(\mathbb{E}_{\mesOut \sim \tilde{\beta}} \left[ \opt(\mesOut) - \g[\frac{M}{2}]{\mesOut}   \right],\mathbb{E}_{\mesOut \sim \tilde{\beta}} \left[ \opt(\mesOut) - \g[M]{\mesOut}   \right]  \right) = \frac{M \cdot \epsilon}{2}.
\end{align*}
We conclude the proof of the lower bound by taking the limsup as $\numS$ goes to $\infty$.

\end{proof}

\begin{proof}[\textbf{Proof of \Cref{thm:SAA_Wasserstein_pricing}}]
Fix $\eta >0$. We show that there exists an instance of the data-driven pricing problem under Wasserstein heterogeneity such that the regret of SAA for that instance is larger than $M-\eta$.

We consider the out-of-sample distribution $\mesOut := \delta_{M-\eta}$, that is to say the probability measure which puts all mass at $M-\eta$. Furthermore, we let $\nu:=\delta_{\min\left(M,M-\eta+\epsilon \right)}$. Remark that by definition of the Wasserstein distance, we have that, 
\begin{equation*}
d_W(\mesOut,\nu) = \min\left(M,M-\eta+\epsilon \right) - \left(M-\eta\right) \leq \epsilon.
\end{equation*}
It suffices to show that the asymptotic regret of SAA under our candidate instance is larger than $M-\eta$. We note that for any $\envVr_0 \in \envSp$ the oracle of the point mass distribution $\delta_{\envVr_0}$ is given by  $ \mathtt{ORACLE}\left(\delta_{\envVr_0} \right) = \envVr_0$, which generates a revenue of $\envVr_0$.   
Furthermore, we have that, 
\begin{align*}
\mathbb{E}_{\envVr_i \sim \nu} \left[ \g[\SAA(\numS,\hat{\nu}_{\envSamples})]{\mesOut} \right] 
&\stackrel{(a)}{=}  \g[\SAA(\numS,\nu)]{\mesOut}\\
&= \g[\min \left(M, M-\eta+\epsilon \right)]{\mesOut}\\
&= \min \left(M, M-\eta+\epsilon \right) \cdot \mathbbm{1} \left \{  \min \left(M, M-\eta+\epsilon \right) \leq M- \eta \right \} = 0,  
\end{align*}
where $(a)$ holds because for every sample size $\numS$, the historical samples $\envSamples$ sampled from $\nu$ are almost surely all equal to $\min\left(M,M-\eta+\epsilon \right)$. 

This shows that SAA generates a revenue equal to $0$ because it posts a price that is higher than the actual willingness to pay of the customers. We finally conclude that,
\begin{equation*}
\mathcal{R}_\numS \left( \SAA, \mesOut, \mesIn, \ldots, \nu\right) = \opt \left( {\mesOut} \right) - \mathbb{E}_{\envVr_i \sim \nu} \left[ \g[\SAA(\numS,\hat{\nu}_{\envSamples})]{\mesOut} \right] = M- \eta.
\end{equation*}
By taking the supremum over $\eta >0$  we obtain that,
\begin{equation*}
 \mathfrak{R}_{W,\numS}^{\SAA}(\epsilon) \geq M.
\end{equation*}
And by taking the limsup as $\numS$ goes to $\infty$, we conclude that $ \mathfrak{R}_{W,\infty}^{\SAA}(\epsilon) \geq M$.
We complete the proof by noting that $\mathfrak{R}_{W,\infty}^{\SAA}(\epsilon) \leq M$ as the revenue of the oracle is bounded above by $M$ and the one of SAA is bounded below by $0$.
 
\end{proof}

\subsection{Example of a boundary case for the approximation parameter}\label{sec:apx_boudary}
Given an objective function $g$ and a generating class of functions $\mathcal{F}$, we presented in \Cref{sec:max_gen} a general methology which enables to derive vanishing regret guarantees for SAA when the family of functions $\left( g(\actVr,\cdot) \right)_{\actVr \in \actSp}$ lies in $\Span \left( \gen{\mathcal{F}} \right)$ the linear span of the maximal generator class of $\mathcal{F}$. When the family $\left( g(\actVr,\cdot) \right)_{\actVr \in \actSp}$ is not in the closure of $\Span \left( \gen{\mathcal{F}} \right)$ we demonstrated through examples (see \Cref{sec:pricing_SAA,sec:ski_W}) that SAA can fail and does not achieve a vanishing regret.   

In this section, we consider the Kolmogorov distance and present a boundary case, for which $\left( g(\actVr,\cdot) \right)_{\actVr \in \actSp}$ does not belongs to $\Span \left( \gen{\mathcal{K}} \right)$ but belong to its closure with respect to the uniform convergence. In this case, \Cref{cor:dro_to_reg}  is insufficient to conclude and one needs to use \Cref{thm:generating_class} with a finer bound on the approximation parameter.

We first define the notion of modulus of continuity.
\begin{definition}[Modulus of continuity]
Consider a non-decreasing function $\omega : [0,\infty) \to [0,\infty)$ continuous at $0$ and such that $\omega(0)=0$. Then we say that a function $f : \mathbb{R} \to \mathbb{R}$ admits $\omega$ as modulus of continuity if and only if, for every $x,y \in \mathbb{R}$,
\begin{equation*}
|f(x) - f(y)| \leq \omega \left( |x - y | \right).
\end{equation*}
\end{definition}
We next present a general bound for Hölder continuous functions under the Kolmogorov distance.
\begin{proposition}
\label{prop:border}
Consider a data-driven instance $\mathcal{I}$ with Kolmogorov distance such that $\actSp = \envSp =  [0,1]$.  Furthermore, assume that for every $\actVr \in \actSp$, $g(\actVr,\cdot)$ admits $\omega$ as modulus of continuity, where for every $t \geq 0$, $\omega(t) \leq t^{\alpha}$ for some $\alpha > 0$. Then, for every $\mesOut, \nu \in \mesSp$ we have, 
\begin{equation*}
\UB_{\mathcal{I}}(\mesOut,\nu) \leq \frac{5}{2}  \left( G \cdot d(\mesOut,\nu) \right)^{\alpha/(\alpha+2)} + 2 G \cdot d(\mesOut,\nu),
\end{equation*}
where $G = \sup_{\actVr \in \actSp} \|g(\actVr,\cdot) \|_{\infty}$.
\end{proposition}
We note that \Cref{prop:border} enables to derive bounds on the uniform deviation $\UB_{\mathcal{I}}(\mesOut,\nu)$ under the Kolmogorov distance. This bound can be translated into a bound on the asymptotic worst-case regret of SAA by leveraging \ref{thm:reduction_stat} (see proof of \Cref{cor:SAA_Newsvendor}). We did so for objective functions that are not necessarily in the linear span of $\gen{\mathcal{K}}$. In particular, we do not assume that these objective functions have bounded variation. However, these objective functions can be approximated well enough by functions with bounded variation therefore one can derive bounds on the DRO regret by developing a tighter analysis of the approximation parameter.

We also note that the dependence in $\epsilon$ derived in \Cref{prop:border}  is not tight across instances. For instance, in the special case where $g$ is $1$-Lipschitz (i.e., $\alpha =1$), our result implies a bound which scales as $\mathcal{O} \left( \epsilon^{\frac{1}{3}} \right)$, whereas one can show a $\mathcal{O} \left( \epsilon \right)$ dependence by leveraging the relation between the regret under Kolmogorov and the regret under Wasserstein distances (see \Cref{sec:apx_rel}).

\begin{proof}[\textbf{Proof of \Cref{prop:border}}]
Fix an integer $k \geq 1$. We next prove a bound on the approximation parameter.

Fix $\actVr \in \actSp$, we approximate the continuous function $g_{\actVr} := g(\actVr, \cdot)$ by using Berstein polynomials defined as follows. For every integer $q \geq 1$, we first consider the family of elementary polynomial $(b_{p,q})_{p \in \{0,\ldots,q\}}$ defined for $p \in \{0,\ldots,q\}$ and every $y \in [0,1]$ as,
\begin{equation*}
b_{p,q}(y) = {q \choose p} \cdot y^p \cdot (1 - y)^{q-p}.
\end{equation*}
Given a function $f$, the $q^{th}$ Bernstein polynomial is then defined for every $y \in [0,1]$ as 
\begin{equation*}
B_q(f)(y) = \sum_{p=0}^q f \left( \frac{p}{q} \right) \cdot b_{p,q}(y).
\end{equation*}
It is known by Weierstrass theorem  that for any continuous function $f$ on $[0,1]$, $B_q(f)$ uniformly converges to $f$ as $q \to \infty$. The following lemma provides the rate of convergence.
\begin{lemma}[\cite{popoviciu1935approximation}]
\label{lem:modulus}
For any $q \geq 1$ and any continuous function $f$ on $[0,1]$ with modulus of continuity $\omega_{f}$, we have that,
\begin{equation*}
\|B_{q}(f) -f  \|_{\infty} \leq \frac{5}{4} \cdot \omega_{f} \left( \sqrt{\frac{1}{q}} \right).
\end{equation*}
\end{lemma}
Furthermore, note that for every $q \geq 1$ and any $p \leq q$, we have that $V(b_{p,q}) \leq 2$ because $b_{p,q}$ is a unimodal function taking value between $[0,1]$. Therefore $\frac{b_{p,q}}{2} \in \gen{\mathcal{K}}$.  
Hence, for every $p \in \{0,\ldots,k\}$, by setting $\lambda^{\actVr}_{p} = 2  g_{\actVr} \left (\frac{p}{k} \right)$, we obtain that  $B_k(g_{\actVr}) = \sum_{p=0}^k \lambda^{\actVr}_{p} \cdot \frac{b_{p,k}}{2}$ and by applying \Cref{lem:modulus} we get that $\| \sum_{p=0}^k \lambda^{\actVr}_k \cdot \frac{b_{p,k}}{2} - g_{\actVr} \|_{\infty} \leq \frac{5}{4} \cdot \omega \left( \frac{1}{\sqrt{k}} \right)$. In turn, we obtain
\begin{equation}
\label{eq:bound_lambda}
\lambda \left( \gen{\mathcal{K}},g, \frac{5}{4} \cdot \omega \left( \frac{1}{\sqrt{k}} \right) \right) \stackrel{(a)}{\leq} \sup_{\actVr \in \actSp} \sum_{p=0}^q |\lambda^{\actVr}_{p}| \leq 2 \sup_{\actVr \in \actSp} \sum_{p=0}^{k} \left| g_{\actVr} \left(  \frac{p}{k}\right) \right| \leq 2 k \cdot \sup_{\actVr \in \actSp} \|g(\actVr,\cdot) \|_{\infty},
\end{equation}
where $(a)$ holds because for every $\actVr \in \actSp$ the solution $ \left(\lambda^{\actVr}_p \right)_{p \in \{0,\ldots,k\}}$ and $(\frac{b_{p,k}}{2})_{p \in \{0,\ldots,k\}}$ is feasible for the inner minimization problem in the definition of the approximation parameter.
Let $G := \sup_{\actVr \in \actSp}  \|g(\actVr,\cdot) \|_{\infty}$. By applying \Cref{thm:generating_class} we obtain that, for every $\mesOut, \nu \in \mesSp$ such that $\mesOut \neq \nu$,
\begin{align*}
\UB_{\mathcal{I}}(\mesOut,\nu) &\leq   \inf_{k \geq 1} \left \{ \lambda \left( \gen{\mathcal{K}},g, \frac{5}{4} \cdot \omega \left( \frac{1}{\sqrt{k}} \right) \right) \cdot d(\mesOut,\nu) + \frac{5}{2} \cdot \omega \left( \frac{1}{\sqrt{k}} \right) \right \}\\
&\stackrel{(a)}{\leq}    \inf_{k \geq 1} \left \{ 2 k \cdot G \cdot d(\mesOut,\nu) + \frac{5}{2} \cdot \omega \left( \frac{1}{\sqrt{k}} \right) \right \}\\
&\stackrel{(b)}{\leq}   \inf_{k \geq 1} \left \{ 2 k \cdot G \cdot d(\mesOut,\nu) + \frac{5}{2} \cdot \frac{1}{k^{\alpha/2}} \right \}\\
&\stackrel{(c)}{\leq} 2  G \cdot d(\mesOut,\nu) \cdot \left \lceil \left(\frac{5}{4 G \cdot d(\mesOut,\nu)}\right)^{2/(2+\alpha)} \right \rceil\\
&\leq 2  G \cdot d(\mesOut,\nu) \cdot \left( \left( \frac{5}{4 G \cdot d(\mesOut,\nu)}\right)^{2/(2+\alpha)} +1 \right) \leq  \frac{5}{2}  \left( G \cdot d(\mesOut,\nu) \right)^{\alpha/(\alpha+2)} + 2 G \cdot d(\mesOut,\nu),
\end{align*}
where $(a)$ follows from \eqref{eq:bound_lambda}, $(b)$ holds by assumption on the modulus of continuity $\omega$ and $(c)$ holds because the real value $t = \left(\frac{5}{4 G \cdot d(\mesOut,\nu)}\right)^{2/(2+\alpha)}$ equalizes the terms $2t \cdot G \cdot \epsilon$ and $\frac{5}{2} \cdot \frac{1}{t^{\alpha/2}} $, therefore the integer $\tilde{k} = \left \lceil \left(\frac{5}{4 G \cdot d(\mesOut,\nu)}\right)^{2/(2+\alpha)} \right \rceil$ satisfies $2\tilde{k} \cdot G \cdot d(\mesOut,\nu) \geq\frac{5}{2} \cdot \frac{1}{\tilde{k}^{\alpha/2}}$.
\end{proof}

\section{Proofs for  \Cref{sec:beyond_SAA} and additional results}
\label{sec:apx_beyond_SAA}

\subsection{Proofs for \Cref{sec:rate_opt_pricing} }

For notational convenience we define in this subsection the mapping,
\begin{equation*}
 \mathtt{ORACLE}: \begin{cases}
      \mesSp \to \actSp\\
      \mesOut \mapsto \actVr \text{ s.t. } \actVr \in \argmax_{\actVr \in \actSp} \g[\actVr]{\mesOut}.
 \end{cases}   
\end{equation*}
We note that for the pricing problem, this function is well-defined because $\argmax_{\actVr \in \actSp} \g[\actVr]{\mesOut}$ is never empty (as we are maximizing an upper semicontinuous function over a compact set).

\begin{lemma}
\label{lem:cdf_vs_dist}
Let $\mesOut$ and $\nu$ be two probability measures on $\mathbb{R}$ with respective cumulative distributions $F$ and $H$. Then, for every $\envVr_1 \leq \envVr_2$, we have that,
$F(\envVr_1)  \leq H(\envVr_2) + \frac{d_W\left(\mu,\nu \right)}{\envVr_2 - \envVr_1}.$
\end{lemma}

\begin{proof}[\textbf{Proof of \Cref{lem:cdf_vs_dist}}]
Let $F$ (resp. $H$) be the cdf of $\mesOut$ (resp. $\nu$) and consider $\envVr_1 < \envVr_2$. Then we have that,
\begin{align*}
d_W \left( \mesOut,\nu\right) &= \int_{\envVr \in \envSp} |F(\envVr) - H(\envVr)| d\envVr\\
 &\geq \int_{\envVr_1}^{\envVr_2 }  F(\envVr) - H(\envVr) d\envVr \stackrel{(a)}{\geq} \int_{\envVr_1}^{\envVr_2 }  F(\envVr_1) - H(\envVr_2) d\envVr
 = \left( \envVr_2 - \envVr_1\right) \cdot \left(F(\envVr_1) - H(\envVr_2) \right),
\end{align*}
where $(a)$ holds because $F$ and $H$ are non-decreasing functions.
\end{proof}

\begin{proof}[\textbf{Proof of \Cref{prop:diff_rev}}]
In what follows, we associate to the measure $\mesOut$ (resp. $\nu$) its cumulative distribution $F$ (resp. $H$). 
We have that,
\begin{align*}
\g[\actVr_2 ]{\nu}   - \g[\actVr_1 ]{\mesOut} &= \actVr_2 \cdot \left( 1 - H \left( \actVr_2 \right) \right) - \actVr_1 \cdot \left( 1- F \left( \actVr_1 \right) \right)\\
&= \actVr_2 \cdot \left( F(\actVr_1) - H \left( \actVr_2 \right) \right) + \left( \actVr_2 - \actVr_1 \right) \cdot \left(1 - F \left(\actVr_1 \right) \right)\\
&\stackrel{(a)}{\leq} \left(\actVr_2  -\actVr_1 \right) \cdot \left( 1 - F \left( \actVr_1 \right) \right) + \actVr_2 \cdot \frac{d_W \left( \mesOut,\nu\right) }{ \actVr_2 - \actVr_1} \leq \left(\actVr_2  -\actVr_1 \right) +  M \cdot \frac{d_W \left( \mesOut,\nu\right) }{ \actVr_2 - \actVr_1},
\end{align*}
where $(a)$ is a consequence of \Cref{lem:cdf_vs_dist}.
\end{proof}

\begin{proof}[\textbf{ Proof of \Cref{thm:ub_pricing_W}}]

\textit{Step 1: Regret upper bound.} Fix $\delta > 0$. We first show the following upper bound on the asymptotic worst-case regret of $\dev{\delta}$.
\begin{equation}
\label{eq:pricing_ub_general_delta}
\mathfrak{R}_{W,\infty}^{\dev{\delta}}(\epsilon)  \leq 2 \sqrt{M \cdot \epsilon} + \delta + M \cdot \frac{\epsilon}{\delta}
\end{equation}
To prove this result, we first bound the quantity 
$\sup_{\mesOut \in \mesSp} \sup_{\nu \in \mathcal{U}_\epsilon(\mesOut)} | \opt(\mesOut) -  \g[\dev{\delta}(\nu)]{\mesOut}|$
and then apply \Cref{thm:reduction_stat}.

Consider $\mesOut, \nu \in \mesSp$. We have that,
\begin{align}
\opt(\mesOut) - \g[\dev{\delta}(\nu)]{\mesOut} &= \left( \opt(\mesOut) - \opt(\nu) \right) + \left( \opt(\nu) -   \g[\dev{\delta}(\nu)]{\mesOut} \right) \nonumber \\
&\stackrel{(a)}{\leq} 2 \sqrt{M \cdot d_{W}\left(\mesOut,\nu \right)} + \delta + M \cdot \frac{d_W \left(\mesOut,\nu\right)}{\delta}. \label{eq:step1_dev_SAA}
\end{align}
where $(a)$ follows from the following lemma.

\begin{lemma}
\label{lem:compare_rev}
Consider the pricing problem.  Fix $\delta > 0$. For every $\mesOut$ and $\nu \in \mesSp$, we have that,
\begin{align*}
\opt \left(\mesOut\right) - \opt \left( \nu \right) &\leq 2 \sqrt{M \cdot d_{W}\left(\mesOut,\nu \right)},\\
\opt(\nu) -   \g[\dev{\delta}(\nu)]{\mesOut} &\leq \delta + M \cdot \frac{d_W \left(\mesOut,\nu\right)}{\delta}.
\end{align*}
\end{lemma}
By taking the supremum in \eqref{eq:step1_dev_SAA}, we obtain that
\begin{equation}
\label{eq:DRO_ub_pricing_W}
\sup_{\mesOut \in \mesSp} \sup_{\nu \in \mathcal{U}_\epsilon(\mesOut)} | \opt(\mesOut) -  \g[\dev{\delta}(\nu)]{\mesOut}| \leq 2 \sqrt{M \cdot \epsilon} + |\delta| + M \cdot \frac{\epsilon}{|\delta|}.
\end{equation}
We conclude that \eqref{eq:pricing_ub_general_delta} holds by applying \Cref{thm:reduction_stat} and by using the upper bound derived in \eqref{eq:DRO_ub_pricing_W} on the uniform DRO problem.

Furthermore, we remark that $\delta = \sqrt{M \cdot \epsilon}$ minimizes the RHS of \eqref{eq:pricing_ub_general_delta} and yields the bound,
\begin{equation*}
\mathfrak{R}_{W,\infty}^{\dev{\delta}}(\epsilon)  \leq 4 \sqrt{M \cdot \epsilon}.
\end{equation*}
This shows the first part of the theorem.

\textit{Step 2: Regret lower bound.} We next provide a universal lower bound across all possible data-driven policies. Fix $\epsilon > 0$ and let,  $\tilde{\nu} = \delta_{\frac{M}{2}}$ the probability measure which puts all the mass at $\frac{M}{2}$. Let $\tilde{\mesOut}$ be the probability measure defined as, 
\begin{equation*}
\tilde{\mesOut} \left( \envVr \right) := \begin{cases}
2 \sqrt{ \frac{\epsilon}{M}} \qquad \text{ if $\envVr = \frac{M}{2} - \frac{\sqrt{M \cdot \epsilon} }{2}$}\\
1 - 2 \sqrt{ \frac{\epsilon}{M}} \qquad \text{ if $\envVr = \frac{M}{2} $}.
\end{cases}
\end{equation*}
Note that for the Wasserstein distance we have that $\tilde{\nu} \in \mathcal{U}_{\tilde{\mu}}(\epsilon)$. Therefore $\mathcal{U}_{\tilde{\mu}}(\epsilon) \cap \mathcal{U}_{\tilde{\nu}}(\epsilon)$ is non-empty and by \Cref{prop:lower_bound}, we have for every $\numS \geq 1$ that,
\begin{equation*}
 \inf_{\pi \in \Pi} \mathfrak{R}_{W,\numS}^{\pi}(\epsilon) \geq \sup_{\beta \in \Delta \left( \{ \tilde{\mesOut} , \tilde{\nu} \} \right) } \inf_{\actVr \in \actSp} \mathbb{E}_{\mesOut \sim \beta} \left[\opt(\mesOut) - \g[\actVr]{\mesOut}   \right]. 
\end{equation*}

Let $\tilde{\beta}$ be the prior which puts mass $\frac{1}{2}$ on $\tilde{\nu}$ and mass $\frac{1}{2}$ on $\tilde{\mesOut}$.
Then, for every $\actVr \in [0, \frac{M}{2} - \frac{\sqrt{M \cdot \epsilon} }{2}]$, we have that,
\begin{equation*}
\mathbb{E}_{\mesOut \sim \tilde{\beta}} \left[ \opt(\mesOut) - \g[\actVr]{\mesOut}   \right] =   \mathbb{E}_{\mesOut \sim \tilde{\beta}} \left[ \opt(\mesOut)  \right] - \actVr = \frac{M}{2} -  \frac{\sqrt{M \cdot \epsilon}}{4} - \actVr \stackrel{(a)}{\geq} \frac{\sqrt{M \cdot \epsilon}}{4}, 
\end{equation*}
where $(a)$ holds because $\actVr \leq \frac{M}{2} - \frac{\sqrt{M \cdot \epsilon} }{2}$.
Furthermore, for every $\actVr \in \left(\frac{M}{2} - \frac{\sqrt{M \cdot \epsilon} }{2},\frac{M}{2} \right]$, we have that,
\begin{align*}
\mathbb{E}_{\mesOut \sim \tilde{\beta}} \left[ \opt(\mesOut) - \g[\actVr]{\mesOut}   \right] &=\frac{M}{2} -  \frac{\sqrt{M \cdot \epsilon}}{4} - \mathbb{E}_{\mesOut \sim \tilde{\beta}} 
\left[ \g[\actVr]{\mesOut}   \right]  \\
&= \frac{M}{2} -  \frac{\sqrt{M \cdot \epsilon}}{4}  - \left(1-  \sqrt{\frac{\epsilon}{M}}\right) \cdot \actVr \geq   \frac{\sqrt{M \cdot \epsilon}}{4}.
\end{align*}
We note that, for $\actVr > \frac{M}{2}$, we have that $\mathbb{E}_{\mesOut \sim \tilde{\beta}} \left[ \opt(\mesOut) - \g[\actVr]{\mesOut}   \right] = \frac{M}{2} -  \frac{\sqrt{M \cdot \epsilon}}{4}  $ because, $\g[\actVr]{\mesOut} = 0$ when $\mesOut = \tilde{\mu}$ and $\mesOut = \tilde{\nu}$.  

Therefore, 
\begin{equation*}
    \inf_{\pi \in \Pi} \mathfrak{R}_{W,\numS}^{\pi}(\epsilon) \geq \sup_{\beta \in \Delta \left( \{ \tilde{\mesOut} , \tilde{\nu} \} \right) } \inf_{\actVr \in \actSp} \mathbb{E}_{\mesOut \sim \beta} \left[\opt(\mesOut) - \g[\actVr]{\mesOut}   \right]  \geq \frac{\sqrt{M \cdot \epsilon}}{4}.
\end{equation*}
And by taking the limit we obtain the desired result,
\begin{equation*}
    \inf_{\pi \in \Pi} \limsup_{\numS \to \infty} \mathfrak{R}_{W,\numS}^{\pi}(\epsilon) \stackrel{(a)}{=}  \limsup_{\numS \to \infty} \inf_{\pi \in \Pi} \mathfrak{R}_{W,\numS}^{\pi}(\epsilon) \geq  \frac{\sqrt{M \cdot \epsilon}}{4},
\end{equation*}
where $(a)$ holds because $\pi$ observes $\numS$.

\end{proof}

\begin{proof}[\textbf{Proof of \Cref{lem:compare_rev}}]

\textit{Step 1:} We first prove that, $\opt \left(\mesOut\right) - \opt \left( \nu \right) \leq 2 \sqrt{M \cdot d_{W}\left(\mesOut,\nu \right)}$.

If $\oracle{\mesOut} \leq \sqrt{M \cdot d_{W}\left(\mesOut,\nu \right)}$, we have that
\begin{equation*}
\opt \left(\mesOut\right) - \opt \left( \nu \right) \leq \g[\oracle{\mesOut}]{\mesOut}  \stackrel{(a)}{\leq} \oracle{\mesOut} \leq \sqrt{M \cdot d_{W}\left(\mesOut,\nu \right)},
\end{equation*}
where $(a)$ holds for every $\actVr \in \actSp$ and any distribution $\tilde{\mu}$, we have that $\g[\actVr]{\tilde{\mu}} \leq \actVr$ for the pricing problem. 

Assume that, $\oracle{\mesOut} > \sqrt{M \cdot d_{W}\left(\mesOut,\nu \right)}$. Let $\tilde{\actVr} := \oracle{\mesOut} - \sqrt{M \cdot d_{W}\left(\mesOut,\nu \right)}$.
We have,
\begin{align*}
\opt \left(\mesOut\right) - \opt \left( \nu \right)  &\leq \g[\oracle{\mesOut}]{\mesOut} - \g[\tilde{\actVr}]{\nu}\\
&\stackrel{(a)}{\leq} \oracle{\mesOut} - \tilde{\actVr} + M \cdot \frac{d_W \left(\mesOut,\nu\right)}{\oracle{\mesOut} - \tilde{\actVr}} \\
&= 2 \sqrt{M \cdot d_{W}\left(\mesOut,\nu \right)},
\end{align*}
where $(a)$ follows from \Cref{prop:diff_rev}.

\textit{Step 2:} We now show that, $\opt(\nu) -   \g[\dev{\delta}(\nu)]{\mesOut} \leq \delta + M \cdot \frac{d_W \left(\mesOut,\nu\right)}{\delta}.$

Similarly to step $1$, we note that if, $\oracle{\nu} \leq \delta$, we have that
\begin{equation*}
\opt(\nu) -   \g[\dev{\delta}(\nu)]{\mesOut} \leq \opt(\nu)    \leq \oracle{\nu} \leq \delta.
\end{equation*}
Furthermore, if $\oracle{\nu} > \delta$, then
\begin{align*}
\opt(\nu) -   \g[\dev{\delta}(\nu)]{\mesOut} = \g[\oracle{\nu}]{\nu} -  \g[\oracle{\nu} - \delta]{\mesOut}   \stackrel{(a)}{\leq} \delta + M \cdot \frac{d_W \left(\mesOut,\nu\right)}{\delta},
\end{align*}
where $(a)$ follows from \Cref{prop:diff_rev}.

\end{proof}

\subsection{Proofs for \Cref{sec:UDRO}}

\begin{proof}[\textbf{Proof of \Cref{prop:RDRO_to_DRO}}]
Fix $\epsilon > 0$. We first show that,
\begin{equation}
\label{eq:optimality_DRO_metric}
\mathfrak{R}_{\scaleto{\mathcal{I},\mathrm{DRO}}{6pt}}^{\pi^{\mathrm{DRO}}}(\epsilon) \leq \inf_{\pi \in \agSp} \mathfrak{R}_{\scaleto{\mathcal{I},\mathrm{DRO}}{6pt}}^{\pi}(\epsilon)+ \sup_{\mesOut \in \mesSp}  \sup_{\nu \in \mathcal{U}_{\epsilon}(\mesOut)} \left \{ \opt(\mesOut) - \opt(\nu) \right \}.
\end{equation}
For any fixed $\nu$ and any $\pi \in \agSp$, we have that,
\begin{align*}
\sup_{\mu\in \mathcal{U}_\epsilon(\nu)} (\opt(\mesOut) - \g[\pi^{\mathrm{DRO}}(\nu)]{\mesOut})
&\leq \sup_{\mu\in \mathcal{U}_\epsilon(\nu)} \opt(\mesOut)+\sup_{\mu\in \mathcal{U}_\epsilon(\nu)}\left \{ - \g[\pi^{\mathrm{DRO}}(\nu)]{\mesOut}\right\}
\\ &= \sup_{\mu\in \mathcal{U}_\epsilon(\nu)} \opt(\mesOut) - \inf_{\mu\in \mathcal{U}_\epsilon(\nu)}\g[\pi^{\mathrm{DRO}}(\nu)]{\mesOut}
\\ &\stackrel{(a)}{\leq} \sup_{\mu\in \mathcal{U}_\epsilon(\nu)} \opt(\mesOut) - \inf_{\mu\in \mathcal{U}_\epsilon(\nu)}\g[\pi(\nu)]{\mesOut}
\\ &= \sup_{\mu\in \mathcal{U}_\epsilon(\nu)} \opt(\mesOut) - \inf_{\mu\in \mathcal{U}_\epsilon(\nu)}\left\{\g[\pi(\nu)]{\mesOut}-\opt(\mu)+\opt(\mu)\right\}
\\ &\leq \sup_{\mu\in \mathcal{U}_\epsilon(\nu)} \opt(\mesOut) - \inf_{\mu\in \mathcal{U}_\epsilon(\nu)}\left\{\g[\pi(\nu)]{\mesOut}-\opt(\mu)\right\}- \inf_{\mu\in \mathcal{U}_\epsilon(\nu)}\opt(\mu)
\\ &= \sup_{\mu\in \mathcal{U}_\epsilon(\nu)}\left\{\opt(\mu)-\g[\pi(\nu)]{\mesOut}\right\}
+\sup_{\mu\in \mathcal{U}_\epsilon(\nu)} \opt(\mesOut)
-\inf_{\mu\in \mathcal{U}_\epsilon(\nu)}\opt(\mu)
\end{align*}
Therefore by taking a supremum over $\nu \in \mesSp$ we obtain that,
\begin{align*}
\sup_{\nu \in \mesSp} \sup_{\mu\in \mathcal{U}_\epsilon(\nu)} (\opt(\mesOut) - \g[\pi^{\mathrm{DRO}}(\nu)]{\mesOut})
\leq \sup_{\nu \in \mesSp} \sup_{\mu\in \mathcal{U}_\epsilon(\nu)}(\opt(\mu)-\g[\pi(\nu)]{\mesOut})
+\sup_{\nu \in \mesSp} \sup_{\mu\in \mathcal{U}_{2\epsilon}(\nu)} |\opt(\mesOut)-\opt(\nu)|
\end{align*}
Given  that this last inequality holds for every $\pi \in \agSp$ we obtain the desired inequality \eqref{eq:optimality_DRO_metric} by taking an infimum over $\pi \in \agSp$.

Furthermore, \Cref{lem:lb_agnostic_proof} implies that,
\begin{equation*}
    \inf_{\pi \in \agSp} \mathfrak{R}_{\scaleto{\mathcal{I},\mathrm{DRO}}{6pt}}^{\pi}(\epsilon) \leq \inf_{\pi \in \Pi} \mathfrak{R}_{\mathcal{I},\infty}^{\pi}(\epsilon).
\end{equation*}
Finally, given that $\pi^{\mathrm{DRO}}$ is sample-size agnostic, \Cref{thm:reduction_stat} allows us to lower bound the LHS of \eqref{eq:optimality_DRO_metric} as follows,
\begin{equation*}
    \lim_{\eta \to \epsilon^{-}} \mathfrak{R}_{\mathcal{I},\infty}^{\pi^{\mathrm{DRO}}}(\eta) \leq \mathfrak{R}_{\scaleto{\mathcal{I},\mathrm{DRO}}{6pt}}^{\pi^{\mathrm{DRO}}}(\epsilon).
\end{equation*}
These last two inequalities, combined with \eqref{eq:optimality_DRO_metric}, imply the desired result.
\end{proof}

\subsection{Additional results}
\label{sec:apx_optimal_RDRO}

\begin{proposition}[Optimality of the RDRO policy]
\label{prop:optimal_RDRO}
Let $\mathcal{I} = \left( \actSp, \envSp, \mesSp, d, g, \mathtt{ORACLE} \right)$ be a data-driven decision problem in a heterogeneous environment and assume $d$ is convex and  satisfies the ETC property.
Then for every $\epsilon \geq 0$,
 \begin{equation*}
\lim_{\eta \to \epsilon^{-}}  \mathfrak{R}_{\mathcal{I},\infty}^{\pi^{\mathrm{RDRO}}}(\eta)  \leq \inf_{\pi \in \Pi} \mathfrak{R}_{\mathcal{I},\infty}^{\pi}(\epsilon). 
 \end{equation*}
\end{proposition}

\begin{proof}[\textbf{Proof of \Cref{prop:optimal_RDRO}}]
Note that for every $\epsilon \geq 0$, we have that
\begin{equation*}
\mathfrak{R}_{\scaleto{\mathcal{I},\mathrm{DRO}}{6pt}}^{\pi^{\mathrm{RDRO}}} (\epsilon)  \stackrel{(a)}{=} \inf_{\pi \in \agSp} \mathfrak{R}_{\scaleto{\mathcal{I},\mathrm{DRO}}{6pt}}^{\pi} (\epsilon) \stackrel{(b)}{\leq} \inf_{\pi \in \Pi} \mathfrak{R}_{\mathcal{I},\infty}^{\pi}(\epsilon),
\end{equation*}
where $(a)$ follows from the definition of $\pi^{\mathrm{RDRO}}$ and $(b)$ holds by \Cref{lem:lb_agnostic_proof}.

Finally, because $\pi^{\mathrm{RDRO}}$ is sample-size-agnostic, and $d$ is a convex distance satisfying the ETC property, \Cref{thm:reduction_stat} implies that,
\begin{equation*}
\lim_{\eta \to \epsilon^{-}} \mathfrak{R}_{\mathcal{I},\infty}^{\pi^{\mathrm{RDRO}}}(\eta) \leq \mathfrak{R}_{\scaleto{\mathcal{I},\mathrm{DRO}}{6pt}}^{\pi^{\mathrm{RDRO}}} (\epsilon)  \leq \inf_{\pi \in \Pi} \mathfrak{R}_{\mathcal{I},\infty}^{\pi}(\epsilon)
\end{equation*}

\end{proof}

\section{Results for the ski-rental problem} \label{sec:ski}
We now study a stochastic version of the prototypical ski-rental.
In this problem, renting skis costs $\$ 1$ per unit of time while buying them costs $\$ b$ up-front, for some real value $b$. 
The environment space $\envSp$ represents the set of possible lengths of the ski trip and a decision $x$ represents the duration after which skis should be bought (if the ski trip has not ended by that time).  We call $x$ the \textit{rental duration}.  Let $\actSp = \envSp = [0,M]$. We note that when $\actVr=M$ then skis should never be bought. The set of probability measures $\mesSp$ is the whole space of probability measures on $[0,M]$. For every rental duration $\actVr \in \actSp$ and trip length $\envVr \in \envSp$, the cost incurred is
\begin{equation}
\label{eq:defS}
    g(\actVr,\envVr) = \envVr \cdot \mathbbm{1} \left\{ \envVr \leq \actVr  \right\} + \left( b +  x\right) \cdot \mathbbm{1} \left \{ \envVr > \actVr \right \}.
\end{equation}
Finally, the goal of decision-maker is to minimize costs\footnote{Our framework was introduced for maximization problems (see \Cref{sec:formulation}) but it readily extends to minimization problems.} and we let $\opt(\mesOut) = \inf_{\actVr \in \actSp} \g[\actVr]{\mesOut}$.
Similarly to the pricing problem, one may remark that the ski-rental objective $g$ defined in \eqref{eq:defS} lies in the linear span of the maximum generating class associated to the Kolmogorov distance.  However, since the function $g(\actVr,\cdot)$ is again not continuous for $\actVr > 0$, it does not lie in the closure of the linear span of the generating class for the Wasserstein distance. Therefore, similarly to the pricing problem, one may expect vanishing regret for SAA for the first two distances, but there is no guarantee on the performance of SAA under the Wasserstein distance. We show in \Cref{sec:ski_K} that while the regret of SAA vanishes for Kolmogorov, it does not achieve the optimal dependence in the scale of the problem $M$ and one may actually improve upon SAA. In \Cref{sec:ski_W} we develop a robustified policy for the ski-rental problem under the Wasserstein heterogeneity.

\subsection{Limitation of SAA for ski-rental under Kolmogorov and robustifications}
\label{sec:ski_K}
Under the Kolmogorov distance, we have established in \Cref{sec:newsvendor_SAA} and \Cref{sec:pricing_SAA} that SAA was near-optimal for both the newsvendor and pricing problems. We establish next that the performance of SAA for the ski-rental problem under these distances is vanishing as $\epsilon$ goes to $0$ but, surprisingly, highly suboptimal.
\begin{proposition}
\label{thm:ski_rental_K}
For the data-driven ski-rental problem in a heterogenous environment under the Kolmogorov distance, we have for $\epsilon$ small enough that,
\begin{align*}
 M \cdot \epsilon  &\leq \mathfrak{R}_{K,\infty}^{\SAA}(\epsilon) \leq 2 \left(M +b \right) \cdot \epsilon,\\
\frac{\epsilon \cdot b}{8}  &\leq \inf_{\pi \in \Pi}\mathfrak{R}_{K,\infty}^{\pi}(\epsilon) \leq b \cdot \left[ \log \left(\epsilon^{-1}\right) +2  \right] \cdot \epsilon.
\end{align*}
\end{proposition}

\Cref{thm:ski_rental_K} shows that the asymptotic worst-case regret of SAA  scales linearly with the radius of heterogeneity $\epsilon$ and with the size of the support $M$, under the Kolmogorov distance. We also characterize the rate of the optimal asymptotic worst-case regret as a function of $\epsilon$. The upper bound on the asymptotic regret is obtained in \cite{diakonikolas2021learning} through a variant of the SAA policy, which caps the maximum number of days to rent. We complement this result by providing a matching lower bound up to a logarithmic factor. Note that, in contrast to the pricing and newsvendor problems, the scaling of SAA is not optimal in $M$. As the scale of the support grows, the asymptotic worst-case regret of SAA is considerably worse than the optimal achievable rate. This highlights that the need for robustification can also be present under the Kolmogorov distance.

Similarly to the previous problems, the objective function of the ski-rental problem is such that the approximation parameter of the function $g$ is small with respect to the generating class of the Kolmogorov distance. Therefore, while the analysis and intuition developed through the approximation parameter provides the correct understanding on the performance of SAA with respect to the radius of heterogeneity $\epsilon$, the dependence in other parameters of the problem need to be tackled by analyzing more carefully the structure of the problem of interest.

\subsection{Ski-rental under Wasserstein distance: failure of SAA and near-optimal policy}
\label{sec:ski_W}
We proved in \Cref{sec:pricing_SAA} that SAA performs arbitrarily poorly for pricing under Wasserstein heterogeneity. We now show that SAA falters similarly for the Wasserstein ski-rental problem. Furthermore, we design and analyze an alternative policy which inflates the action selected by SAA and achieves the best rate possible as a function of the heterogeneity radius.
\begin{proposition}
\label{thm:ski_rental_W}
For the data-driven ski-rental problem in a heterogeneous environment under the Wasserstein distance, for any $\epsilon$ small enough, we have
\begin{equation*}
 \frac{b}{2}  \leq \mathfrak{R}_{W,\infty}^{\SAA}(\epsilon) \leq 2 \left(b + \epsilon \right).
 \end{equation*}
Recall the policy $\dev{{\delta}}$ defined in \Cref{sec:rate_opt_pricing}. By letting $\tilde{\delta} = -\sqrt{b \cdot \epsilon}$ we get that,
 \begin{equation*}
 \frac{\sqrt{b \cdot \epsilon}}{4}  \leq \inf_{\pi \in \Pi} \mathfrak{R}_{W,\infty}^{\pi}(\epsilon) \leq \mathfrak{R}_{W,\infty}^{\dev{\tilde{\delta}}}(\epsilon) \leq 4 \cdot \sqrt{b \cdot \epsilon} + 2 \epsilon.
 \end{equation*}

\end{proposition}
\Cref{thm:ski_rental_W} formalizes the failure of SAA for ski-rental under Wasserstein distance. 
A notable fact about the ski-rental problem under Wasserstein heterogeneity is that both the regret of SAA and of the optimal data-driven decision do not scale with the size of the support. For pricing,  SAA scaled linearly with $M$ and the optimal policy scales in $\sqrt{M}$. The fact that $\mathfrak{R}_{W,\infty}^{\SAA}(\epsilon)=O(1)$ in ski-rental requires a separate non-trivial proof in \Cref{thm:ski_rental_W}.\footnote{The fact that $\mathfrak{R}_{W,\infty}^{\SAA}(\epsilon)\le M$ in pricing was vacuous.}

We prove the upper bound on the minimal asymptotic worst-case regret by an argument similar to the one developed for the pricing problem. Indeed, we show that for a well-crafted deviation parameter $\delta$, the policy $\delta$-SAA performs well under the Wasserstein ski-rental problem. In contrast to the pricing problem, we show that in the ski-rental problem, one needs to inflate the decision to robustify SAA (hence we take $\delta$ to be negative in \Cref{thm:ski_rental_W}). In more detail, for  two distributions $\mesOut$ and $\nu\in \mesSp$
and two actions $\actVr_1$ and $\actVr_2 \in \actSp$ such that $\actVr_1 < \actVr_2$, we establish in the appendix a result similar to \Cref{prop:diff_rev} and show that,
\begin{equation*}
\g[\actVr_2] {\mesOut} - \g[\actVr_1] {\nu}   \leq b \cdot \frac{d_W \left(\mesOut,\nu \right)}{\actVr_2 - \actVr_1} +  d_W \left(\mesOut,\nu \right) +  \actVr_2 - \actVr_1.
\end{equation*} 
 We note that this result crucially differs from \Cref{prop:diff_rev} as the RHS does not involve the size of the support $M$. We leverage this inequality to construct a $\delta$-SAA policy achieving rate optimality.

\subsection{Proofs}
\label{sec:apx_ski}

\subsubsection{Proofs of main results}

For notational convenience we define in what follows the mapping,
\begin{equation*}
 \mathtt{ORACLE}: \begin{cases}
      \mesSp \to \actSp\\
      \mesOut \mapsto \actVr \text{ s.t. } \actVr \in \argmin_{\actVr \in \actSp} \g[\actVr]{\mesOut}.
 \end{cases}   
\end{equation*}
We note that for the ski-rental problem, this function is well-defined because $\argmin_{\actVr \in \actSp} \g[\actVr]{\mesOut}$ is never empty (as we are minimizing a lower semicontinuous function over a compact set).

\begin{proof}[\textbf{Proof of \Cref{thm:ski_rental_K}}]
\textit{Step 1-a: Regret upper bound of SAA for Kolmogorov.}
We first characterize the performance of SAA.
We note that for the ski-rental problem,
\begin{equation*}
\sup_{\actVr \in \actSp} V \left( g(\actVr,\cdot) \right) = \sup_{\actVr \in [0,M]} \{ \actVr + b  \} = M + b. 
\end{equation*}
Therefore, \Cref{cor:dro_to_reg} implies that,
\begin{equation*}
\mathfrak{R}_{K,\infty}^{\SAA}(\epsilon) \leq 2\left(M+b\right) \cdot \epsilon.
\end{equation*}

\textit{Step 1-b: Regret lower bound of SAA for Kolmogorov distance.} 
Next, we derive a lower bound on the worst-case asymptotic regret of SAA for the Kolmogorov distance.

The proof proceed as follows.
We first construct a distribution on which one is indifferent between playing any action $x$.
We then set $\nu$ to be a slightly perturbed version of this distribution, so that SAA (with high probability over its samples) errs on the side of playing $x=M$, i.e.\ never buying.
Finally, the true distribution $\mu$ differs from $\nu$ by moving $\epsilon$ mass from $\xi=1$ (skiing for 1 day) to $\xi=M$ (skiing for $M$ days), causing an expected regret of $M\epsilon$ from not buying the skis.

Formally, assume that $M \in \mathbb{N^*}$ and that $b \in \{2, \ldots, M-1\}$. We first note that when $\actSp = \envSp = \{0,\ldots,M\}$, we can rewrite the expected ski-rental objective under the distribution $\mesOut$ as follows. For any $k \in \actSp$ we have,
\begin{equation}
\label{eq:discrete_ski}
\g[k]{\mesOut} = \sum_{i=1}^k \mathbb{P}_{\mesOut} \left( \envVr \geq i \right) + b \cdot \mathbb{P}_{\mesOut} \left( \envVr \geq k+1 \right).
\end{equation}
The cost in~\eqref{eq:discrete_ski} should be interpreted as planning to buy skis upon reaching the \textit{start of day $k+1$}, which is why we pay the rental cost for each day $i=1,\ldots,k$ before $k+1$ and pay the buying cost of $b$ when $\xi\ge k+1$.

We next construct a distribution $\nu$ such that all actions in $\{0,\ldots, M-b-1 \} \cup \{M\}$ yield the same expected objective. That is to say, for every $k \in \{0,\ldots,M-b-2\}$, we have that,$\g[k]{\nu} =\g[k+1]{\nu}$ and $\g[M-b-1]{\nu} =\g[M]{\nu}$.
Note that in that case, it is never optimal to buy in one of the days, $\{M-b, \ldots, M-1\}$ since buying would cost strictly more than renting until the end of the selling season.

Consider the distribution $\nu$ which satisfies the following constraints.
\begin{align}
\nu(k) &= 0 &\forall k \in \{0\} \cup \{M-b+1,\ldots,M-1\} \nonumber \\
\mathbb{P}_{\nu}[\envVr \geq M-b] &= \frac{b}{b-1} \nu(M) \label{eq:M} \\
b \cdot \mathbb{P}_{\nu}[\envVr \geq k+1] &=\mathbb{P}_{\nu}[\envVr \geq k+1 ]+\mathbb{P}_{\nu}[\envVr \geq k+2 ]\cdot b &\forall k \in \{0,\ldots,M-b-2\}. \label{eq:recur}
\end{align}
We note that if these conditions define a valid distribution then all actions of interest have the same expected objective. Indeed, by applying \eqref{eq:discrete_ski} we remark that \eqref{eq:M} implies $\g[M-b-1]{\nu} =\g[M]{\nu}$ and that \eqref{eq:recur} implies $\g[k]{\nu} =\g[k+1]{\nu}$.

We now argue that this distribution is well defined. We obtain by solving the recursion \eqref{eq:recur} that, for every $k \in \{1,\ldots,M-b-1\}$, $\mathbb{P}_{\nu} \left( \envVr \geq k \right) = \frac{b}{b-1} \cdot \mathbb{P}_{\nu} \left( \envVr \geq k+1 \right)$. The boundary condition \eqref{eq:M} and $\mathbb{P}_{\nu} \left( \envVr \geq k \right) = 1$ enables to conclude that this is well defined.
We note that from this construction, we can directly conclude that $\mathbb{E}_{\nu}\left[\envVr \right] = \g[M]{\nu} = \g[0]{\nu} = b$, where the first and last equalities holds by definition of the ski-rental objective and the middle-one follows from our construction. Furthermore $\nu(1)$ and $\nu(M)$ are positive.

We next construct our counter-example for the lower bound. 

Fix $\epsilon \geq \alpha \geq 0.$ 
We define the distributions $\nu_{\alpha}$ and $\mesOut$ as follows. For every $k \in \{1,\ldots,M\}$,
\begin{align*}
\mathbb{P}_{\nu_{\alpha}}\left( \envVr \geq k \right) &= \mathbb{P}_{\nu}\left( \envVr \geq k \right) - \alpha\\
\mathbb{P}_{\mesOut}\left( \envVr \geq k \right) &= \mathbb{P}_{\nu_{\alpha}}\left( \envVr \geq k \right) + \epsilon.
\end{align*}
Note that this construction is well defined for $\epsilon$ and $\alpha$ small enough as $\nu(1)$ and $\nu(M)$ are positive.

In what follows, $\nu_{\alpha}$ will be the in-sample distribution and $\mesOut$ the out of sample distribution. These two distributions have a total variation of $\epsilon$. Furthermore, by denoting by $B_{\numS}$ the event under which SAA never buys, we get that, 
\begin{align*}
\mathfrak{R}_{K,\numS}^{\SAA}(\epsilon) &\stackrel{(a)}{\geq }\mathcal{R}_\numS \left( \SAA, \mesOut, \nu_{\alpha}, \ldots, \nu_{\alpha} \right) \nonumber \\ 
&= \mathbb{E}_{\nu_{\alpha}} \left[ \g[\SAA(\numS,\hat{\nu}_{\envSamples})]{\mesOut} \right] - \opt(\mesOut) \nonumber \\
&\geq \mathbb{E}_{\nu_{\alpha}} \left[ \g[\SAA(\numS,\hat{\nu}_{\envSamples})]{\mesOut} - \opt(\mesOut) \, \vert \,  E_\numS  \right]  \cdot \mathbb{P}_{\nu_{\alpha}} \left( B_\numS \right) \stackrel{(b)}{\geq} (\g[M]{\mesOut} - \g[0]{\mesOut}) \cdot \mathbb{P}_{\nu_{\alpha}} \left( B_\numS \right),
\end{align*}
where $(a)$ holds because $\nu \in \mathcal{U}_\epsilon(\mesOut)$ and $(b)$ follows from the fact that under the event $B_\numS$, SAA never buys and by bounding the cost of the oracle by the one obtained by purchasing at day $0$.

We conclude the proof by showing that $\mathbb{P}_{\nu_{\alpha}}  \left( B_\numS \right) \to 1$ as $\numS \to \infty$ and noting that,
\begin{equation*}
\g[M]{\mesOut} - \g[0]{\mesOut} = \mathbb{E}_{\mesOut} [ \envVr] - b \stackrel{(a)}{=} \mathbb{E}_{\nu} [ \envVr]  + (\epsilon - \alpha) \cdot M -  b =   (\epsilon - \alpha) \cdot M,
\end{equation*}
where $(a)$ follows from the construction of $\mesOut$ and $\nu_{\alpha}$.
By taking the limit as $\numS$ goes to $\infty$, we obtain that
\begin{equation*}
\mathfrak{R}_{K,\numS}^{\SAA}(\epsilon) \geq (\epsilon - \alpha) \cdot M
\end{equation*}
and since this inequality holds for any $\alpha$ small enough we conclude that,
\begin{equation*}
\mathfrak{R}_{K,\numS}^{\SAA}(\epsilon) \geq  \epsilon \cdot M.
\end{equation*}

We complete the proof by showing that  $\mathbb{P}_{\nu_{\alpha}}  \left( B_\numS \right) \to 1$ as $\numS \to \infty$.

Fix $\delta >0$.
Given samples $\envSamples$ from $\nu$, the decision of SAA is completely characterized by the empirical cdf $\hat{\nu}$.
Consider the event,
 $$E_\numS = \left\{ \hat{\nu}(k) \in [(1-\delta) \nu(k), (1+\delta) \nu(k)] , \text{ for all $k \in \{0,\ldots,M\}$} \right\}.$$
We next show that $E_\numS \subset B_\numS$, i.e., when $E_\numS$ holds, SAA never buys.

 Remark that SAA never buys if for each day $k$, buying at the end of day $k$ leads to a larger cost than the cost to go when renting for one more day. Formally SAA does not buy if for every $k \in \{1,\ldots,M-1\}$,
 \begin{equation}
 \label{eq:no_buy}
 b > 1 + \mathbb{P}_{\hat{\nu}}[ \envVr \geq k+1 \, \vert \, \envVr \geq k  ] \cdot b.
 \end{equation}
 Furthermore, when $E_\numS$ holds, we have that for every $k \in \{0,\ldots,M\}$,
 $$\mathbb{P}_{\nu_{\alpha}}\left(\envVr \geq k \right)(1-\delta) \leq \mathbb{P}_{\hat{\nu}}\left(\envVr \geq k \right) \leq \mathbb{P}_{\nu_{\alpha}}\left(\envVr \geq k \right)(1+\delta).$$
 Therefore,
 \begin{align*}
 \mathbb{P}_{\hat{\nu}}[ \envVr \geq k+1 \, \vert \, \envVr \geq k  ] &= \frac{ \mathbb{P}_{\hat{\nu}}[ \envVr \geq k+1] }{\mathbb{P}_{\hat{\nu}}[ \envVr \geq k]}\\
  &\leq \frac{1+\delta}{1-\delta} \cdot \frac{ \mathbb{P}_{\nu_\alpha}[ \envVr \geq k+1] }{\mathbb{P}_{\nu_\alpha}[ \envVr \geq k]}\\
  &= \frac{1+\delta}{1-\delta} \cdot \frac{ \mathbb{P}_{\nu}[ \envVr \geq k+1] - \alpha}{\mathbb{P}_{\nu}[ \envVr \geq k] - \alpha}
  < \frac{ \mathbb{P}_{\nu}[ \envVr \geq k+1] }{\mathbb{P}_{\nu}[ \envVr \geq k] } =  \mathbb{P}_{\nu}[ \envVr \geq k+1 \, \vert \, \envVr \geq k  ] ,
 \end{align*}
 where the last inequality holds by choosing $\delta$ small enough. Therefore, we conclude from \eqref{eq:recur}, that \eqref{eq:no_buy} holds.
 Finally we have that,
\begin{align*}
\mathbb{P}(B_\numS) \geq \mathbb{P}(E_\numS) \geq 1 - \sum_{k=1}^M\mathbb{P} \left( \hat{\nu}(k) \not \in [(1-\delta) \nu(k), (1+\delta) \nu(k)] \right) \geq 1- 2\sum_{k=1}^M \exp \left(- \frac{2 \delta^2 \cdot \nu(k)}{\numS} \right),
\end{align*}
where the first inequality follows by union bound and the second is a consequence of the multiplicative Hoeffding's inequality for bounded random variables.

This implies that $\mathbb{P}(B_\numS) \to 1$ as $\numS \to \infty$ and concludes the proof.

\textit{Step 2-a: Rate-optimal regret upper bound for the Kolmogorov distance.}  
We proceeds as follows.
We saw earlier in Step 1-a of our proof that SAA could incur a loss of $M\epsilon$ when it played $x=M$, i.e.\ it never bought the skis.
To assuage this, we consider a policy $\pi^C$ that follows SAA, except it buys at time $C$ no matter what, even if SAA suggests an $x>C$.
We show that the loss of $\pi^C$ is then $C\epsilon$ plus an error $\mathcal{G}(C,\nu)-\mathcal{G}(x,\nu)$, which is the error incurred when we truncated the buying time to $C$ even though the optimal buying time $x$ on distribution $\nu$ was greater than $C$.
We prove that this error term is at most $be^{-C/b}$, by arguing that if the optimal buying time on $\nu$ was later than $C$, then the probability of still skiing at time $C$ must be very small (otherwise, it would have been better to buy sooner).
Setting $C=b\log(1/\epsilon)$ to balance the terms, we achieve
\begin{equation*}
\inf_{\pi \in \Pi}  \mathfrak{R}_{K,\infty}^{\pi}(\epsilon) \leq b \cdot \left[ \log \left( \epsilon^{-1} \right) + 2 \right] \cdot \epsilon
\end{equation*}
which eliminates the dependence on $M$.\footnote{As mentioned earlier, \citet{diakonikolas2021learning} has already established $\mathcal{O} \left( b \cdot \epsilon \cdot \log \left( \epsilon^{-1} \right) \right)$ regret for the ski rental uniform DRO problem under Kolmogorov distance.  However, the proof of this result is still unavailable at the time of this paper, so we present our own proof for completeness and also because we believe it is conceptually elegant.}

Formally, for any $C > 0$, we consider the policy $\pi^{C}$ proposed by \cite{diakonikolas2021learning} which, given an empirical distribution $\hat{\nu}$, buys at time $\min \left(C, \oracle{\hat{\nu}} \right)$. We next show that, for every $\epsilon \geq 0$,
\begin{equation*}
\mathfrak{R}_{K,\infty}^{\pi^C}(\epsilon) \leq  b \cdot \left[ \log \left( \epsilon^{-1} \right) + 2 \right] \cdot \epsilon.
\end{equation*}
To do so, we first show that, 
\begin{equation}
\label{eq:DRO_ub_ski_SAA_W0}
\sup_{\mesOut \in \mesSp} \sup_{\nu \in \mathcal{U}_\epsilon(\mesOut)} | \opt(\mesOut) -  \g[\pi^C(\nu)]{\mesOut}| \leq b \cdot \left[ \log \left( \epsilon^{-1} \right) + 2 \right] \cdot \epsilon\end{equation}
and conclude by applying \Cref{thm:reduction_stat}.

Remark that for every $\nu \in \mesSp$ and $\nu \in \mathcal{U}_\epsilon(\mesOut)$, the regret satisfies,
\begin{align}
\g[\pi^C(\nu)]{\mesOut} - \opt(\mesOut) &= \g[\pi^C(\nu)]{\mesOut} - \opt(\nu) + \opt(\nu)  - \opt(\mesOut) \nonumber \\
&\leq \g[\pi^C(\nu)]{\mesOut} - \opt(\nu) + \g[\pi^C(\mesOut)]{\nu}  - \opt(\mesOut) \nonumber \\
&\leq 2 \sup_{\mesOut \in \mesSp} \sup_{\nu \in \mathcal{U}_\epsilon(\mesOut)}  \g[\pi^C(\nu)]{\mesOut} - \opt(\nu), \label{eq:cross_diff}
\end{align}
where the last inequality follows from the symmetry of the distance $d$. To bound the RHS of \eqref{eq:cross_diff} we remark that for every $\mesOut \in \mesSp$ and $\nu \in \mathcal{U}_\epsilon(\mesOut)$, we have that,
\begin{align*}
\g[\pi^C(\nu)]{\mesOut} - \opt(\nu) &= \g[\min(\oracle{\nu},C)]{\mesOut} - \g[\min(\oracle{\nu},C)]{\nu}\\
&\qquad + \g[\min(\oracle{\nu},C)]{\nu} - \opt(\nu)\\
& \stackrel{(a)}{\leq} (\min(\oracle{\nu},C) + b) \cdot \epsilon + \g[\min(\oracle{\nu},C)]{\nu} - \opt(\nu)\\
&\leq \left( C + b \right) \cdot \epsilon +  \g[\min(\oracle{\nu},C)]{\nu} - \opt(\nu),
\end{align*}
where $(a)$ follows form \Cref{lem:diff_cost_ski} (see \Cref{sec:aux_ski}). 

Let $F$ denote the cdf of $\nu$.
To bound the second difference we first remark that if $\oracle{\nu} \leq C$ then $\pi^C(\nu) = \oracle{\nu}$ and thus $\g[\min(\oracle{\nu},C)]{\nu} - \opt(\nu) = 0$. In contrast, if $\oracle{\nu} > C$, we have that $\pi^C(\nu) = C$ and,
\begin{align*}
\g[\min(\oracle{\nu},C)]{\nu} - \opt(\nu) &= \g[C]{\nu} - \g[\oracle{\nu}]{\nu}\\
&\stackrel{(a)}{=}  b \cdot \left( F(\oracle{\nu}) - F(C) \right) +   \int_{\oracle{\nu}}^C \left(1-F(\envVr)\right) d\envVr\\
&\stackrel{(b)}{\leq} b\cdot \left( 1-F(C) \right),
\end{align*}
where $(a)$ follows from \Cref{lem:ski_rental_cost} (see \Cref{sec:aux_ski}) and $(b)$ holds because  $\oracle{\nu} > C$ and $F(\oracle{\nu}) \leq 1$.

Our next step consists in bounding $1-F(C)$ as a function of $C$. Note that, for every $\actVr \in [0,C]$,
$\g[\actVr]{\nu} \geq \g[\oracle{\nu}]{\nu}$.
Therefore \Cref{lem:ski_rental_cost} implies that for every $\actVr \in [0,C]$,
\begin{equation*}
b \cdot (1- F(\actVr)) + \actVr - \int_{0}^\actVr F(\envVr) d \envVr \geq b \cdot (1- F(\oracle{\nu})) + \oracle{\nu} - \int_{0}^{\oracle{\nu}} F(\envVr) d \envVr 
\end{equation*}
and by reordering the terms we obtain that, 
\begin{equation*}
b \cdot (1- F(\actVr))  \geq b \cdot (1- F(\oracle{\nu})) +  \int_{x}^{\oracle{\nu}} \left( 1 - F(\envVr) \right) d \envVr
\end{equation*}
Finally, we denote by $\bar{F}:= 1 - F$ the complementary cdf associated to $F$ . We note that since $C < \oracle{\nu}$ and $F$ is non-decreasing and takes value in [0,1] we have that for ever $\actVr \in [0,C]$,
\begin{equation}
\label{eq:ODE}
b \cdot  \left(\bar{F}(\actVr) - \bar{F}(C) \right) \geq    \int_{x}^{C}  \bar{F}(\envVr)  d \envVr.
\end{equation}
Our next lemma bounds $\bar{F}(C)$ for all functions satisfying \eqref{eq:ODE}.
\begin{lemma}
\label{lem:ODE}
For any decreasing function $u$ which satisfies, u(0) =1 and for every $\actVr \in [0,C]$,
\begin{equation*}
b \cdot \left( u(\actVr) - u(C) \right)\geq \int_{x}^{C}  u(\envVr)  d \envVr,
\end{equation*}
we have that $u(C) \leq \exp(-\frac{C}{b})$.
\end{lemma}
By applying \Cref{lem:ODE} to $\bar{F}$, we obtain that $\bar{F}(C) \leq \exp(-\frac{C}{b})$. This implies that,
\begin{equation*}
\g[\min(\oracle{\nu},C)]{\nu} - \opt(\nu) \leq b \cdot \exp(-\frac{C}{b}),
\end{equation*}
and therefore,
\begin{equation*}
\g[\pi^C(\nu)]{\mesOut} - \opt(\nu) \leq (C+b) \cdot \epsilon + b \cdot \exp \left(-\frac{C}{b} \right).
\end{equation*}
Finally, by taking $C = b \cdot \log \left (\epsilon^{-1} \right)$, we obtain that
\begin{equation*}
\g[\pi^C(\nu)]{\mesOut} - \opt(\nu) \leq b \cdot \left[ \log \left( \epsilon^{-1} \right) + 2 \right] \cdot \epsilon,
\end{equation*}
which implies \eqref{eq:DRO_ub_ski_SAA_W0} and concludes the proof.

\textit{Step 2-b: Rate-optimal regret lower bound for the Kolmogorov distance.}  Fix $0 < \epsilon \leq \frac{1}{2}$. Assume $b \geq 1$.
Let $\tilde{\nu}$ and $\tilde{\mesOut}$  be the probability measures defined,
\begin{equation*}
\tilde{\nu}(\envVr) = \begin{cases}
\frac{1}{2} + \frac{\epsilon}{2} \qquad \text{if $\envVr = \frac{3}{4}b$}\\
\frac{1}{2} -  \frac{\epsilon}{2} \qquad \text{if $\envVr = \frac{5}{4}b$}
\end{cases} \qquad \mbox{and} \qquad \tilde{\mu}(\envVr) = \begin{cases}
\frac{1}{2} - \frac{\epsilon}{2} \qquad \text{if $\envVr = \frac{3}{4}b$}\\
\frac{1}{2} + \frac{\epsilon}{2} \qquad \text{if $\envVr = \frac{5}{4}b$}
\end{cases}
\end{equation*}
Note that for the Kolmogorov distance we have that $\tilde{\nu} \in \mathcal{U}_{\tilde{\mu}}(\epsilon)$. Therefore $\mathcal{U}_{\tilde{\mu}}(\epsilon) \cap \mathcal{U}_{\tilde{\nu}}(\epsilon)$ is non-empty and by \Cref{prop:lower_bound}, we have for every $\numS \geq 1$ that,
\begin{equation*}
 \inf_{\pi \in \Pi} \mathfrak{R}_{K,\numS}^{\pi}(\epsilon) \geq \sup_{\beta \in \Delta \left( \{ \tilde{\mesOut} , \tilde{\nu} \} \right) } \inf_{\actVr \in \actSp} \mathbb{E}_{\mesOut \sim \beta} \left[  \g[\actVr]{\mesOut} - \opt(\mesOut)  \right]. 
\end{equation*}

Furthermore, remark that,

\begin{equation*}
\g[0]{\tilde{\nu}} = b, \qquad  \g[\frac{3}{4}b]{\tilde{\nu}} = b + b \cdot \left(\frac{1}{4} - \frac{\epsilon}{2} \right), \qquad \g[\frac{5}{4}b]{\tilde{\nu}} =  b - \frac{\epsilon \cdot b}{4}, 
\end{equation*}
\begin{equation*}
\g[0]{\tilde{\mu}} = b,  \qquad \g[\frac{3}{4}b]{\tilde{\mu}} = b + b \cdot \left(\frac{1}{4} + \frac{\epsilon}{2} \right) \qquad \mbox{and} \qquad \g[\frac{5}{4}b]{\tilde{\mu}} =  b + \frac{\epsilon \cdot b}{4}. 
\end{equation*}
As a consequence, $\opt(\tilde{\nu}) = b - \frac{\epsilon \cdot b}{4}$ and $\opt(\tilde{\mu}) = b$. Let $\tilde{\beta}$ be the prior which puts mass $\frac{1}{2}$ on $\tilde{\nu}$ and mass $\frac{1}{2}$ on $\tilde{\mu}$. We have that,
\begin{equation*}
\mathbb{E}_{\mesOut \sim \tilde{\beta}} \left[\g[0]{\mesOut}   \right] = b,  \qquad
\mathbb{E}_{\mesOut \sim \tilde{\beta}} \left[\g[\frac{3}{4}b]{\mesOut}   \right] = \frac{9}{8}b \qquad \mbox{and} \qquad
\mathbb{E}_{\mesOut \sim \tilde{\beta}} \left[\g[\frac{5}{4}b]{\mesOut}   \right] = b. 
\end{equation*}
Therefore,
\begin{equation*}
    \inf_{\pi \in \Pi} \mathfrak{R}_{K,\numS}^{\pi}(\epsilon) \geq \ \inf_{\actVr \in \actSp} \mathbb{E}_{\mesOut \sim \tilde{\beta}} \left[ \g[\actVr]{\mesOut}  - \opt(\mesOut) \right]  = \frac{\epsilon \cdot b}{8}.
\end{equation*}
We conclude the proof of the lower bound by taking the limsup as $\numS$ to $\infty$.
\end{proof}

\begin{proof}[\textbf{Proof of \Cref{thm:ski_rental_W}}]
We first characterize the performance of SAA.

\textit{Step 1-a: Regret upper bound for SAA.} We derive an upper bound on the worst-case asymptotic regret of SAA by considering the uniform DRO problem. According to \Cref{thm:reduction_stat}, it is sufficient to show the following upper bound
\begin{equation}
\label{eq:DRO_ub_ski_SAA_W}
\sup_{\mesOut \in \mesSp} \sup_{\nu \in \mathcal{U}_\epsilon(\mesOut)} | \opt(\mesOut) -  \g[\SAA(\nu)]{\mesOut}| \leq 2 b +  2 \epsilon.
\end{equation}

For every $\mesOut, \nu \in \mesSp$, we have that,
\begin{align*}
  \g[\SAA(\nu)]{\mesOut} - \opt(\mesOut) &\stackrel{(a)}{=} \g[\oracle{\nu}]{\mesOut} - \opt(\mesOut)\\
  &=\left( \g[\oracle{\nu}]{\mesOut} -  \opt(\nu) \right) \\
  &\qquad +  \left( \opt(\nu) - \g[\oracle{\mesOut}]{\nu} \right) + \left( \g[\oracle{\mesOut}]{\nu}  - \opt(\mesOut) \right) \\
  &\stackrel{(b)}{\leq} \left( b +  d_{W} \left( \mesOut, \nu \right) \right) + 0 + \left( b +  d_{W} \left( \mesOut, \nu \right) \right) \\
  &= 2b +  2d_{W} \left( \mesOut, \nu \right),
\end{align*}
where $(a)$ follows by noting that $\SAA(\nu) = \oracle{\nu}$ by definition
and $(b)$ holds by applying \Cref{lem:diff_cost_ski} (see \Cref{sec:aux_ski}) to the first and third difference. By taking supremums over $\mesOut$ and $\nu$ we obtain the inequality \eqref{eq:DRO_ub_ski_SAA_W}. Furthermore, \Cref{thm:reduction_stat} and equation \eqref{eq:DRO_ub_ski_SAA_W} imply that,
\begin{equation*}
\mathfrak{R}_{W,\infty}^{\SAA}(\epsilon)  \leq  2b +  2\epsilon.
\end{equation*}

\textit{Step 1-b: Regret lower bound for SAA.}  Next, we derive a lower bound on the worst-case asymptotic regret of SAA.
Let $\mesOut$ and $\nu$  be the probability measures defined,
\begin{equation*}
\nu(\envVr) = \begin{cases}
\frac{3}{4} \qquad \text{if $\envVr = \frac{b}{2} $}\\
\frac{1}{4} \qquad \text{if $\envVr = M$}
\end{cases} \qquad \mbox{and} \qquad \mu(\envVr) = \begin{cases}
\frac{3}{4} \qquad \text{if $\envVr = \frac{b}{2} + \epsilon $}\\
\frac{1}{4} \qquad \text{if $\envVr = M$}.
\end{cases}
\end{equation*}
By construction, we have that $d_W\left(  \mesOut, \nu \right)= \frac{3}{4}\epsilon \leq \epsilon$.

Fix $\numS \geq 1$ and let $\envSamples$ be samples from $\nu$. Let $\hat{p}_\numS$ be the fraction of samples equal to $M$. Formally, 
$$\hat{p}_\numS = \frac{1}{\numS} \sum_{i = 1}^{\numS} \mathbbm{1} \left\{ \envVr_i = M \right\}.$$ 

Remark that $\hat{p}_\numS$ is a sufficient statistic to describe the empirical cdf $\hat{\nu}_{\envSamples}$. We next claim that $\SAA \left( \hat{\nu}_{\envSamples} \right) = \frac{b}{2}$ if and only if, $\hat{p}_\numS < \frac{1}{2}$. Without loss of optimality one may restrict attentions to one of the following actions: purchasing at day $0$, never purchasing (i.e. purchasing at the end of day $M$) or purchasing after a skiing day. 
Furthermore, by using the expression of the ski-rental cost, we remark that,
$\g[\frac{b}{2}]{\hat{\nu}_{\envSamples}} < \g[0]{\hat{\nu}_{\envSamples}}$ if and only if $\frac{b}{2} +  \hat{p}_\numS \cdot b < b 
$ and we have  $\g[\frac{b}{2}]{\hat{\nu}_{\envSamples}}  < \g[M]{\hat{\nu}_{\envSamples}} $ whenever $M > 2 \cdot b$ (independently of $\hat{p}_\numS$). Therefore, the action $\frac{b}{2}$ yields the smallest cost if and only if $\hat{p}_\numS< \frac{1}{2}$.

Let $E_\numS$ be the event, $\{ \hat	{p}_\numS < \frac{1}{2}\}$. We note that,
\begin{align}
\mathbb{P} \left( E_\numS \right) \geq \mathbb{P} \left( \left| \hat{p}_\numS - \frac{1}{4}\right|<\frac{1}{4}   \right) &\stackrel{(a)}{=} 1- \mathbb{P} \left( \left| \hat	{p}_\numS - \mathbb{E}[\hat{p}_\numS]\right| \geq \frac{1}{4}   \right) \nonumber \\
 &\stackrel{(b)}{\geq} 1 - 16 \text{Var}[\hat	{p}_\numS] = 1 - \frac{16}{\numS} \cdot \text{Var}[\mathbbm{1} \left\{ \envVr = M \right\}] = 1 - \frac{3}{\numS}  ,
\label{eq:bounding_proba}
\end{align}
where $(a)$ follows from the fact that $\mathbb{E}[\hat{p}_\numS] = \frac{1}{4}$ and $(b)$ is a consequence of Chevyshev's inequality.

We finally conclude that,
\begin{align*}
\mathfrak{R}_{W,\numS}^{\SAA}(\epsilon) &\stackrel{(a)}{\geq}  \mathcal{R}_\numS \left( \SAA, \mesOut, \nu, \ldots, \nu \right)\\
 &=\mathbb{E}_{\envVr_i \sim \nu} \left[ \g[\SAA(\hat{\nu}_{\envSamples})]{\mesOut} \right] - \opt(\mesOut)\\
 &\stackrel{(b)}{\geq} \left[ \g[\SAA(\hat{\nu}_{\envSamples})]{\mesOut} \, \Big \vert \, E_{\numS} \right] \cdot \mathbb{P} \left(E_\numS \right)   - \g[0]{\mu}\\ 
 &\stackrel{(c)}{=}  \mathbb{P} \left(E_\numS \right)  \g[\frac{b}{2}]{\mesOut} - \g[0]{\mu}= \frac{3}{2} \mathbb{P} \left(E_\numS \right) \cdot b - b, 
\end{align*}
where $(a)$ follows from the fact that $d_W(\mu,\nu) \leq \epsilon$ and $(b)$ holds because the function $g$ is non-negative and because $\opt(\mesOut) \leq \g[0]{\mesOut}$. $(c)$ holds because under the event $E_\numS$, $\SAA(\hat{\nu}_{\envSamples}) = \frac{b}{2}$.

Finally, by taking the limsup we obtain that,
\begin{equation*}
\limsup_{\numS \to \infty} \mathfrak{R}_{W,\numS}^{\SAA}(\epsilon) \geq \limsup_{\numS \to \infty} \frac{3}{2} \mathbb{P} \left(E_\numS \right) \cdot b - b \stackrel{(a)}{=} \frac{b}{2},
\end{equation*}
where $(a)$ follows from \eqref{eq:bounding_proba}.

We now characterize the rate of the optimal achievable performance.

\textit{Step 2-a: Rate-optimal regret upper bound.} Fix $\delta < 0$. We first show the following upper bound on the asymptotic worst-case regret of $\dev{\delta}$.
\begin{equation}
\label{eq:ski_ub_general_delta}
\mathfrak{R}_{W,\infty}^{\dev{\delta}}(\epsilon)  \leq b \cdot \frac{\epsilon}{|\delta|} +  2 \epsilon +  |\delta| +  2 \sqrt{b \cdot \epsilon}.
\end{equation}
To prove this result, we first bound the quantity 
$\sup_{\mesOut \in \mesSp} \sup_{\nu \in \mathcal{U}_\epsilon(\mesOut)} | \opt(\mesOut) -  \g[\dev{\delta}(\nu)]{\mesOut}|$
and then apply \Cref{thm:reduction_stat}.

Consider $\mesOut, \nu \in \mesSp$. We have that,
\begin{align}
\g[\dev{\delta}(\nu)]{\mesOut} - \opt(\mesOut)  &= \g[\dev{\delta}(\nu)]{\mesOut} - \opt(\nu) + \opt(\nu) -  \opt(\mesOut)   \nonumber \\
&\stackrel{(a)}{\leq}  b \cdot \frac{d_W \left(\mesOut,\nu \right)}{|\delta|} +  2 d_W \left(\mesOut,\nu \right) +  |\delta| +  2 \sqrt{b \cdot d_{W}\left(\mesOut,\nu \right)}, \label{eq:dev_SAA_ski}
\end{align}
where $(a)$ follows from the following lemma.
\begin{lemma}
\label{lem:compare_rev_ski}
Consider the ski-rental problem.  Fix $\delta <  0$. For every $\mesOut$ and $\nu \in \mesSp$, we have that,
\begin{align*}
\opt \left(\mesOut\right) - \opt \left( \nu \right) &\leq2 \sqrt{b \cdot d_{W}\left(\mesOut,\nu \right)} + d_{W}\left(\mesOut,\nu \right),\\
 \g[\dev{\delta}(\nu)]{\mesOut} - \opt(\nu) &\leq b \cdot \frac{d_W \left(\mesOut,\nu \right)}{|\delta|} +  d_W \left(\mesOut,\nu \right) +  |\delta|.
\end{align*}
\end{lemma}

By taking the supremum in equation \eqref{eq:dev_SAA_ski}, we obtain,
\begin{equation}
\label{eq:DRO_ub_ski_W}
\sup_{\mesOut \in \mesSp} \sup_{\nu \in \mathcal{U}_\epsilon(\mesOut)} | \opt(\mesOut) -  \g[\dev{\delta}(\nu)]{\mesOut}| \leq b \cdot \frac{\epsilon}{|\delta|} +  2 \epsilon +  |\delta| +  2 \sqrt{b \cdot \epsilon}.
\end{equation}
We conclude that \eqref{eq:ski_ub_general_delta} holds by applying \Cref{thm:reduction_stat} and by using the upper bound derived in \eqref{eq:DRO_ub_ski_W} on the uniform DRO problem.

Furthermore, we remark that $\delta = - \sqrt{b \cdot \epsilon}$ minimizes the RHS of \eqref{eq:ski_ub_general_delta} and yields the bound,
\begin{equation*}
\mathfrak{R}_{W,\infty}^{\dev{\delta}}(\epsilon)  \leq 4 \sqrt{b \cdot \epsilon} + 2 \epsilon.
\end{equation*}
This concludes the proof of the upper bound on the asymptotic worst-case regret of the optimal policy.

\textit{Step 2-b: Rate-optimal regret lower bound.} We next provide a universal lower bound across all possible data-driven policies. Fix $0 < \epsilon \leq \frac{1}{2}$. Assume $b \geq 1$.
Let $\tilde{\nu}$ and $\tilde{\mesOut}$  be the probability measures defined,
\begin{equation*}
\tilde{\nu}(\envVr) = \begin{cases}
\frac{1}{2}  \qquad \text{if $\envVr = \frac{b}{2}  - \frac{\sqrt{b \cdot \epsilon}}{2}$}\\
\frac{1}{2}  \qquad \text{if $\envVr = M$}
\end{cases} \qquad \mbox{and} \qquad \tilde{\mu}(\envVr) = \begin{cases}
\frac{1}{2} - \sqrt{\frac{\epsilon}{b}} \qquad \text{if $\envVr = \frac{b}{2}  - \frac{\sqrt{b \cdot \epsilon}}{2}$}\\
\sqrt{\frac{\epsilon}{b}} \qquad \text{if $\envVr = \frac{b}{2}  + \frac{\sqrt{b \cdot \epsilon}}{2}$}\\
\frac{1}{2}  \qquad \text{if $\envVr = M$}.
\end{cases}
\end{equation*}
Note that for the Wasserstein distance we have that $\tilde{\nu} \in \mathcal{U}_{\tilde{\mu}}(\epsilon)$. Therefore $\mathcal{U}_{\tilde{\mu}}(\epsilon) \cap \mathcal{U}_{\tilde{\nu}}(\epsilon)$ is non-empty and by \Cref{prop:lower_bound}, we have for every $\numS \geq 1$ that,
\begin{equation*}
 \inf_{\pi \in \Pi} \mathfrak{R}_{W,\numS}^{\pi}(\epsilon) \geq \sup_{\beta \in \Delta \left( \{ \tilde{\mesOut} , \tilde{\nu} \} \right) } \inf_{\actVr \in \actSp} \mathbb{E}_{\mesOut \sim \beta} \left[  \g[\actVr]{\mesOut} - \opt(\mesOut)  \right]. 
\end{equation*}

Without loss of optimality one may restrict attentions to one of the following actions: purchasing at day $0$ or purchasing after a skiing day. We remark that,
\begin{equation*}
\g[0]{\tilde{\nu}} = b, \quad  \g[\frac{b}{2}  - \frac{\sqrt{b \cdot \epsilon}}{2}]{\tilde{\nu}} = b - \frac{\sqrt{b \cdot \epsilon}}{2},  \quad \g[\frac{b}{2}  + \frac{\sqrt{b \cdot \epsilon}}{2}]{\tilde{\nu}} =  b,
\end{equation*}
\begin{equation*}
\g[0]{\tilde{\mu}} = b, \quad \g[\frac{b}{2}  - \frac{\sqrt{b \cdot \epsilon}}{2}]{\tilde{\mu}} = b + \frac{\sqrt{b \cdot \epsilon}}{2} \quad \mbox{and} \quad \g[\frac{b}{2}  + \frac{\sqrt{b \cdot \epsilon}}{2}]{\tilde{\mu}} =  b + \epsilon. 
\end{equation*}
As a consequence, $opt(\tilde{\nu}) = b - \frac{ \sqrt{b \cdot \epsilon} }{2}$ and $\opt(\tilde{\mu}) = b$. Let $\tilde{\beta}$ be the prior which puts mass $\frac{1}{2}$ on $\tilde{\nu}$ and mass $\frac{1}{2}$ on $\tilde{\mu}$. We have that,
\begin{equation*}
\mathbb{E}_{\mesOut \sim \tilde{\beta}} \left[\g[0]{\mesOut}   \right] = b, \quad
\mathbb{E}_{\mesOut \sim \tilde{\beta}} \left[\g[\frac{b}{2}  - \frac{\sqrt{b \cdot \epsilon}}{2}]{\mesOut}   \right] = b \: \mbox{ and } \:
\mathbb{E}_{\mesOut \sim \tilde{\beta}} \left[\g[\frac{b}{2}  + \frac{\sqrt{b \cdot \epsilon}}{2}]{\mesOut}   \right] = b + \frac{\epsilon}{2}. 
\end{equation*}
Therefore,
\begin{equation*}
    \inf_{\pi \in \Pi} \mathfrak{R}_{W,\numS}^{\pi}(\epsilon) \geq \ \inf_{\actVr \in \actSp} \mathbb{E}_{\mesOut \sim \tilde{\beta}} \left[ \g[\actVr]{\mesOut}  - \opt(\mesOut) \right]  = \frac{ \sqrt{b \cdot \epsilon} }{4}.
\end{equation*}
We conclude the proof of the lower bound by taking the limsup as $\numS$ goes to $\infty$.
\end{proof}

\subsubsection{Proofs of auxiliary results}
\label{sec:aux_ski}

\begin{proof}[\textbf{Proof of \Cref{lem:ODE}}]
We prove the statement by first showing by induction that for every $k \geq 0$ and for every $\actVr \in [0,C]$,
\begin{equation}
\label{eq:induction}
u(\actVr) \geq u(c) \cdot \sum_{i=0}^k \frac{(C-x)^i}{i! \cdot b^i}.
\end{equation}
Note that for $k = 0$, \eqref{eq:induction} holds because $u$ is non-decreasing.
Fix $k \geq 0$. We now assume that the \eqref{eq:induction}  holds for $k$ and we next prove it for $k+1$.
Note that, for every $\actVr \in [0,C]$.
\begin{align*}
b \cdot u(\actVr) &\stackrel{(a)}{\geq} \int_{\actVr}^C u(\envVr) d\envVr + b \cdot u(C)\\
&\stackrel{(b)}{\geq} u(c) \cdot \int_{\actVr}^C  \sum_{i=0}^k \frac{(C-\envVr)^i}{i! \cdot b^i} d\envVr + b \cdot u(C)\\
&= u(c) \cdot  \sum_{i=0}^k \frac{(C-\actVr)^{i+1}}{(i+1)! \cdot b^i}  + b \cdot u(C) = u(C) \cdot b \cdot \sum_{i=0}^{k+1} \frac{(C-\actVr)^i}{i! \cdot b^i} 
\end{align*}
where $(a)$ holds by assumption assumption on the inequality satisfied by $u$ and $(b)$ follows from the induction hypothesis.
Finally, by dividing by $b$ on both sides, we obtain that \eqref{eq:induction} holds for $k+1$. Therefore, by induction \eqref{eq:induction} holds for every $k \geq 0$.

Next, by letting $\actVr = 0$ and taking the limit as $k$ goes to $\infty$, we obtain that,
\begin{equation*}
u(0) \geq u(C) \cdot \sum_{i=0}^{\infty} \frac{C^i}{i! \cdot b^i}  = u(C) \cdot \exp \left( \frac{C}{b} \right).
\end{equation*}
Finally, we remark that $u(0) = 1$. Therefore, $u(C) \leq \exp \left( -\frac{C}{b} \right)$.

\end{proof}

\begin{proof}[\textbf{Proof of \Cref{lem:compare_rev_ski}}]
\textit{Step 1:} We first prove that, $\opt \left(\mesOut\right) - \opt \left( \nu \right) \leq 2 \sqrt{b \cdot d_{W}\left(\mesOut,\nu \right)} + d_{W}\left(\mesOut,\nu \right)$.

First, if $\oracle{\mesOut} \geq M - \sqrt{b \cdot d_{W}\left(\mesOut,\nu \right)}$, we have that
\begin{align*}
\opt \left(\nu \right) - \opt \left( \mesOut \right) &= \opt \left(\nu \right) - \g[M]{\mesOut} + \g[M]{\mesOut} - \g[\oracle{\mesOut}]{\mesOut} \\
&\stackrel{(a)}{\leq} \opt \left(\nu \right) - \g[M]{\mesOut} + M - \oracle{\mesOut}\\
&\leq \g[M]{\nu} - \g[M]{\mesOut} + M - \oracle{\mesOut}\\
&\stackrel{(b)}{=} \int_0^M \left( F(\envVr) - H(\envVr) \right) d\envVr + M - \oracle{\mesOut} \leq d_W\left(\mesOut, \nu \right) +  \sqrt{b \cdot d_{W}\left(\mesOut,\nu \right)},
\end{align*}
where $(a)$ follows from \Cref{lem:ski_rental_partialy_lip} and $(b)$ follows from \Cref{lem:ski_rental_cost}.

Assume now that, $\oracle{\mesOut} < M - \sqrt{b \cdot d_{W}\left(\mesOut,\nu \right)}$. Let $\tilde{\actVr} := \oracle{\mesOut} + \sqrt{b \cdot d_{W}\left(\mesOut,\nu \right)}$.
We have,
\begin{align*}
\opt \left( \nu \right) - \opt \left(\mesOut\right)   &\leq \g[\tilde{\actVr}]{\nu} - \g[\oracle{\mesOut}]{\mesOut} \\
&\stackrel{(a)}{\leq}  \tilde{\actVr} - \oracle{\mesOut}  + b \cdot \frac{d_W \left(\mesOut,\nu\right)}{\tilde{\actVr} - \oracle{\mesOut} } + d_W \left(\mesOut,\nu\right) \\
&= 2 \sqrt{b \cdot d_{W}\left(\mesOut,\nu \right)} + d_{W}\left(\mesOut,\nu \right) ,
\end{align*}
where $(a)$ follows from \Cref{prop:diff_rev_ski}.

\textit{Step 2:} We now show that, $\g[\dev{\delta}(\nu)]{\mesOut}  - \opt(\nu) \leq  b \cdot \frac{d_W \left(\mesOut,\nu \right)}{|\delta|} +  d_W \left(\mesOut,\nu \right) +  |\delta|.$

We note that if, $\oracle{\nu} \geq M - |\delta|$, we have that $\dev{\delta}(\nu)$ plays action $M$ and,  
\begin{align*}
\g[\dev{\delta}(\nu)]{\mesOut} - \opt(\nu) &= \g[M]{\mesOut} - \opt(\nu)\\
&= \g[M]{\mesOut} - \g[M]{\nu} + \g[M]{\nu} - \g[\oracle{\nu}]{\nu}\\
&\stackrel{(a)}{\leq} d_W \left(\mesOut,\nu \right) +  |\delta|,
\end{align*}
where $(a)$ is obtained by a similar argument to the one developed in step 1.
Furthermore, if $\oracle{\nu} \leq M - |\delta|$, then
\begin{align*}
\g[\dev{\delta}(\nu)]{\mesOut} - \opt(\nu) &= \g[\oracle{\nu} + |\delta|]{\mesOut} -  \g[\oracle{\nu}]{\nu} \\
& \stackrel{(a)}{\leq} b \cdot \frac{d_W \left(\mesOut,\nu \right)}{|\delta|} +  d_W \left(\mesOut,\nu \right) +  |\delta|.
\end{align*}
where $(a)$ follows from \Cref{prop:diff_rev_ski}.

\end{proof}

\begin{lemma}\label{lem:ski_rental_cost}
Consider the ski-rental cost. For every $\mu \in \mesSp$ associated with the cumulative distribution $F$, we have for every $\actVr \in \actSp$ that,
\begin{equation*}
    \g[\actVr]{\mesOut} =  b \cdot \left( 1 - F(\actVr) \right)  + \actVr -  \int_{0}^{\actVr} F(\envVr) d\envVr.
\end{equation*}
\end{lemma}
\begin{proof}[\textbf{Proof of \Cref{lem:ski_rental_cost}}]
For every $\mesOut \in \mesSp$ and $\actVr \in \actSp$, we have that,
\begin{align*}
        \g[\actVr]{\mesOut} &= \mathbb{E}_{\envVr \sim \mesOut }\left[ \envVr \cdot \mathbbm{1} \left\{ \envVr \leq \actVr  \right\} \right] + \left( b + x \right) \cdot \left(1 - F(x) \right)\\
        &\stackrel{(a)}{=}  \left( \actVr \cdot F(\actVr) - \int_{0}^x F(\envVr) d\envVr \right) +  \left( b+  x \right) \cdot \left(1 - F(x) \right)\\
        &=   b \cdot \left( 1 - F(\actVr) \right)  + \actVr -  \int_{0}^{\actVr} F(\envVr) d\envVr,
\end{align*}
where $(a)$ follows from the Riemman-Stieljes integration by part.
\end{proof}

\begin{lemma}\label{lem:ski_rental_partialy_lip}
Consider the ski-rental cost. For every $\mu \in \mesSp$ and every $\actVr_1, \actVr_2  \in \actSp$ such that $\actVr_1 < \actVr_2$ we have that,
\begin{equation*}
\g[\actVr_2]{\mesOut} - \g[\actVr_1]{\mesOut} \leq \actVr_2 - \actVr_1.
\end{equation*}
\end{lemma}

\begin{proof}[\textbf{Proof of \Cref{lem:ski_rental_partialy_lip}}]
Let $\mu \in \mesSp$ and let $\actVr_1, \actVr_2  \in \actSp$ such that $\actVr_1 \leq \actVr_2$. We have,
\begin{equation*}
\g[\actVr_2]{\mesOut} - \g[\actVr_1]{\mesOut} \stackrel{(a)}{=} b \cdot \left( F(\actVr_1) - F(\actVr_2) \right) + \actVr_2 - \actVr_1 - \int_{\actVr_1}^{\actVr_2 } F(\envVr) d\envVr \stackrel{(b)}{\leq} \actVr_2 - \actVr_1,
\end{equation*}
where $(a)$ follows from \Cref{lem:ski_rental_cost} and $(b)$ holds because $F$ is non-negative and non-decreasing.
\end{proof}

\begin{lemma}\label{lem:diff_cost_ski}
For the ski-rental problem, we have that for every $\mesOut,\nu \in \mesSp$ and any $\actVr \in \actSp$,
\begin{equation*}
    |\g[\actVr]{\mesOut} - \g[\actVr]{\nu} | \leq \left(  \actVr + b \right)  \cdot d_{K} \left( \mesOut, \nu \right) \leq \left(  M + b \right)  \cdot d_{K} \left( \mesOut, \nu \right)  \qquad \mbox{and} \qquad |\g[\actVr]{\mesOut} - \g[\actVr]{\nu} | \leq b +  d_{W} \left( \mesOut, \nu \right)
\end{equation*}
\end{lemma}

\begin{proof}[\textbf{Proof of \Cref{lem:diff_cost_ski}}]
For every $\mesOut, \nu \in \mesSp$, such that $\mesOut$ (resp. $\nu$) is associated to the cumulative distribution $F$ (resp. $H$) we have that,
\begin{align*}
     |\g[\actVr]{\mesOut} - \g[\actVr]{\nu} | &\stackrel{(a)}{=} \left |b \cdot \left( H(\actVr) - F(\actVr) \right)  +  \int_{0}^{\actVr} \left( H(\envVr) - F(\envVr) \right) d\envVr \right| \\
     &\leq b \cdot |   H(\actVr) - F(\actVr) |   +  \int_{0}^{\actVr} | H(\envVr) - F(\envVr) | d\envVr \\
  &\leq b \cdot    d_{K} \left( \mesOut, \nu \right)  + \int_{0}^x d_{K} \left( \mesOut, \nu \right) d\envVr
 = \left(  x + b \right)  \cdot d_{K} \left( \mesOut, \nu \right) \leq \left(  M + b \right)  \cdot d_{K} \left( \mesOut, \nu \right)
\end{align*}
and
\begin{align*}
     |\g[\actVr]{\mesOut} - \g[\actVr]{\nu} | &\stackrel{(a)}{=} \left |b \cdot \left( H(\actVr) - F(\actVr) \right)  +  \int_{0}^{\actVr} \left( H(\envVr) - F(\envVr) \right) d\envVr \right|\\
      &\leq b \cdot |   H(\actVr) - F(\actVr) |   +  \int_{0}^{\actVr} | H(\envVr) - F(\envVr) | d\envVr \stackrel{(b)}{\leq} b +  d_{W} \left( \mesOut, \nu \right), 
\end{align*}
where $(a)$ follows from \Cref{lem:ski_rental_cost} and $(b)$ holds because $d_{W} \left( \mesOut, \nu \right) = \int_{\envVr \in \envSp} | H(\envVr) - F(\envVr) | d\envVr $. 
\end{proof}

\begin{proposition}
\label{prop:diff_rev_ski}
Consider the ski-rental problem.
Let $\mesOut$ and $\nu\in \mesSp$ be two different distributions and let $\actVr_1$ and $\actVr_2$ be two different actions such that $\actVr_1 < \actVr_2$. We have that,
\begin{equation*}
\g[\actVr_2] {\mesOut} - \g[\actVr_1] {\nu}   \leq b \cdot \frac{d_W \left(\mesOut,\nu \right)}{\actVr_2 - \actVr_1} +  d_W \left(\mesOut,\nu \right) +  \actVr_2 - \actVr_1.
\end{equation*} 
\end{proposition}
\begin{proof}[\textbf{Proof of \Cref{prop:diff_rev_ski}}]
Let $\actVr_1 < \actVr_2$. 
Let $F$ (resp. $H$) be the cdf associated to $\mesOut$ (resp. $\nu$). We have,
\begin{align*}
    \g[\actVr_2]{\mesOut} - \g[\actVr_1]{\nu} &\stackrel{(a)}{=}  b \cdot \left( H\left( \actVr_1 \right) -  F\left( \actVr_2 \right) \right) + \actVr_2 - \actVr_1 + \int_{0}^{\actVr_1} \left( F(\envVr) -H (\envVr) \right) d \envVr - \int_{\actVr_1}^{\actVr_2} F(\envVr) d \envVr \\
    &\leq b \cdot  \left( H\left( \actVr_1 \right) - F \left( \actVr_2 \right) \right) +  \actVr_2 - \actVr_1  + \int_{0}^{M}| F(\envVr) - H (\envVr) |d \envVr \\
    &\stackrel{(b)}{\leq} b \cdot \frac{d_W \left(\mesOut,\nu \right)}{\actVr_2 - \actVr_1} +  \actVr_2 - \actVr_1  + \int_{0}^{M}| F(\envVr) - H (\envVr) |d \envVr\\
    &\stackrel{(c)}{=}b \cdot \frac{d_W \left(\mesOut,\nu \right)}{\actVr_2 - \actVr_1} +  \actVr_2 - \actVr_1  + d_W \left(\mesOut,\nu \right),
\end{align*}
where $(a)$ follows from \Cref{lem:ski_rental_cost} and $(b)$ holds by \Cref{lem:cdf_vs_dist} and $(c)$ follows from the definition of the Wasserstein distance.  
\end{proof}

\section{Multi-dimensional example: Bayesian Mechanism Design}
\label{sec:apx_auction}
In this section, we illustrate how our framework can be applied in a multi-dimensional setting by developing guarantees for the bayesian mechanism design problem. In \Cref{sec:apx_formulation_auction} we present the Bayesian Mechanism Design problem in heterogeneous environment and provide a way to relate the heterogeneity structure analyzed in \cite{brustle2020multi} to an IPM. In \Cref{sec:apx_SAA_auction}, we leverage the methodology developed in our paper to  bound the asymptotic worst-case regret of SAA for this problem class. This shows that our methodology allows to recover results derived for specific multi-dimensional problems analyzed in the literature.

\subsection{Problem formulation}\label{sec:apx_formulation_auction}

In this problem, a decision-maker sells a single indivisible good to $m$ buyers. The environment space $\envSp = [0,M]^m$ represents the set of vectors $\bm{\envVr}$, such that for every $i \in \{1,\ldots,m\}$, $\envVr_i$ is the private value that the $i^{th}$ buyer associates to the product. We assume that this value is sampled independently from a distribution $\mu_i$. Therefore the vector of values $\bm{\envVr}$ is sampled from the product distribution $\bm{\mu} = \mu_1 \times \ldots \times \mu_m$. The set of probability measure $\mathcal{P}$ considered  in this problem is the set of product measure $\Delta([0,M])^m$.

A mechanism $(\mathbf{x},\mathbf{p})$ is  defined by an allocation function $\mathbf{x}: \envSp \rightarrow [0,1]^m$, such that for every vector of values $\bm{\envVr} \in \envSp$, $\mathbf{x}(\bm{\envVr})$ is the vector of probabilities that the item is allocated to each buyer, and a payment function $p: \envSp \rightarrow \mathbb{R}_{+}^m $, where $\mathbf{p}(\bm{\envVr})$ represents the vector of payments of each buyer.

In what follows, we will restrict attention to mechanisms satisfying the following two conditions for every $i \in \{1,\ldots m\}$:
\begin{align}
\envVr_i \cdot x_i(\envVr_i, \bm{\envVr}_{-i}) - p_i(\envVr_i, \bm{\envVr}_{-i})  \geq \envVr_i \cdot x_i(b_i, \bm{\envVr}_{-i}) - p_i(b_i, \bm{\envVr}_{-i}) 
 \quad &\text{for every $\envVr_i,b_i \in [0,M]$ and $\bm{\envVr}_{-i} \in [0,M]^{m-1}$}  \tag{DSIC} \label{eq:DSIC} \\ 
\envVr_i \cdot x_i(\envVr_i, \bm{\envVr}_{-i}) - p_i(\envVr_i, \bm{\envVr}_{-i}) \geq 0 \quad &\text{for every $\envVr_i \in [0,M]$ and $\bm{\envVr}_{-i} \in [0,M]^{m-1}$} \tag{IR} \label{eq:IR},
\end{align}
where $\bm{\envVr_{-i}}$ is the bid vector of all buyers except buyer $i$.
The first constraint is the \textit{Dominant Strategy Incentive Compatibility} and ensures that, for the given mechanism, buyers are always better-off when bidding their own value, i.e. $b_i = \envVr_i$ for every $i$. The second constraint ensures \textit{Individual Rationality}. In what follows, $\actSp$ will be the set of \textit{randomized} mechanisms satisfy \eqref{eq:DSIC} and \eqref{eq:IR}.

For any truthful mechanism $(\mathbf{x},\mathbf{p}) \in \actSp$, the total revenue is defined for every $\bm{\envVr} \in \envSp$ as,
\begin{equation*}
g((\mathbf{x},\mathbf{p}),\bm{\envVr}) =  \sum_{i=1}^m p_i(\bm{\envVr})  .
\end{equation*}
The decision-maker aims at selecting a mechanism satisfying \eqref{eq:DSIC} and \eqref{eq:IR} in order to maximize the expected total revenue $\g[(\mathbf{x},\mathbf{p})]{\bm{\mu}} = \mathbb{E}_{\bm{\envVr} \sim \bm{\mu}} \left[ \sum_{i=1}^m p_i(\bm{\envVr}) \right].$ Equivalently, the decision-maker aims at minimizing the regret against the oracle defined as,
\begin{equation*}
\mathtt{ORACLE}: \bm{\mu} \mapsto (\bm{x}^{\text{Myerson}},\bm{p}^{\text{Myerson}}),
\end{equation*}
where $(\bm{x}^{\text{Myerson}},\bm{p}^{\text{Myerson}})$ is Myerson's optimal mechanism \citep{myerson1981optimal}.

\noindent \textbf{Multi-dimensional Kolmogorov distance.}  We illustrate the applicability of our framework to higher dimensions by considering the heterogeneity structure previously used in \cite{brustle2020multi}. In particular, they evaluate the distance between $\bm{\mu} \in \mesSp$ and $\bm{\nu} \in \mesSp$, by comparing the distance between their marginals. Formally, for a fixed epsilon $\epsilon > 0$, $\bm{\mesOut}$ is at most $\epsilon$ away from $\bm{\nu}$ if and only if for every $i \in \{1,\ldots,\numS\}$,
\begin{equation*}
d_K(\mu_i,\nu_i) \leq \epsilon. 
\end{equation*}
where $d_K$ is the one-dimensional Kolmogorov distance defined in \Cref{sec:IPM}. We note that there is no simple characterization of this definition of heterogeneity as an IPM. In what follows, we will leverage the independence assumption of buyers values to relate this heterogeneity notion defined through marginals to an IPM which is defined for the whole product distribution. 

Formally, let $\mathcal{F}$ be a generating class of functions from $[0,M] \to \mathbb{R}$ and define $\mathcal{F}^{(m)}$,
\begin{equation*}
\mathcal{F}^{(m)} = \bigcap_{i=1}^m \left \{ f : [0,M]^m \to \mathbb{R} \text{ s.t. for every } \bm{\envVr_{-i}} \in [0,M]^{m-1}\,, \, \envVr_i \mapsto f(\envVr_i,\envVr_{-i}) \in \gen{\mathcal{F}} \right\}.
\end{equation*}
Intuitively, the class $\mathcal{F}^{(m)}$ is the class of functions defined on $[0,M]^{m}$ such that every single dimensional functions which induced by fixing $m-1$ parameters to arbitrary values is in the original class $\mathcal{F}$.

Our next results shows that for product of measure, we can relate the IPM generated by $\mathcal{F}^{(m)}$ to the one generated by $\mathcal{F}$.
\begin{proposition}
\label{prop:multiD_to_1D}
For every $\bm{\mu}, \bm{\nu} \in \Delta([0,M])^m$,
\begin{equation*}
d_{\mathcal{F}^{(m)}}(\bm{\mu},\bm{\nu}) \leq \sum_{i=1}^m d_{\mathcal{F}} (\mu_i,\nu_i).
\end{equation*}
\end{proposition}
We emphasize here that the assumption that the probability measures are product measures is critical.

Recall that, $\mathcal{K}$ is the generating set defining the Kolmogorov distance. We consider the IPM generated by $\mathcal{K}^{(m)}$ and 
denote by $\mathcal{I}_{mK}$ the problem instance for which the heterogeneity set is defined by $\dis{\mathcal{K}^{(m)}}$.
In what follows, we will analyze the performance of SAA for $\mathcal{I}_{mK}$. We will then relate our result to the one derived by \cite{brustle2020multi} by using \Cref{prop:multiD_to_1D}.

\subsection{Regret of SAA for Bayesian Mechanism Design with Kolmogorov heterogeneity}\label{sec:apx_SAA_auction}

In this section we analyze the SAA policy which takes as an input past samples $(\bm{\envVr}_1,\ldots,\bm{\envVr}_\numS) \in \envSp^{\numS}$ and constructs for each $i \in \{1,\ldots, m\}$ the empirical distribution of values of buyer $i$ defined for every measurable set $A \subset \envSp$ as $\hat{\nu_i}(A) = \sum_{j=1}^\numS \mathbbm{1} \{ \envVr_{j,i} \in A \}$. The SAA policy then outputs the Myerson mechanism associated to the product measure $\hat{\bm{\nu}}  = \hat{\nu}_1 \times \ldots \times \hat{\nu}_m$.

We next apply the methodology developed in \Cref{sec:SAA-analysis} to compute the asymptotic worst-case regret of SAA under the Kolmogorov heterogeneity. We first show that one can relate in this setting the asymptotic worst-case regret of SAA to the uniform deviation $\UB_{\mathcal{I}}$ as in \eqref{eq:bound_mohri}. We then control $\UB_{\mathcal{I}}$ by analyzing the approximation parameter $\lambda \left( \mathcal{K}^{(m)}, g, 0 \right)$ and leveraging \Cref{thm:generating_class}.

We first show that the ETC property holds for the distance $\dis{\mathcal{K}^{(m)}}$ for every product of probability measures. To that end, we prove the following more general statement.
\begin{proposition}[Lifting ETC from marginals to product measures]
\label{prop:ETC_multiK}
Let $\mathcal{F}$ be a generating class of functions from $[0,M] \to \mathbb{R}$ and consider the associated multi-dimensional IPM $\mathcal{F}^{(m)}$. Assume that $\mathcal{F}$ satisfies the ETC property. Then for every triangular array sequence of product of probability measures $(\bm{\mu}_{i,\numS})_{1 \leq i \leq \numS, \numS \in \mathbb{N}^*}$ all belonging to $\Delta([0,M])^m$, we have
\begin{equation*}
\lim_{\numS \to \infty} \dis{\mathcal{F}^{(m)}}(\hat{\bm{\mu}}_\numS, \bar{\bm{\mu}}_\numS) = 0 \qquad \text{a.s.},
\end{equation*}
where $\hat{\bm{\mu}}_\numS = \hat{\nu}_{1,\numS} \times \ldots \times \hat{\nu}_{k,\numS} $ is the empirical product probability measure with marginal defined for every measurable set $A \in [0,M]$, and every $j \in \{1,\ldots,k\}$ as $\hat{\nu}_{j,\numS}(A) :=  \frac{1}{\numS} \sum_{i=1}^\numS \mathbbm{1} \left \{ \envVr^{(i,\numS)}_{j} \in A \right \}$ 
for samples $\bm{\envVr}^{(i,\numS)} \sim \bm{\mu}_{i,\numS}$. Furthermore, $\bar{\mu}_\numS = \frac{1}{\numS} \sum_{i=1}^\numS \mu_{i,\numS}$. 
\end{proposition}
Note that \Cref{prop:ETC_multiK} can be applied to $\dis{\mathcal{K}^{(m)}}$ given that the Kolmogorov distance satisfies the ETC property (see \Cref{prop:etc_Kol}).
Furthermore, we highlight that \eqref{prop:ETC_multiK} only ensures a form of ETC on product measures (as opposed to space of all probability measures). The proof of \Cref{thm:reduction_stat} goes through in this case for the implementation of SAA described above as the latter policy also uses the product empirical measure to make decisions.

Let $\mathcal{X}^{\mathrm{Myerson}}$ be the subset of mechanisms which are Myerson 
mechanisms for some choice of product measures. Formally,
\begin{equation*}
\mathcal{X}^{\mathrm{Myerson}} = \left \{ (\bm{x},\bm{p}) \; \text{s.t. there exists $\bm{\mu} \in \mesSp$ s.t. $\oracle{\bm{\mesOut}} = (\bm{x},\bm{p})$} \right \}.
\end{equation*}

By using an argument similar to the one in the proof of \Cref{cor:dro_to_reg}, we can show that,
 \begin{equation*}
  \mathfrak{R}_{\mathcal{I}_{mK},\infty}^{\SAA}(\epsilon)  \leq 2 \sup_{\mesOut \in \mesSp} \sup_{\nu \in \mathcal{U}_{\epsilon}(\mesOut)} \sup_{\actVr \in \mathcal{X}^{\mathrm{Myerson}}} | \g[\actVr]{\mesOut} - \g[\actVr]{\nu} |,
 \end{equation*}
 as opposed to the bound \eqref{eq:bound_mohri} which involves a supremum over $\actSp$.

In turn, adapting the proof of \Cref{thm:generating_class} to the previous upper bound implies,
\begin{align}
\sup_{\mesOut \in \mesSp} \sup_{\nu \in \mathcal{U}_{\epsilon}(\mesOut)} \sup_{\actVr \in \mathcal{X}^{\mathrm{Myerson}}} | \g[\actVr]{\mesOut} - \g[\actVr]{\nu} | &\leq  \inf_{\delta \geq 0} \left \{ \sup_{\actVr \in \mathcal{X}^{\mathrm{Myerson}}} \lambda_{\actVr} \left( \mathcal{K}^{(m)}, g, \delta \right) \cdot \epsilon + 2\delta \right \} \nonumber\\
&\leq \sup_{\actVr \in \mathcal{X}^{\mathrm{Myerson}}} \lambda_{\actVr} \left( \mathcal{K}^{(m)}, g, \delta \right) \cdot \epsilon,\label{eq:SAA_md}
\end{align}
where the approximation parameter $\lambda \left( \mathcal{K}^{(m)}, g, \delta \right)$ is defined in \Cref{def:approx}. 
To conclude the analysis of SAA, we will derive an upper bound on the approximation parameter the bayesian mechanism design problem. Formally, we show the following.
\begin{proposition}
\label{prop:lambda_md}
For the bayesian mechanism design problem we have that,
\begin{equation*}
\sup_{\actVr \in \mathcal{X}^{\mathrm{Myerson}}} \lambda_{\actVr} \left( \mathcal{K}^{(m)}, g, \delta \right) \leq M.
\end{equation*}
\end{proposition}
\Cref{prop:lambda_md} and \eqref{eq:SAA_md} imply that for every $\epsilon \geq 0$,
\begin{equation*}
 \mathfrak{R}_{\mathcal{I}_{mK},\infty}^{\SAA}(\epsilon) \leq 2 M \cdot \epsilon.
\end{equation*}
Finally, to retrieve \cite{brustle2020multi}, we leverage  \Cref{prop:multiD_to_1D} and note that for every product measures $\bm{\mu}$ and $\bm{\nu} \in \Delta([0,M])^m$, if $\dis{\mathcal{K}}(\mu_i,\nu_i) \leq \epsilon$ for every $i \in \{1,\ldots,m\}$ then $d_{\mathcal{K}^{(m)}}(\bm{\mu},\bm{\nu}) \leq m \cdot \epsilon$. Therefore, we recover the bound $2M \cdot m \cdot \epsilon$ on the worst-case DRO regret of SAA under their notion of heterogeneity for every $\epsilon \geq 0$.

\subsection{Proofs and auxiliary results}

\begin{proof}[\textbf{Proof of \Cref{prop:multiD_to_1D}}]
For every $i \in \{1,\ldots,m\}$ $\tilde{\bm{\mu}}^{(i)} = \mu_1\times \ldots \times \mu_{i} \times \nu_{i} \times \ldots \times \nu_{m}$ and let $\tilde{\bm{\mu}}^{(0)} = \bm{\nu}$. We then have that for every $f \in \mathcal{F}_m$,
\begin{equation*}
\left| \mathbb{E}_{\bm{\envVr} \sim \bm{\mu}} \left[  f(\bm{\envVr}) \right] - \mathbb{E}_{\bm{\envVr'} \sim \bm{\nu}} \left[  f(\bm{\envVr'}) \right]  \right| = \left |\sum_{i=1}^m \mathbb{E}_{\bm{\envVr} \sim \bm{\mu}^{(i-1)}} \left[  f(\bm{\envVr}) \right] - \mathbb{E}_{\bm{\envVr'} \sim \bm{\mu}^{(i)}} \left[  f(\bm{\envVr'}) \right] \right | \leq \sum_{i=1}^m \left|  \mathbb{E}_{\bm{\envVr} \sim \bm{\mu}^{(i-1)}} \left[  f(\bm{\envVr}) \right] - \mathbb{E}_{\bm{\envVr'} \sim \bm{\mu}^{(i)}} \left[  f(\bm{\envVr'}) \right] \right|.
\end{equation*}
Furthermore, for every $i \in \{1, \ldots, m\}$, we have that,
\begin{align*}
| &\mathbb{E}_{\bm{\envVr} \sim \bm{\mu}^{(i-1)}}  \left[  f(\bm{\envVr}) \right] - \mathbb{E}_{\bm{\envVr'} \sim \bm{\mu}^{(i)}} \left[  f(\bm{\envVr'}) \right] |\\
\qquad &  = \left| \mathbb{E}_{\bm{\envVr_{-i}} \sim \mu_1 \times \ldots \mu_{i-1} \times \nu_{i+1} \times \ldots \nu_{\numS}} \left[ \int_{\envVr_i \in [0,M]}  f(\envVr_i,\bm{\envVr_{-i}})) d\mu_i(\envVr_i) - \int_{\envVr_i \in [0,M]}  f(\envVr_i,\bm{\envVr_{-i}})) d \nu_{i}(\envVr_i)  \right]  \right|\\
&\leq \mathbb{E}_{\bm{\envVr_{-i}} \sim \mu_1 \times \ldots \mu_{i-1} \times \nu_{i+1} \times \ldots \nu_{\numS}} \left[ \left|  \int_{\envVr_i \in [0,M]}  f(\envVr_i,\bm{\envVr_{-i}})) d\mu_i(\envVr_i) - \int_{\envVr_i \in [0,M]}  f(\envVr_i,\bm{\envVr_{-i}})) d \nu_{i}(\envVr_i) \right| \right] \\
&\stackrel{(a)}{\leq} \mathbb{E}_{\bm{\envVr_{-i}} \sim \mu_1 \times \ldots \mu_{i-1} \times \nu_{i+1} \times \ldots \nu_{\numS}} \left[ d_{\mathcal{F}}(\mu_i,\nu_i) \right] = d_{\mathcal{F}}(\mu_i,\nu_i).  
\end{align*}
where $(b)$ holds because $f \in \mathcal{F}^{(m)}$ and therefore, for every $\bm{\envVr_{-i}}$, we have that $\envVr_i \mapsto f(\envVr_i,\bm{\envVr_{-i}})) \in \gen{\mathcal{F}}$. Furthermore, by definition of the maximal generator $\gen{\mathcal{F}}$, we have that $\dis{\mathcal{F}} = \dis{\gen{\mathcal{F}}}$.

Therefore,
\begin{equation*}
\left| \mathbb{E}_{\bm{\envVr} \sim \bm{\mu}} \left[  f(\bm{\envVr}) \right] - \mathbb{E}_{\bm{\envVr'} \sim \bm{\nu}} \left[  f(\bm{\envVr'}) \right]  \right| \leq \sum_{i=1}^m d_{\mathcal{F}}(\mu_i,\nu_i).
\end{equation*}
We conclude the proof by taking a supremum over $f \in \mathcal{F}^{(m)}$.

\end{proof}

\begin{proof}[\textbf{Proof of \Cref{prop:ETC_multiK}}]
Consider a triangular array sequence of probability measures $(\bm{\mu}_{i,\numS})_{1 \leq i \leq \numS, \numS \in \mathbb{N}^*}$ all belonging to $\Delta([0,M])^k$.

For every $\numS \geq 1$, we note that, by definition, $\hat{\bm{\mu}}_\numS =  \hat{\nu}_{1,\numS} \times \ldots \times \hat{\nu}_{k,\numS}$ is a product probability measure. Furthermore,  $\bar{\mu}_\numS$ is a product probability measure as it is a convex combination of such measures. We denote by $(\bar{\mu}_{j,\numS})_{j \in \{1,\ldots,k\}}$ its marginals. Therefore, \Cref{prop:multiD_to_1D} implies that, 
\begin{equation*}
\dis{\mathcal{F}^{(m)}}(\hat{\bm{\mu}}_\numS, \bar{\bm{\mu}}_\numS) \leq \sum_{j=1}^m d_{\mathcal{F}}(\hat{\nu}_{j,\numS},\bar{\mu}_{j,\numS}),
\end{equation*}
We conclude that,
\begin{equation*}
\lim_{\numS \to \infty} \dis{\mathcal{F}^{(m)}}(\hat{\bm{\mu}}_\numS, \bar{\bm{\mu}}_\numS)  \leq \sum_{j=1}^m \lim_{\numS \to \infty} d_{\mathcal{F}}(\hat{\nu}_{j,\numS},\bar{\mu}_{j,\numS}) = 0,
\end{equation*}
where the last equality follows from the assumption that $\dis{\mathcal{F}}$ satisfies the ETC property.
\end{proof}

\begin{proof}[\textbf{Proof of \Cref{prop:lambda_md}}]
Recall that for every $(\mathbf{x},\mathbf{p}) \in \actSp$,
\begin{subequations}
\begin{alignat*}{2}
\lambda_{(\mathbf{x},\mathbf{p})}(\mathcal{K}^{(m)},g,0) := &\! \inf_{(\lambda_i)_{i \in \mathbb{N}}, (f_i)_{i \in \mathbb{N}}}        &\qquad& \sum_{i=1}^\infty |\lambda_i|  \\
&\text{subject to} &      & \sum_{i=1}^\infty \lambda_i \cdot f_i = g((\mathbf{x},\mathbf{p}),\cdot) , \\
&		    &     & f_i \in \mathcal{K}^{(m)}, \; \lambda_i \in \mathbb{R}.
\end{alignat*}
\end{subequations}

Therefore, to prove the result, it is sufficient to show that for every $(\mathbf{x},\mathbf{p}) \in \actSp^{\textrm{Myerson}}$, we have that, $\frac{g((\mathbf{x},\mathbf{p}),\cdot)}{M} \in \mathcal{K}^{(m)}.$

Fix $i \in \{1,\ldots,m\}$ and $\bm{\envVr}_{-i} \in [0,M]^{m-1}$ and let $\tilde{g} : \envVr_i \mapsto g((\mathbf{x},\mathbf{p}),(\envVr_i,\bm{\envVr}_{-i})).$ In what follows, we will show that $\frac{1}{M} \cdot \tilde{g} \in \gen{\mathcal{K}}$. By definition of  $\mathcal{K}^{(m)}$, this would imply that $\frac{g((\mathbf{x},\mathbf{p}),\cdot)}{M} \in \mathcal{K}^{(m)}$.

We have established in \Cref{prop:common_max} that $\gen{\mathcal{K}}$ is the set of one dimensional mappings such that the total variation is less or equal than $1$. Therefore, to show that $\frac{1}{M} \cdot \tilde{g} \in \gen{\mathcal{K}}$, it suffices to show that $V(\tilde{g}) \leq M.$
To do so, we will show that: i) $\tilde{g}$ is non-decreasing, and ii) $\tilde{g}(M) - \tilde{g}(0) \leq M$.   

We note that for every Myerson mechanism the item is allocated to the bidder with the highest ironed virtual value function (see \cite{hartline2013bayesian} for a detailed description of the virtual value function and the ironing process). Furthermore, the payment of the winning agent equals the minimum payment such that they are still winning. The key is to note that the ironed virtual value function is by definition non-decreasing. Therefore one clearly see that increasing the value of a single agent while keeping all other values fixed weakly increase the revenue of the decision-maker. Hence, $\tilde{g}$ is non-decreasing.

Furthermore, $\tilde{g}(0) \geq 0$ because the revenue of the decision-maker is non-negative and $\tilde{g}(M) \leq M$ by the constraint \eqref{eq:IR}. This concludes the proof.

\end{proof}

\section{Strict inequality in Remark 2}
\label{sec:apx_strict_inequality}
In this section, we show that there exists an instance $\mathcal{I}$ such that,
\begin{equation*}	
\mathfrak{R}_{\scaleto{\mathcal{I},\mathrm{DRO}}{6pt}}^{\SAA} (\epsilon) < 2  \cdot \sup_{\mesOut \in \mesSp} \sup_{\nu\in \mathcal{U}_{\epsilon}(\mesOut)} \sup_{\actVr \in \actSp} | \g[\actVr]{\mesOut} - \g[\actVr]{\nu} |.
\end{equation*}	

Let $\mathcal{I}$ be the newsvendor problem under the Kolmogorov distance presented in \Cref{sec:newsvendor_SAA}. For this instance, we characterize exactly the value of the DRO problem for SAA.
\begin{proposition}[Exact DRO regret for SAA for the Kolmogorov Newsvendor problem]
\label{prop:DRO_News_K}
Let $\mathcal{I}$ be the newsvendor problem under the Kolmogorov distance. For every $\epsilon \geq 0$, we have that,
\begin{equation*}
\mathfrak{R}_{\scaleto{\mathcal{I},\mathrm{DRO}}{6pt}}^{\SAA} (\epsilon) = \left( c_u + c_o \right) \cdot M \cdot \epsilon.
\end{equation*}
\end{proposition}

In the proof of \Cref{cor:SAA_Newsvendor}, we have showed that for the newsvendor problem under the Kolmogorov distance,
\begin{equation*}
2  \cdot \sup_{\mesOut \in \mesSp} \sup_{\nu\in \mathcal{U}_{\epsilon}(\mesOut)} \UB_{\mathcal{I}}(\mesOut,\nu) \leq 2 \max\left(c_u,c_o \right) \cdot M \cdot \epsilon.
\end{equation*}
We next show that this inequality is in fact an equality.
\begin{proposition}
\label{prop:newsvendor_K_equality}
Let $\mathcal{I}$ be the newsvendor problem under the Kolmogorov distance. For every $\epsilon \geq 0$, we have that,
\begin{equation*}
2  \cdot \sup_{\mesOut \in \mesSp} \sup_{\nu\in \mathcal{U}_{\epsilon}(\mesOut)} \UB_{\mathcal{I}}(\mesOut,\nu) = 2 \max\left(c_u,c_o \right) \cdot M \cdot \epsilon.
\end{equation*}
\end{proposition}
Therefore, \Cref{prop:DRO_News_K} and \Cref{prop:newsvendor_K_equality} imply that when $c_u \neq c_o$ we have that,
\begin{equation*}
\mathfrak{R}_{\scaleto{\mathcal{I},\mathrm{DRO}}{6pt}}^{\SAA} (\epsilon) < 2  \cdot \sup_{\mesOut \in \mesSp} \sup_{\nu\in \mathcal{U}_{\epsilon}(\mesOut)} \UB_{\mathcal{I}}(\mesOut,\nu) .
\end{equation*}

\subsection{Proofs of the auxiliary results}
Denote by $q = \frac{c_u}{c_u + c_o}$  the newsvendor critical fractile. In this section, we define the optimal action for a probability  measure $\mesOut$ as $\oracle{\mu}$. We note that for the newsvendor problem, $ \oracle{\mu} = F^{-1}(q)$, where $F^{-1}(z) = \inf \{\actVr \text{ s.t } F(\actVr) \geq z\}$ and $F$ is the cdf associated to the probabity measure $\mesOut$.

\begin{lemma}
\label{lem:worst-reg_decision}
Consider the Kolmogorov newsvendor problem.
Let $\epsilon \leq \min(q,1-q).$ Let $\nu \in \mesSp$ and denote by $H$ its associated cdf. Then, for every $\actVr \in \actSp$  we have that,
\begin{equation*}
\sup_{\mesOut \in \mathcal{U}_{\epsilon}(\nu)}  \g[\actVr]{\mu} - \opt(\mu) = (c_u + c_o) \cdot \max \left (\int_{H^{-1}(q- \epsilon)}^{\actVr} \left(H_U(\envVr) - q \right) d \envVr, \int_{\actVr}^{H^{-1}(q+\epsilon)} \left(q - H_L(\envVr) \right) d \envVr \right),
\end{equation*}
where $H_U = \min(1,H + \epsilon)$ and $H_L = \max(0,H-\epsilon)$.
\end{lemma}

\begin{proof}[\textbf{Proof of \Cref{lem:worst-reg_decision}}]
Lemma A-1 of \cite{besbes2023big} shows that the expected newsvendor cost satisfies for every action $\actVr \in \actSp$ and probability measure $\mu \in \mesSp$ with associated cdf $F$, that
\begin{equation}
\label{eq:integral_newsvendor}
\g[\actVr]{\mu} = c_u \cdot \left( \mathbb{E}_{\envVr \sim \mu}[\envVr] - \actVr \right) + (c_u + c_o) \cdot \int_0^{\actVr} F(\envVr) d \envVr.
\end{equation}
Denote by $q = \frac{c_u}{c_u + c_o}$  the newsvendor critical fractile. We note that for the newsvendor, $\oracle{\mu} = F^{-1}(q)$, where $F^{-1}(z) = \inf \{\actVr \text{ s.t } F(\actVr) \geq z\}$. 

Hence, for any action $\actVr \in \actSp$, the regret can be rewritten as,
\begin{align*}
\g[\actVr]{\mu} - \opt(\mu) &= \g[\actVr]{\mu} - \g[\oracle{\mu}]{\mu} \\
& \stackrel{(a)}{=} c_u \cdot ( \oracle{\mu} - \actVr) + (c_u + c_o) \cdot \int_{\oracle{\mu}}^{\actVr} F(\envVr) d \envVr\\
& = (c_u + c_o) \cdot \int_{F^{-1}(q)}^{\actVr} \left(F(\envVr) - q \right) d \envVr,
\end{align*}
where $(a)$ follows from \eqref{eq:integral_newsvendor}. In what follows, let us fix $\nu \in \mesSp$ and let denote by $H$ the associated cdf. We have,
\begin{align*}
\sup_{\mesOut \in \mathcal{U}_{\epsilon}(\nu)} \g[\actVr]{\mu} - \opt(\mu)  &= (c_u + c_o) \cdot \sup_{\mesOut \in \mathcal{U}_{\epsilon}(\nu)} \int_{F^{-1}(q)}^{\actVr} \left(F(\envVr) - q \right) d \envVr\\
&=(c_u + c_o) \cdot \max \Big( \sup_{ \substack{F \text{ cdf s.t.}\\ \actVr \geq F^{-1}(q),\\ \|F-H\|_{\infty} \leq \epsilon}}  \int_{F^{-1}(q)}^{\actVr} \left(F(\envVr) - q \right) d \envVr,\\
 &\quad \sup_{ \substack{F \text{ cdf s.t.}\\ \actVr \leq F^{-1}(q),\\ \|F-H\|_{\infty} \leq \epsilon}}  \int_{\actVr}^{F^{-1}(q)} \left(q- F(\envVr)\right) d \envVr  \Big)
\end{align*}
We first solve the optimization problem appearing in the first term of the maximum and show that,
\begin{equation}
\label{eq:sup_mu}
\sup_{ \substack{F \text{ cdf s.t.}\\ \actVr \geq F^{-1}(q),\\ \|F-H\|_{\infty} \leq \epsilon}}  \int_{F^{-1}(q)}^{\actVr} \left(F(\envVr) - q \right) d \envVr = \begin{cases}
 \int_{H^{-1}(q- \epsilon)}^{\actVr} \left(H_U(\envVr) - q \right) d \envVr \qquad \text{if $H^{-1}(q-\epsilon) \leq \actVr$,}\\
 -\infty \qquad \text{o.w.},
\end{cases}
\end{equation}
where $H_U = \min(1,H + \epsilon)$.

\textit{Case 1: $H^{-1}(q-\epsilon) > \actVr.$} In this case, we show that the problem is infeasible.  Assume that,
\begin{equation}
\label{eq:point_max}
\inf_{ \substack{F \text{ cdf s.t.}\\ \|F-H\|_{\infty} \leq \epsilon}} F^{-1}(q) \geq H^{-1}(q- \epsilon).
\end{equation}
If  $H$  is such that, $H^{-1}(q-\epsilon) > \actVr$, \eqref{eq:point_max} implies that there is no $F$ such that $\|F-H\|_{\infty} \leq \epsilon$ and $F^{-1}(q) \leq \actVr$ which implies that the problem in LHS of \eqref{eq:sup_mu} is infeasible.

We next show \eqref{eq:point_max}. Assume for sake of contradiction that there exists $F$ such that $\|F-H\|_{\infty} \leq \epsilon$ and $F^{-1}(q) < H^{-1}(q- \epsilon)$. This implies that, $q < F(H^{-1}(q- \epsilon))$. Using the fact that $F(H^{-1}(q- \epsilon)) \leq H(H^{-1}(q- \epsilon)) + \epsilon = q$ we obtain the contradiction that $q < q$. Hence \eqref{eq:point_max} holds.

\textit{Case 2: $H^{-1}(q-\epsilon) \leq \actVr.$}  In that case, we note that for every $F$ such that, $\|F-H\|_{\infty} \leq \epsilon$ and $F^{-1}(q) \leq \actVr$, we have that,
\begin{equation}
\label{eq:upper}
 \int_{F^{-1}(q)}^{\actVr} \left(F(\envVr) - q \right) d \envVr \stackrel{(a)}{\leq}  \int_{F^{-1}(q)}^{\actVr} \left(H_{U}(\envVr) - q \right) d \envVr \stackrel{(b)}{\leq} \int_{H^{-1}(q- \epsilon)}^{\actVr} \left(H_{U}(\envVr) - q \right) d \envVr,
\end{equation}
where $(a)$ holds because, $F(\envVr) \leq \min(1,H(\envVr) + \epsilon)$ for every $\envVr \in [0,M]$ and $(b)$ follows from the following argument: $i)$ we note that the application, $z \mapsto \int_{z}^{\actVr} \left(H_{U}(\envVr)  - q \right) d \envVr$ is non-increasing on $[H^{-1}(q- \epsilon),\actVr]$, because for every $\envVr \geq H^{-1}(q- \epsilon)$, we have by definition of $H^{-1}$, that $H(\envVr)  +\epsilon - q \geq 0$. $ii)$ \eqref{eq:point_max} implies that any feasible $F$ satisfies that $F^{-1}(q) \geq H^{-1}(q- \epsilon)$. $i)$ and $ii)$ implies that $(b)$ holds.

To conclude that \eqref{eq:sup_mu} holds, we need to show that the upper bound derived in \eqref{eq:upper} is achieved. We consider the distribution $H_U(\envVr) = \min(1, H(\envVr) + \epsilon)$ for every $\envVr \in [0,M]$. $H_U$ satisfies clearly $\|H_{U} -H\|_{\infty} \leq \epsilon$. Furthermore $H_U^{-1}(q) = H^{-1}(q-\epsilon) \leq \actVr$. Indeed, we have that $H_U^{-1}(q) = H^{-1}(q-\epsilon)$ because, by definition of $H^{-1}(q-\epsilon)$, we have on the one hand that, $H_U(H^{-1}(q-\epsilon)) \geq H(H^{-1}(q-\epsilon)) + \epsilon \geq q$. On the other hand, for all $z < H^{-1}(q-\epsilon)$, $H(z) + \epsilon < q$ which implies that $H_{U}(z) < q$.  Hence $H_U$ is feasible for the optimization problem  in \eqref{eq:sup_mu}. Moreover, it verifies that, 
\begin{equation*}
\int_{H_U^{-1}(q)}^{\actVr} (H_U(\envVr) -q) d \envVr = \int_{H^{-1}(q- \epsilon)}^{\actVr} \left(H_U(\envVr) - q \right) d \envVr.
\end{equation*}
Therefore \eqref{eq:sup_mu} holds.

By a similar argument, one can prove that,
\begin{equation*}
\sup_{ \substack{F \text{ cdf s.t.}\\ \actVr \leq F^{-1}(q),\\ \|F-H\|_{\infty} \leq \epsilon}}  \int_{\actVr}^{F^{-1}(q)} \left( q - F(\envVr) \right) d \envVr = \begin{cases}
 \int_{\actVr}^{H^{-1}(q+\epsilon)} \left(q - H_{L}(\envVr) \right) d \envVr \qquad \text{if $H^{-1}(q+\epsilon) \geq \actVr$,}\\
 -\infty \qquad \text{o.w.}
 \end{cases}
\end{equation*}
\end{proof}
Finally, by remarking that $H^{-1}(q-\epsilon) > \actVr$ and $H^{-1}(q+\epsilon) < \actVr$ cannot hold simultaneously, we obtain the desired result.

\begin{proof}[\textbf{Proof of \Cref{prop:DRO_News_K}}]
Recall that, $\mathfrak{R}_{\scaleto{\mathcal{I},\mathrm{DRO}}{6pt}}^{\SAA} (\epsilon)  = \sup_{\mu \in \mesSp} \sup_{\nu \in \mathcal{U}_\epsilon(\mesOut)} \g[\SAA(\nu)]{\mu} - \opt(\mu)$. This problem can alternatively be written by swapping the supremum symbols as, $\mathfrak{R}_{\scaleto{\mathcal{I},\mathrm{DRO}}{6pt}}^{\SAA} (\epsilon)  = \sup_{\nu \in \mesSp} \sup_{\mesOut \in \mathcal{U}_\epsilon(\nu)} \g[\SAA(\nu)]{\mu} - \opt(\mu)$.
To prove this result, we first solve the inner optimization problem over the distribution $\mesOut$ and then solve the problem over $\nu$.

Note that the inner optimization problem is solved by \Cref{lem:worst-reg_decision}. For any probability measure $\nu \in \mesSp$ and denoting by $H$ the associated cdf, we have that, $\SAA(\nu) = H^{-1}(q)$ and by \Cref{lem:worst-reg_decision} we obtain that,
\begin{equation*}
\sup_{\mesOut \in \mathcal{U}_{\epsilon}(\nu)}  \g[\SAA(\nu)]{\mu} - \opt(\mu) = (c_u + c_o) \cdot \max \left (\int_{H^{-1}(q- \epsilon)}^{H^{-1}(q)} \left(H_U(\envVr) - q \right) d \envVr, \int_{H^{-1}(q)}^{H^{-1}(q+\epsilon)} \left(q - H_L(\envVr) \right) d \envVr \right)
\end{equation*}

We next solve the optimization problem over $\nu \in \mesSp$, or alternatively over the cdf $H$.
In particular, we next show that,
\begin{equation}
\label{eq:outer_sup}
\sup_{H \text{ cdf over $[0,M]$}} \int_{H^{-1}(q- \epsilon)}^{H^{-1}(q)} \left(H_U(\envVr) - q \right) d \envVr = M \cdot \epsilon.
\end{equation}
We first note that, for every cumulative distribution function $H$ supported on $[0,M]$,
\begin{equation*}
\int_{H^{-1}(q-\epsilon)}^{H^{-1}(q)} \left(H_{U}(\envVr) - q \right) d \envVr \leq \int_{H^{-1}(q-\epsilon)}^{H^{-1}(q)} (H(\envVr) + \epsilon - q)   d \envVr \stackrel{(a)}{\leq}  \left[H^{-1}(q) - H^{-1}(q-\epsilon) \right] \cdot \epsilon \stackrel{(b)}{\leq}  M \cdot \epsilon,
\end{equation*}
where $(a)$ holds because $H(\envVr) \in [q-\epsilon,q]$ for every $\envVr \in [ H^{-1}(q-\epsilon),H^{-1}(q)]$ and $(b)$ follows from the fact that $H$ is supported on $[0,M]$. Furthermore for $ 0 < \eta < \epsilon$, we define for every $z \in [0,M]$,
\begin{equation*}
H_{\eta}(z) = \begin{cases}
 q -\eta \quad \text{if $z < M$}\\
 1 \quad \text{if $z =M$.}
\end{cases}
\end{equation*}
We note that since $\epsilon \leq 1-q$, we have $\int_{H_{\eta}^{-1}(q-\epsilon)}^{H_{\eta}^{-1}(q)} \left(\min(1,H_{\eta}(\envVr) +\epsilon) - q \right) d \envVr  = M \cdot (\epsilon - \eta).$ Hence for every $0 < \eta < \epsilon$, 
\begin{equation*}
M \cdot (\epsilon - \eta) \leq \sup_{H \text{ cdf over $[0,M]$}} \int_{H^{-1}(q-\epsilon)}^{H^{-1}(q)} \left(H_{U}(\envVr) - q \right) d \envVr \leq M \cdot \epsilon,
\end{equation*}
which implies that \eqref{eq:outer_sup} holds. We can similarly show that,
\begin{equation*}
\sup_{H \text{ cdf over $[0,M]$}} \int_{H^{-1}(q)}^{H^{-1}(q+\epsilon)} \left(q- H_{L}(\envVr)\right) d \envVr = M \cdot \epsilon,
\end{equation*}
and conclude that,
\begin{equation*}
\sup_{\mu \in \mesSp} \sup_{\nu \in \mathcal{U}_\epsilon(\mesOut)} \g[\SAA(\nu)]{\mu} - \opt(\mu) = (c_u + c_o) \cdot M \cdot \epsilon.
\end{equation*}

\end{proof}

\begin{proof}[\textbf{Proof of \Cref{prop:newsvendor_K_equality}}]
The upper bound has been derived in the proof of \Cref{cor:SAA_Newsvendor}. We next show that,
\begin{equation*}
2  \cdot \sup_{\mesOut \in \mesSp} \sup_{\nu\in \mathcal{U}_{\epsilon}(\mesOut)} \sup_{\actVr \in \actSp} | \g[\actVr]{\mesOut} - \g[\actVr]{\nu} | \geq 2 \max\left(c_u,c_o \right) \cdot M \cdot \epsilon.
\end{equation*}

For every $p \in [0,1]$, let $\mathcal{B}_M(p)$ denote the two-point-mass distribution which puts mass $p$ at $M$ and mass $1-p$ at $0$. For every $p \in [0,1]$ we note that,
\begin{equation*}
\g[0]{\mathcal{B}_M(p)} = p \cdot c_u \cdot M \quad \text{and} \quad \g[1]{\mathcal{B}_M(p)} = (1-p) \cdot c_o \cdot M
\end{equation*}
Furthermore, for every $p \in [0,1-\epsilon]$, $\mathcal{B}_M(p+\epsilon) \in \mathcal{U}_{\epsilon} \left( \mathcal{B}_M(p) \right)$ and,
\begin{equation*}
|\g[0]{\mathcal{B}_M(p)} - \g[0]{\mathcal{B}_M(p+ \epsilon)}| = \epsilon \cdot c_u \cdot M \quad \text{and} \quad |\g[1]{\mathcal{B}_M(p)} - \g[1]{\mathcal{B}_M(p+ \epsilon)}| = \epsilon \cdot c_o \cdot M,
\end{equation*}
which implies that,
\begin{equation*}
\sup_{\mesOut \in \mesSp} \sup_{\nu\in \mathcal{U}_{\epsilon}(\mesOut)} \sup_{\actVr \in \actSp} | \g[\actVr]{\mesOut} - \g[\actVr]{\nu} | \geq \max(c_o,c_u) \cdot M \cdot \epsilon.
\end{equation*}
\end{proof}

\end{document}